\documentclass{article}

\usepackage{PRIMEarxiv}
\usepackage[utf8]{inputenc} 
\usepackage[T1]{fontenc}    
\usepackage{hyperref}       
\usepackage{url}            
\usepackage{booktabs}       
\usepackage{amsfonts}       
\usepackage{amsmath}  
\usepackage{amssymb}  
\usepackage{amsthm}
\usepackage{bm}
\usepackage{nicefrac}       
\usepackage{microtype}      
\usepackage{lipsum}
\usepackage{fancyhdr}       
\usepackage{graphicx}       
\usepackage{subcaption}
\usepackage[labelfont=bf]{caption}
\usepackage{tabularx}
\usepackage{natbib}
\usepackage{makecell}
\usepackage{array}
\usepackage[ruled, linesnumbered]{algorithm2e}
\DontPrintSemicolon
\usepackage{pdflscape}
\usepackage{longtable}
\usepackage{cleveref}
\usepackage{multirow}
\usepackage{threeparttable}
\usepackage[section]{placeins}

\usepackage[dvipsnames]{xcolor}
\usepackage{xargs}
\usepackage[colorinlistoftodos, prependcaption, textsize=small, tickmarkheight=0.7em, textwidth=20mm]{todonotes}
\usepackage{todonotes}
\newcommandx{\ilker}[2][1=]{\todo[linecolor=Green,backgroundcolor=Green!25,bordercolor=Green,#1]{\footnotesize{\textcolor{blue}{\textbf{Ilker:} #2}}}}
\newcommandx{\ilkercaption}[2][1=]{\todo[inline,linecolor=Green,backgroundcolor=Green!25,bordercolor=Green,#1]{\footnotesize{\textcolor{blue}{\textbf{Ilker:} #2}}}}
\newcommandx{\wenhao}[2][1=]{\todo[linecolor=Green,backgroundcolor=Green!25,bordercolor=Green,#1]{\footnotesize{\textcolor{blue}{\textbf{Wenhao:} #2}}}}
\newcommandx{\wenhaocaption}[2][1=]{\todo[inline,linecolor=Green,backgroundcolor=Green!25,bordercolor=Green,#1]{\footnotesize{\textcolor{blue}{\textbf{Wenhao:} #2}}}}

\newtheorem{theorem}{Theorem}
\newtheorem{corollary}{Corollary}
\newtheorem{proposition}{Proposition}

\theoremstyle{definition}
\newtheorem{definition}{Definition}

\theoremstyle{remark}
\newtheorem{remark}{Remark}
\title{Learning an Interpretable Risk Scoring System for \\ Maximizing Decision Net Benefit}
\author{
  Wenhao Chi \\
  Yau Mathematical Sciences Center, Tsinghua University \\
  Amsterdam Business School, University of Amsterdam\\
  whchi@tsinghua.edu.cn
  \And
  \c{S}. \.{I}lker Birbil\\
  Amsterdam Business School, University of Amsterdam\\
  s.i.birbil@uva.nl
}

\begin{document}
\maketitle

\begin{abstract}
Risk scoring systems are widely used in high-stakes domains to assist decision-making. However, existing approaches often focus on optimizing predictive accuracy or likelihood-based criteria, which may not align with the main goal of maximizing utility. In this paper, we propose a novel risk scoring system that directly optimizes net benefit over a range of decision thresholds. The model is formulated as a sparse integer linear programming problem which enables the construction of a transparent scoring system with integer coefficients, and hence, facilitates interpretation and practical application. We also establish fundamental relationships among net benefit, discrimination, and calibration. Specifically, we derive bounds relating the area under the net benefit curve to a ROC functional, both evaluated on a fixed threshold grid, and show that post-processing can achieve moderate calibration on the training data without decreasing the area under the net benefit curve on that grid. We evaluated our method on multiple public datasets as well as on a large-scale credit risk dataset. This computational study demonstrated that our interpretable method can effectively achieve high net benefit while maintaining competitive discrimination and calibration performance.
\end{abstract}

\section{Introduction}
Risk scoring models are widely used in decision analysis, particularly in healthcare and criminal justice, to assess risk and guide decision making. These models are favored for their simplicity, ease of interpretation, and rapid evaluation using linear, sparse, integer-based coefficients. However, developing effective risk scoring models remains a challenge. 

A good risk scoring model should not only achieve accurate calibration and high discrimination, but also have high utility in decision-making \citep{steyerbergAssessingPerformancePrediction2010, sadatsafaviMovingAUCDecision2021, FAZEL2022101902,MARKOV2022180}. Calibration ensures that the predicted risks are closely aligned with the actual outcomes, enabling predictions to be interpreted as meaningful probabilities. For example, a predicted risk of 1\% means that, on average, out of every 100 individuals with that risk score, approximately one is expected to experience the event. High discrimination, on the other hand, allows the model to distinguish effectively between different risk levels. These two metrics, although valuable, cannot alone determine whether a model will be practically useful in real-world settings. More importantly, they do not help decision-makers choose between competing models \citep{vickersTraditionalStatisticalMethods2010}.  

To address this issue, researchers have introduced decision curve analysis (DCA) to measure the utility of the model \citep{vickersDecisionCurveAnalysis2006}. DCA is a framework that quantifies a model's net benefit by evaluating the trade-off between true positives and false positives at various decision thresholds. Conventional model development typically focuses on optimizing objectives such as likelihood, mean squared error, or other unweighted loss functions. Although these objectives can produce models with good calibration and discrimination, they do not necessarily ensure that the resulting predictions lead to better utility, and that is what really matters \citep{vickersHypothesisNetBenefit2025}. The fundamental difficulty is that neither calibration nor discrimination captures what happens after a prediction is acted upon. Discrimination, as measured by the Area Under the Receiver Operating Characteristic (AUROC) curve, ranks individuals relative to one another but is entirely insensitive to the decision threshold a practitioner actually uses, and it weights false positives and false negatives symmetrically --- an assumption that rarely holds in practice, where the two types of error carry very different consequences. Calibration ensures that predicted probabilities are accurate on average, but a well-calibrated model can still yield poor decisions if its probability estimates are imprecise precisely in the neighborhood of the operative threshold. More critically, neither metric can answer the question that decision-makers actually face: given my threshold, is this model better than simply treating everyone or no one, and which of two competing models should I prefer? 

Recent studies have increasingly emphasized that predictive model evaluation should account for downstream decision consequences rather than relying solely on conventional discrimination metrics. \citet{zhuWeightedBrierScore2025} introduced a weighted Brier score that incorporates clinical utility considerations into risk prediction evaluation and established connections among discrimination, calibration. \citet{floresConsequentialistCritiqueBinary2025} argued that commonly used classification metrics may fail to reflect the practical consequences of prediction errors. Furthermore, \citet{2025arXiv250614540F} investigated how calibration, label shift, and asymmetric error costs influence the alignment between model evaluation and clinical priorities. These studies highlight the importance of utility assessment of predictive models. However, they primarily focus on evaluating existing models through decision-related criteria rather than incorporates decision utility into the learning process. Therefore, the objective of this study is to develop a risk scoring model that directly optimizes utility. In addition, our goal is to ensure that the proposed model retains strong learning capacity and generalization while simultaneously achieving high levels of calibration and discrimination.

The development of risk scoring systems involves three interrelated challenges: constructing parsimonious models with interpretable integer coefficients, evaluating predictive performance through calibration and discrimination, and ultimately ensuring that model predictions translate into high-quality decisions. Next, we review prior work across these three dimensions. We begin with the line of research on sparse integer scoring systems, focusing in particular on SLIM \citep{ustunSupersparseLinearInteger2016} and RISKSLIM \citep{ustunLearningOptimizedRisk2019} as the methodological predecessors of the approach proposed here. We then discuss DCA as the evaluation framework that motivated our choice of net benefit as a training objective, reviewing both the foundational work of \citet{vickersDecisionCurveAnalysis2006} and subsequent extensions. Throughout, we highlight the gap that motivates the present work: existing scoring system methods optimize predictive accuracy or likelihood-based criteria rather than decision utility, and existing utility evaluation methods are applied post hoc rather than embedded in model training.

\paragraph{Risk Scoring Systems}
Ustun and Rudin proposed the Supersparse Linear Integer Model (SLIM), which learns sparse linear classifiers with small integer coefficients by optimizing a 0-1 loss function through mixed-integer linear programming \citep{ustunSupersparseLinearInteger2016}. Building on this work, they further proposed RISKSLIM, a variant designed specifically for risk assessment, which instead minimizes logistic loss by solving a mixed-integer nonlinear program \citep{ustunLearningOptimizedRisk2019}. Compared with the SLIM model that produces only binary classification outputs, RISKSLIM can generate probability estimates and achieve better calibration and discrimination.

Our model departs from both SLIM and RISKSLIM by shifting the optimization focus from predictive accuracy or calibration to explicit decision utility. While SLIM minimizes a 0-1 loss for binary classification and RISKSLIM optimizes logistic loss to improve calibration and risk estimation, the objective of the proposed approach is to maximize the weighted sum of net benefits across multiple decision thresholds. When interval-length weights are used, this objective coincides with the discretized area under the net benefit (AUNBC). This utility-driven perspective integrates DCA within the learning process. Thus, it enables the model to align predictive performance with practical decision outcomes rather than post-hoc evaluation. Structurally, all three models share the use of sparse integer coefficients for interpretability, but the proposed model introduces multiple integer intercepts and decision-dependent variables to accommodate piecewise constant risk probabilities tailored to user-defined thresholds. Unlike SLIM’s single binary output and RISKSLIM’s continuous risk scores, the proposed model generates calibrated, piecewise risk estimates that explicitly balance discrimination and calibration with decision utility. {\Cref{tab:rssdnb_comparison} summarizes the key differences between RSS-DNB, SLIM, and RISKSLIM from multiple perspectives.} Finally, it generalizes the theoretical framework of SLIM by extending learning capacity results to multi-threshold risk scoring. This clearly demonstrates comparable generalization performance when coefficient bounds are sufficiently broad.

\begin{table}[htbp]
\centering
\caption{Comparison of RSS-DNB, SLIM, and RISKSLIM.}
\label{tab:rssdnb_comparison}
\renewcommand{\arraystretch}{1.3}
\begin{tabularx}{\textwidth}{@{}l l X X X@{}}
\toprule
& & \textbf{RSS-DNB} & \textbf{SLIM} & \textbf{RISKSLIM} \\
\midrule
\multirow{6}{*}{\rotatebox{90}{\textbf{Target/Aim}}}
 & Objective
   & Weighted sum of net benefits across multiple decision thresholds (explicitly utility-driven).
   & 0--1 loss for binary classification (prediction accuracy).
   & Logistic loss (calibration, discrimination, risk estimation). \\[2.2em]
 & Net Benefit
   & Directly incorporates decision curve analysis and evaluates utility as area under net benefit curve.
   & Not included (focuses on discrimination via accuracy).
   & Not included directly (focuses on risk probabilities via calibration). \\
\midrule
\multirow{9}{*}{\rotatebox{90}{\textbf{Output/Prediction}}}
 & Output Type
   & Predicts risk probabilities via piecewise constant mapping determined by integer thresholds.
   & Binary classification only.
   & Probabilistic risk scores. \\[1.5em]
 & Thresholds
   & Multiple, user-defined thresholds for decision analysis.
   & One default threshold.
   & Calibrated continuous risk scores. \\[1.5em]
 & \makecell[l]{Calibration \&\\ Discrimination}
   & Explicitly incorporated and analyzed for net benefit.
   & Not explicitly addressed, but sparse coefficients provide interpretability.
   & Designed to improve calibration via logistic link. \\
\midrule
\multirow{9}{*}{\rotatebox{90}{\textbf{Structure/Variables}}}
 & \makecell[l]{Coefficient\\ constraints}
   & Sparse integer coefficients, finite integer sets.
   & Sparse integer coefficients, finite integer sets.
   & Sparse integer coefficients, finite integer sets. \\[1.5em]
 & Intercepts
   & Multiple integer intercepts for each threshold.
   & Single intercept for binary separation.
   & Single intercept (or logistic model). \\[1.5em]
 & \makecell[l]{Decision\\ variables}
   & Binary assignment, selection vector, coefficients, threshold vector, and sparsity indicator.
   & Sparse selection for coefficients.
   & Sparse selection for risk score coefficients. \\[1.5em]
\bottomrule
\end{tabularx}
\end{table}

\paragraph{Decision Curve Analysis}
In order to comprehensively evaluate the net benefit across a range of thresholds, Talluri and Shete proposed a measure named weighted area under the net benefit curve (WA-NBC) to perform decision curve analysis, which provides a reasonable method to compare two competing models crossing in the range of interest \citep{talluriUsingWeightedArea2016}. In our study, we used the discretized area under the net benefit curve based on a prespecified threshold grid as a measure of utility.

Rousson and Zumbrunn showed that, for a given decision threshold and an estimate of disease prevalence, the optimal operating point on the Receiver Operating Characteristic (ROC) curve — the one that maximizes the net benefit — can be identified as the point where the slope of the curve equals a specific value determined jointly by the prevalence and the threshold \citep{roussonDecisionCurveAnalysis2011}. This provides a new perspective on how to select an optimal cutoff point for a given model, but does not address the issue of how to construct a model that optimizes the net benefit, which is the focus of our work.

Van Calster et al. defined a hierarchy of four increasingly high levels of calibration: mean, weak, moderate, and strong calibration. Mean calibration refers to the observed event rate equal to the average predict risk; weak calibration is known as logistic calibration; moderate calibration means that the predicted risks are consistent with the observed event frequencies within each risk stratum; and strong calibration requires that predicted risks are consistent with the observed event frequencies for every covariate pattern. They further demonstrated that if a risk prediction model achieves moderate calibration, the net benefit of decisions based on this model will not be lower than that of the baseline strategies -- treating all or treating none \citep{vancalsterCalibrationHierarchyRisk2016}.

Recently, Vickers et al. hypothesized that directly optimizing net benefit during model development --rather than relying on unweighted loss functions such as mean squared error-- may yield models with greater clinical utility. They also called for methodological research to identify the scenarios in which the net benefit should be adopted as the objective function for the development of models \citep{vickersHypothesisNetBenefit2025}.

Our work is also related to cost-sensitive learning, which considers misclassification costs. Both methods take into account asymmetric loss; however, they differ in how they incorporate this asymmetry into the learning objective. Cost-sensitive learning typically introduces asymmetry through a predefined cost matrix or a cost-weighted prediction loss, where the relative costs of false positives and false negatives are specified a priori \citep{Ling2010}. In contrast, our model aims to optimize the weighted sum of net benefits, which summarizes decision utility across a range of thresholds rather than relying on a single predefined cost setting. 

Decision-focused learning also incorporates downstream decision into model development. \citet{donti2017Taskbased} proposed to train a probabilistic model based on the task loss in stochastic programming. For combinatorial optimization, \citet{wilder2019Melding} used a continuous relaxation of the discrete problem to propagate gradients from decision to predictive model. \citet{bertsimas2020Predictive} studied conditional stochastic optimization, using weighted historical observations to estimate cost. \citet{elmachtoub2022Smart} proposed a general predict-then-optimize framework that directly uses the structure of optimization problem to design prediction models. Our model shares this decision-oriented motivation but differs in its setting and model structure. It focuses on binary risk decisions across multiple prespecified thresholds and learns a sparse linear integer model.

The main contributions of this paper are as follows. First, we propose a novel {decision-oriented} risk scoring model -- the Risk Scoring System for Decision Net Benefit (RSS-DNB) -- whose training objective is a weighted sum of net benefits across prespecified thresholds minus a sparsity penalty. The model is formulated as a sparse mixed-integer linear program, producing transparent, integer-coefficient scoring rules suitable for high-stakes applications. Second, we establish rigorous theoretical relationships between net benefit, discrimination, and calibration. Specifically, for a fixed threshold grid, we show that the lower bound on a discretized ROC functional is non-decreasing in the area under the net benefit curve evaluated on a same threshold grid (\Cref{thm:auroc_aunbc} and \Cref{cor:auroc_aunbc}). We also show that post-processing can achieve moderate calibration on the training data without decreasing the area under the net benefit curve on that grid (\Cref{thm:improve_aunbc} and \Cref{cor:calibration}). Third, we characterize the learning capacity of the proposed integer scoring framework, showing that sufficiently large coefficient bounds allow the integer model to match the weighted net benefit of any real-valued linear classifier (Theorem \ref{thm:learning_capacity} and Corollary \ref{cor:learning_capacity}), and we derive finite-sample generalization bounds for the empirical-to-expected net benefit gap (Theorem \ref{thm:generalization}). Fourth, we propose a simulated annealing algorithm (RSS-DNB-SA) that efficiently scales the approach to large datasets where exact mixed-integer programming becomes computationally prohibitive. Fifth, we evaluate the proposed method on eight public benchmark datasets and present a case study on a large-scale credit risk dataset, demonstrating competitive or superior performance relative to SLIM, RISKSLIM, logistic regression, LASSO, and decision trees across discrimination, calibration, utility, and sparsity.
\section{Method}
\label{sec:Method}
We start with a dataset of $N$ i.i.d. training samples $D_N=\{(\bm {x}_j,y_j)\}_{j=1}^N$ where $\bm{x}_j \in \mathcal{X} \subseteq \mathbb{R}^P$ denotes a set of predictors $\left[\bm{x}_{j1},\bm{x}_{j2},\dots,\bm{x}_{jP}\right]^\top$ and $y_j \in \mathcal{Y}=\{0,1\}$ denotes a class label. The sample with $y_j=1$ is called a positive sample, and the sample with $y_j=0$ is called a negative sample. We will focus on developing a risk scoring model $c:\mathcal{X} \to [0,1]$ that predicts the probability of a positive sample occurring ($y=1$) according to a set of predictors $\bm{x} \in \mathcal{X}$. We evaluate the binary classification performance of the risk scoring model at different thresholds. Let there be $M$ predefined thresholds, $0 < p_1 < p_2< \cdots < p_M < 1$. For each threshold $p_i$, we define $\mathrm{TP}_i$ and $\mathrm{FP}_i$ as the number of true positives and false positives, respectively, above $p_i$. Specifically, 
\begin{align}\label{eq:TPFP}
    \mathrm{TP}_i = \sum_{j=1}^N {I}\left(c(\bm{x}_j) \ge p_i, y_j=1\right), \quad \mathrm{FP}_i=\sum_{j=1}^N {I}\left(c(\bm{x}_j) \ge p_i, y_j=0\right),
\end{align}
where ${I}(\cdot)$ denotes an indicator function. For convenience, we set $p_0=0$ and $p_{M+1}=1$, and accordingly $\mathrm{TP}_0$ is equal to the number of positive samples $N^+$, $\mathrm{FP}_0$ is equal to the number of negative samples $N^-$, and we define $\mathrm{TP}_{M+1}=\mathrm{FP}_{M+1}=0$.

The goal of this study is to establish a sparse linear risk scoring model with integer coefficients so that the model can achieve the maximum net benefit under given thresholds.  The net benefit is calculated over a range of threshold probabilities, defined as:
\begin{align}
    \mathrm{NB}(p_i;c) = \frac{\mathrm{TP}_i}{N}-\frac{\mathrm{FP}_i}{N}\cdot \frac{p_i}{1-p_i}, \quad i=0,1,\cdots,M.
\end{align}
We learn the values of the coefficients $\bm{\lambda}=[\lambda_1,\cdots,\lambda_P ]^\top \in \mathcal{L}\subseteq \mathbb{R}^P$ and a series of intercepts $\bm{T}=[T_0,T_1,\cdots,T_M]^\top \in \mathcal{B} \subseteq \mathbb{R}^{M+1}$ corresponding to thresholds $\bm{p}=\left[p_0,p_1,\cdots,p_M\right]^\top$ from the training data by solving an optimization problem of the following form:
\begin{equation}\label{ilp:original}
\begin{aligned}
     \min_{\bm{\lambda},\bm{T}} \quad & \displaystyle -\frac{1}{N}\sum_{i=0}^M \omega_i\left({\mathrm{TP}_i} - {\mathrm{FP}_i}\cdot\frac{p_i}{1-p_i} \right) + C_0\|\bm{\lambda}\|_0 &  \\
     \text{s.t.}  \quad & \displaystyle \mathrm{TP}_i=\sum_{j=1}^N {I}\left(\bm{x}_j^\top\bm{\lambda}\ge T_i, y_j=1\right), & i=0,1,\cdots,M,\\
     \displaystyle  & \displaystyle \mathrm{FP}_i=\sum_{j=1}^N {I}\left(\bm{x}_j^\top\bm{\lambda}\ge T_i, y_j=0\right), & i=0,1,\cdots,M, \\
     \displaystyle  &\bm{\lambda}\in\mathcal{L},\bm{T}\in\mathcal{B},&
\end{aligned}    
\end{equation}
where $\omega_i$ is the weight of the net benefit under the threshold $p_i$ satisfying $0 \le \omega_i \le1$, $i=0,1,\cdots, M$, and $\sum_{i=0}^M \omega_i=1$; $C_0>0$ is the penalty factor associated with the $l_0$-norm of $\bm{\lambda}$; $\mathcal{L}$ and $\mathcal{B}$ are two finite sets of integers. 

In practice, once the coefficient set $\mathcal{L}$ is specified and the dataset is given, the range of values for the linear scores is determined. Then, the intercept set $\mathcal{B}$ can be taken as all integers between the minimum and maximum values of the linear combination. Therefore, the main design choice lies in selecting a suitable set of coefficient set $\mathcal{L}$. This choice involves a trade-off between interpretability and computational complexity. Typically, $\mathcal{L}$ is restriced to a small bounded set of integers (e.g., $\{-5,-4,\cdots,5\}$ or $\{-10,-9,\cdots,10\}$), which makes the scoring system simple and easy to use. When prior knowledge is available, additional structure can be imposed on $\mathcal{L}$. For example, if a feature is known to have a positive effect on the outcome, its coefficients can be set to a non-negative value. Furthermore, the range of $\mathcal{L}$ may be guided by the scale of the features, or it can be tuned by validation, i.e., selecting the minimum range that achieves satisfactory performance. The first term in the objective function of the problem (\ref{ilp:original}) represents the weighted sum of the negative net benefit under different thresholds; and the second term represents the $l_0$-norm penalty of the coefficients. Once we obtain the coefficients and intercepts, given a sample $\bm{x}\in \mathcal{X}$, the corresponding risk probability can be computed as:

\begin{align} \label{eq:risk prob}
    c(\bm{x}) = \left\{
    \begin{array}{cl}
    q_0,     &  \text{if }\bm{x}^\top\bm{\lambda}<T_1;  \\
    q_i,     &  \text{if } T_i \le \bm{x}^\top\bm{\lambda} < T_{i+1},\, i=1,2,\cdots,M-1;\\
    q_M,     &  \text{otherwise},
    \end{array}
    \right.
\end{align}
where $q_i \in [p_i,p_{i+1})$ can be given arbitrarily. In \Cref{subsec:calibration&utility}, we present a procedure that improves or preserves d-AUNBC and then assigns the appropriate $q_i$ to ensure that the model is moderately calibrated on the training set.

We show formally in \ref{app:NPhard} that optimization problem (\ref{ilp:original}) is NP-hard. Later in this section, we will use this fact to discuss computational considerations to solve the problem.

\subsection{Illustrative Application to the Give Me Some Credit Dataset} \label{subsec:method_example}
To illustrate how the proposed model operates in practice, we will first present an example using the publicly available Give Me Some Credit dataset \citep{give-me-some-credit-20210326}. This dataset contains information on 150,000 borrowers and is one of the most widely used benchmarks for modeling credit risk. The objective is to develop a risk scoring model that predicts whether a borrower will experience serious delinquency in the next two years. Accurate credit risk assessment is crucial for financial institutions because it helps identify high-risk borrowers and supports lending decisions.

The dataset comprises 11 variables, including a binary response variable and ten explanatory variables. The response variable indicates whether a borrower experience 90 days past due delinquency or worse. The explanatory variables describe characteristics of borrowers, such as age, monthly income, debt ratio, credit utilization, the number of credit accounts and loans, the number of dependents in the family, and several measures of past delinquency.

Prior to model training, the features were preprocessed to improve the data quality. Missing values and invalid observations were imputed using the median of the corresponding features. Extreme values of continuous variables were truncated to reduce the impact of outliers. Moreover, monthly income and debt ratio were log-transformed to improve numerical stability. The transformed monthly income was then rounded, while the transformed debt ratio and revolving credit utilization were rounded to one decimal place. Age, the numbers of open credit lines and loans, the number of real estate loans, and the number of dependents were converted into binary variables. Finally, the three delinquency-related variables (30--59, 60--89, and at least 90 days past due) were combined into a single binary feature indicating whether the borrower had any previous delinquency history.

RSS-DNB was then trained using the preprocessed features. In this application, we set $p_i=\frac{i}{10}$ ($i=0,\ldots,9$), $\mathcal{L}=\{-10,\ldots,10\}$, and $C_0=0$, and solved the optimization problem in (\ref{ilp:original}). The resulting model selected five predictors and produced a sparse additive scorecard with integer coefficients. Table~\ref{tab:system1} presents the point allocation for each selected feature, while Table~\ref{tab:system2} provides the estimated probability of serious delinquency corresponding to each total score. Each lender has an implicit risk preference, which can be interpreted as a decision threshold. For example, a threshold of 10\% means that a lender consider a borrower acceptable when the estimated probability of serious delinquency does not exceed 10\%. Using this scoring system to support lending decisions, a borrower whose total score meets or exceeds 10 points (the corresponding risk exceeds 10\%) would be flagged as a high-risk applicant and may require additional review, whereas borrowers with lower scores would be considered lower risk.

\begin{table}[htbp]
\centering
\caption{Integer-based risk scoring system for predicting serious delinquency}
\label{tab:system1}
\begin{tabular}{llc}
\toprule
\textbf{Predictor} & \textbf{Transformation} & \textbf{Coefficient} \\
\midrule
Age & $I(x \ge 52)$ & -2 \\
Revolving credit utilization & $\text{round}(\min\{x,1.5\},1)$ & 10 \\
Number of open credit lines and loans & $I(x \ge 8)$ & 1\\
Number of real estate loans & $I( x \ge 1 )$ & -1\\
Previous delinquency & $I(x \ge 1)$ & 9 \\
\bottomrule
\end{tabular}
\end{table}

\begin{table}[htbp]
    \centering
    \caption{Predicted probability of serious delinquency according to total risk score}
\begin{tabular}{rccccccc}
    \toprule
    \textbf{Total Score} & $\le$ 9 & 10--14 & 15--17 & 18--19 & 20 &21--24 & $\ge$ 25 \\
    \midrule
    \textbf{Predicted Risk}     & 2.8\% &14.0\% &26.8\% &37.1\% &48.9\% & 53.4\% & 90.0\% \\
    \bottomrule
    \end{tabular}   
    \label{tab:system2}
\end{table}

This illustrative application is intended to provide insight into the model-building process and the characteristics of the resulting scoring system, rather than to evaluate predictive performance or decision utility. A comprehensive evaluation is provided as a case study in Section \ref{sec:application}.

Next, we aim to clarify whether minimizing the negative weighted sum of net benefits across different thresholds can yield a model that exhibits a high level of both discrimination and calibration. To this end, we investigate how net benefit relates to both discrimination and calibration.
\subsection{Discrimination and Utility} \label{subsec:discrimination&utility}
We use the area under the receiver operating characteristic curve (AUROC) to evaluate the model discrimination. The conventional AUROC is calculated from the complete ranking of prediction scores and can equivalently be interpreted as the probability that a randomly selected positive sample receives a higher score than a randomly selected negative sample. It is a widely used metric for assessing the performance of binary classifiers, reflecting the trade-off between the true positive rate and the false positive rate across varying threshold values. A higher AUROC value indicates better discriminative ability.

Similarly, we use decision curve analysis (DCA) to evaluate decision utility. In DCA, a decision curve plots the relationship between net benefit of a prediction model and the threshold, which we also refer to as the net benefit curve. Following the approach proposed by \citet{talluriUsingWeightedArea2016}, we use the area under the net benefit curve (AUNBC) to summarize the decision utility over a range of thresholds.

However, to establish a theoretical relationship between discrimination and utility, we use discretized versions of both measures, evaluated on the same prespecified threshold grid. Geometrically, the discretized AUROC is constructed as a rectangular sum under the stepwise ROC curve induced by threshold grid, whereas the discretized AUNBC is a Riemann-sum approximation to the continuous version. An intuitive interpretation of AUNBC and an illustrative example are provided in \ref{app:interpretation_aunbc}. We refer to these two threshold-grid functionals as the discretized AUROC (d-AUROC) and the discretized AUNBC (d-AUNBC), denoted by $\mathrm{AUROC}(\bm p;c)$ and $\mathrm{AUNBC}(\bm p; c)$, respectively. Specifically, we define them as follows.

\begin{definition}\label{def:auroc}
    (d-AUROC) Let $\bm p = \left[p_0,p_1,\cdots,p_{M+1}\right]^\top$ denote a fixed threshold grid satisfying $0=p_0 \le p_1 \le \cdots\le p_{M+1}=1$. For a risk model $c:\mathcal{X}\to [0,1]$, define
    \begin{align}
        \mathrm{AUROC}(\bm p; c) = \frac{1}{N^+N^-} \sum_{i=0}^M \left(\mathrm{FP}_i-\mathrm{FP}_{i+1}\right)\cdot \mathrm{TP}_i,
    \end{align}
    where $\mathrm{TP}_i$ and $\mathrm{FP}_i$ is defined in Equations (\ref{eq:TPFP}).
\end{definition}

\begin{definition}\label{def:aunbc}
    (d-AUNBC) Let $\bm p = \left[p_0,p_1,\cdots,p_{M+1}\right]^\top$ denote a fixed threshold grid satisfying $0=p_0 \le p_1 \le \cdots\le p_{M+1}=1$. For a risk model $c:\mathcal{X}\to [0,1]$, define
    \begin{align}
        \mathrm{AUNBC}(\bm p; c) = \frac{1}{N} \sum_{i=0}^M \left(p_{i+1} - p_i\right) \left(\mathrm{TP}_i - \mathrm{FP}_i\cdot \frac{p_i}{1-p_i}\right).
    \end{align}
\end{definition}
\begin{remark}
The d-AUNBC is a special case of the weighted sum of net benefits across different thresholds, where the weights are given by $\omega_i=p_{i+1}-p_i$, $i=0,1,\cdots,M$. 
\end{remark}

Our first two results show a fundamental relationship between d-AUROC and d-AUNBC. For a given threshold grid $\bm p$, the upper bound of $\mathrm{AUNBC}(\bm p; c)$ is a monotonically increasing function of $\mathrm{AUROC}(\bm p; c)$, while its lower bound is independent of $\mathrm{AUROC}(\bm p; c)$.
\begin{theorem}\label{thm:auroc_aunbc}
    Let $\bm p = \left[p_0,p_1,\cdots,p_{M+1}\right]^\top$ denote a fixed threshold grid satisfying $0=p_0 \le p_1 \le \cdots\le p_{M+1}=1$, $P_k=\sum_{i=1}^k \frac{(p_{i+1}-p_i)p_i}{1-p_i}$, $k=1,2,\cdots,M$, then for a risk model $c:\mathcal{X}\to [0,1]$,
    \begin{align}
        a_0 p_1 - (1-a_0)P_M \le \mathrm{AUNBC}(\bm p; c) \le \max_{1\le k \le M} A_k(\mathrm{AUROC}(\bm p; c);a_0),
    \end{align}
    where $a_0=\frac{N^+}{N}$, and $A_k(\cdot;a_0)$ is a function defined in $[0,1]$, satisfying
    \begin{align} \label{eq:Ak}
        A_k(x;a_0)=\left\{\begin{array}{cc}
        -a_0(1-p_k)(1-x) + a_0 -(1-a_0)P_k,     &  \text{if } x \le 1-\frac{(1-a_0)P_k}{(1-p_k)a_0}; \\
        a_0 - 2\sqrt{P_k(1-p_k)a_0(1-a_0)(1-x)},     & \text{otherwise}.
        \end{array}\right.
    \end{align}
Moreover, these bounds are tight.
\end{theorem}
\begin{proof}
    See the appendix.
\end{proof}
The next corollary shows that the lower bound of $\mathrm{AUROC}(\bm p; c)$ is a monotonically increasing function of $\mathrm{AUNBC}(\bm p; c)$, whereas the upper bound of $\mathrm{AUROC}(\bm p; c)$ is independent of $\mathrm{AUNBC}(\bm p; c)$.
\begin{corollary}\label{cor:auroc_aunbc}
    Let $\bm p = \left[p_0,p_1,\cdots,p_{M+1}\right]^\top$ denote a fixed threshold grid satisfying $0=p_0 \le p_1 \le \cdots\le p_{M+1}=1$, $P_k=\sum_{i=0}^k \frac{(p_{i+1}-p_i)p_i}{1-p_i}$, $k=1,2,\cdots,M$, then
    \begin{align}\label{eq:auroc}
        1 \ge \mathrm{AUROC}(\bm p; c) \ge \max\left\{\min_{1\le k \le M} B_k(\mathrm{AUNBC}(\bm p; c);a_0), 0\right\},
    \end{align}
    where $a_0=\frac{N^+}{N}$, and $B_k(\cdot;a_0)$ is a function defined on $\left[a_0p_1-(1-a_0)P_M,a_0\right]$, satisfying 
    \begin{align} \label{eq:Bk}
        B_k(y;a_0) = \left\{
        \begin{array}{cc}
            1-\frac{a_0-y-(1-a_0)P_k}{a_0(1-p_k)},  & \text{if } y \le a_0-2(1-a_0)P_k; \\
            1-\frac{(a_0-y)^2}{4P_k(1-p_k)a_0(1-a_0)}, & \text{otherwise}.
        \end{array}
        \right.
    \end{align}    
Moreover, these bounds are tight.
\end{corollary}
\begin{proof}
    See the appendix.
\end{proof}

Based on the discretized measures, we have now demonstrated that a high level of discrimination does not necessarily imply high utility -- as shown in previous studies \citep{albaDiscriminationCalibrationClinical2017,vickersSimpleStepbystepGuide2019} -- while high utility guarantees a correspondingly high level of discrimination. \Cref{fig:auroc_aunbc_1} visualizes the relationship between d-AUROC and d-AUNBC as stated in Theorem \ref{thm:auroc_aunbc}. To further illustrate this relationship, we consider the example in \Cref{subsec:method_example} and plot the curve under this setting in \Cref{fig:auroc_aunbc_2}. In addition to the risk scoring system (RSS-DNB) obtained in \Cref{subsec:method_example}, we also generate two types of synthetic predictions (\ref{app:additional}, \Cref{alg:synthetic_1,alg:synthetic_2}), and compute the values of d-$\mathrm{AUROC}$ and d-$\mathrm{AUNBC}$. The first type consists of random synthetic predictions. The resulting (d-$\mathrm{AUROC}$,d-$\mathrm{AUNBC}$) pairs all lie below the curve, which provides empirical support for \Cref{thm:auroc_aunbc}.
The second type consists of synthetic predictions constructed according to the proof of \Cref{thm:auroc_aunbc}. The resulting (d-$\mathrm{AUROC}$,d-$\mathrm{AUNBC}$) pairs all lie exactly on the curve, indicating that the bound established in the theorem is tight. Likewise, when d-$\mathrm{AUNBC}$ is plotted on the horizontal axis and d-$\mathrm{AUROC}$ on the vertical axis, the conclusion in \Cref{cor:auroc_aunbc} can be verified. 

Next, we will explore the relationship between calibration and utility.

\begin{figure}[htbp]
\centering
\begin{subfigure}{0.45\textwidth}
  \centering
  \includegraphics[width=\linewidth]{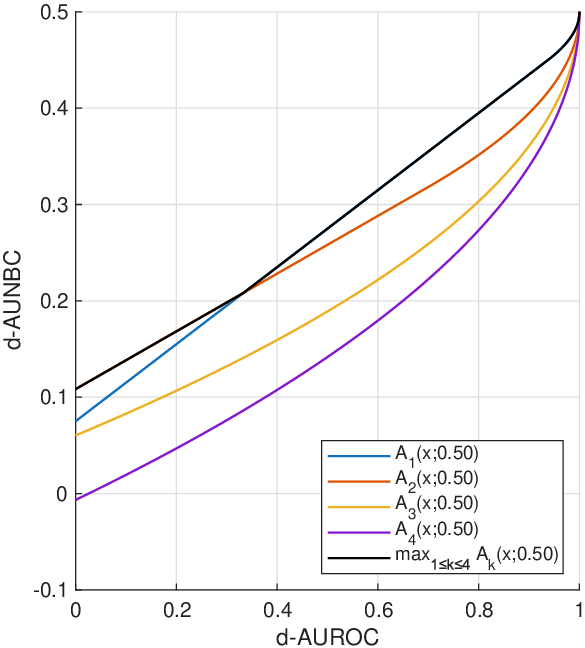}
  \caption{The theoretical relationship.}
  \label{fig:auroc_aunbc_1}
\end{subfigure}
\hfill
\begin{subfigure}{0.45\textwidth}
  \centering
  \includegraphics[width=\linewidth]{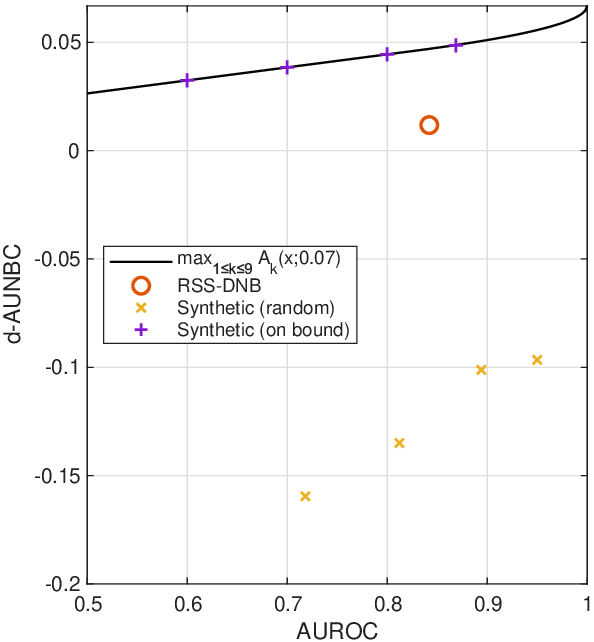}
  \caption{Example with RSS-DNB and synthetic predictions.}
  \label{fig:auroc_aunbc_2}
\end{subfigure}
\caption{Relationship between d-$\mathrm{AUROC}$ and d-$\mathrm{AUNBC}$. (a) The theoretical relationship, generated with $M=4$, $a_0=0.5$, and $p_i=i/5$ for $i=0,1,\cdots,4$; (b) The relationship under the setting of \Cref{subsec:method_example} with RSS-DNB model and two types of synthetic predictions. In this example, AUROC denotes conventional AUROC rather than d-AUROC, while the solid curve shows the theoretical boundary derived for d-AUROC and d-AUNBC.}
\end{figure}
\subsection{Calibration and Utility} \label{subsec:calibration&utility}

According to the work of \citet{vancalsterCalibrationHierarchyRisk2016}, we know that a risk model is moderately calibrated if the observed event rate is equal to the predicted risk. For example, if a group of patients is assigned a 10\% risk of disease by a moderately calibrated model, then approximately 10\% of those patients will actually develop the disease. First, we will show that if the model is underestimated (predicted risk $<$ observed event rate) or overestimated (predicted risk $>$ observed event rate) in a subgroup of the model population, we can modify the model to obtain a better one with respect to d-AUNBC.
\begin{theorem} \label{thm:improve_aunbc}
Let $0 = p_0 < p_1 < \cdots < p_{M} < p_{M+1} = 1$ be a sequence of thresholds and $c:\mathcal{X}\to [0,1]$ be a risk model. Define the intervals $\mathcal{S}_i = [p_i,p_{i+1})$ for $i=0,1,\cdots,M-1$ and $\mathcal{S}_M=[p_M,1]$. For each $i=0,1,\cdots,M$, let $N_i$ and $O_i$ denote, respectively, the number of samples and the number of positive samples whose predicted risk by $c$ lies in the interval $\mathcal{S}_i$. If there exists some $0 \le k \le M$ such that $p_{k+1}N_k<O_k$, then the modified model  $c_k':\mathcal{X} \to [0,1]$ defined by
\begin{equation}
    \begin{aligned}
        c_k'(\bm{x}):=\left\{
        \begin{array}{cc}
        p_{k+1}, & \text{if }c(\bm {x}) \in \mathcal{S}_k;\\
        c(\bm {x}),    & \text{otherwise},
        \end{array}
        \right.
    \end{aligned}
\end{equation}
achieves a strictly higher $\mathrm{d}$-$\mathrm{AUNBC}$ than $c$.

Similarly, if for some $0 \le k \le M$ we have $p_kN_k > O_k$, then the modified model $c_k'':\mathcal{X} \to [0,1]$ defined by
\begin{equation}
    \begin{aligned}
        c_k''(\bm{x}):=\left\{
        \begin{array}{cc}
        p_{k-1}, & \text{if }c(\bm {x}) \in {[p_k,p_{k+1})}\mathcal{S}_k;\\
        c(\bm {x}),    & \text{otherwise},
        \end{array}
        \right.
    \end{aligned}
\end{equation}
also achieves a strictly higher $\mathrm{d}$-$\mathrm{AUNBC}$ than $c$.
\end{theorem}
\begin{proof}
    See the appendix.
\end{proof}
\begin{remark}
    The modified model may collapse different predictions into the same level. Informally, the modified model can be further optimized to preserve the original ordering of predictions, for example, by adding a small fraction (e.g., 1\%) of the original scores. This adjustment maintains the order of the predictions at the cost of a negligible loss in calibration and model utility.
\end{remark}

Based on Theorem \ref{thm:improve_aunbc}, we can derive an algorithm to improve the d-AUNBC of any risk model (Algorithm \ref{alg:improve_aunbc}).
\IncMargin{1em}
\begin{algorithm}\SetKwInOut{Input}{Input}\SetKwInOut{Output}{Output }
    \caption{Improve the d-AUNBC}\label{alg:improve_aunbc}
        \KwData{$\{(\bm{x}_j,y_j)\}_{j=1}^N \subseteq \mathbb{R}\times \{0,1\}$}
        \Input{Thresholds $0=p_0 < p_1<\ldots<p_M<p_{M+1}=1$, and model $c:\mathcal{X}\to[0,1]$ }
        \Output{Model $c^*$ with improved d-AUNBC}
        $c^* \leftarrow c$\;
        \For{$i \leftarrow 0$ to $M-1$}{
            ${O}_i \gets \sum_{j=1}^N I\left(c^*(\bm x_j) \in \mathcal{S}_i, y_j=1 \right)$\; \tcp{We use the intervals $\mathcal{S}_i$ defined in \Cref{thm:improve_aunbc}}
            ${N}_i \gets \sum_{j=1}^N I\left(c^*(\bm x_j) \in \mathcal{S}_i \right)$\; 
            \If{$O_i > p_{i+1} N_i $}{$c^*(\bm x) \leftarrow c^*(\bm x) + I(c^*(\bm x_j) \in \mathcal{S}_i) \cdot (p_{i+1} - c^*(\bm x))$ \;}
        }
        \For{$i \leftarrow M$ to $1$}{
            ${O}_i \gets \sum_{j=1}^N I\left(c^*(\bm x_j) \in \mathcal{S}_i, y_j=1 \right)$\; 
            ${N}_i \gets \sum_{j=1}^N I\left(c^*(\bm x_j) \in \mathcal{S}_i \right)$\; 
            \If{$O_i < p_{i} N_i $}{$c^*(\bm x) \leftarrow c^*(\bm x) + I(c^*(\bm x) \in \mathcal{S}_i) \cdot (p_{i-1} - c^*(\bm x))$ \;}
        }
\end{algorithm}

The following proposition established a key property of the model obtained by \Cref{alg:improve_aunbc}.
\begin{proposition} \label{prop:property of the inproved model}
    Let $0 = p_0 < p_1 < \cdots < p_{M} < p_{M+1} = 1$ be a sequence of thresholds and $c:\mathcal{X}\to [0,1]$ be a risk model. Suppose that $c^*$ is obtained by \Cref{alg:improve_aunbc}. Define the intervals $\mathcal{S}_i = [p_i,p_{i+1})$ for $i=0,1,\cdots,M-1$ and $\mathcal{S}_M=[p_M,1]$. For each $i=0,1,\cdots,M$, let $N_i^*$ and $O_i^*$ denote, respectively, the number of samples and the number of positive samples whose predicted risk by $c^*$ lies in the interval $\mathcal{S}_i$. Then, for each $i=0,1,\cdots, M$,
    \begin{equation}
     N_i^* p_i \le O_i^* \le N_i^* p_{i+1}. 
    \end{equation}
\end{proposition}
\begin{proof}
    See the appendix. 
\end{proof}
The next corollary shows that any model satisfying the property established in \Cref{prop:property of the inproved model}, moderate calibration on the training set can be achieved by assigning appropriate prediction probabilities. Since every model maximizing the d-AUNBC must satisfy the property in \Cref{prop:property of the inproved model}; otherwise, \Cref{thm:improve_aunbc} would imply the existence of another model with strictly higher d-AUNBC, a contradiction. Therefore, the corollary applies in particular to d-AUNBC-maximizing model.

\begin{corollary} \label{cor:calibration}
Let $0 = p_0 < p_1 < \cdots < p_{M} < p_{M+1} = 1$ be a sequence of thresholds and $c^*:\mathcal{X} \to [0,1]$ be a risk model. Define the intervals $\mathcal{S}_i = [p_i,p_{i+1})$ for $i=0,1,\cdots,M-1$ and $\mathcal{S}_M=[p_M,1]$. For each $i=0,1,\cdots,M$, let $N_i^*$ and $O_i^*$ denote, respectively, the number of samples and the number of positive samples whose predicted risk by $c^*$ lies in the interval $\mathcal{S}_k$. Assume that for each $i=0,1,\cdots,M$, $N_i^* p_i \le O_i^* \le N_i^* p_{i+1}$. Then at least one of the following two statements holds:
\begin{itemize}
    \item[(1)] There exist real numbers $q_0,\dots,q_M$ such that $q_i \in \mathcal{S}_i$ satisfying $q_iN_i^*= O_i^*$, $i=0,1,\cdots,M$.
    \item[(2)] There exists another model $c'$ with the same $\mathrm{d}$-$\mathrm{AUNBC}$ as $c^*$, in which case conclusion (1) holds.
\end{itemize}

\end{corollary}
\begin{proof}
    See the appendix.
\end{proof}
Corollary \ref{cor:calibration} implies that, a model satisfying the property in \Cref{prop:property of the inproved model} can be adjusted to achieve moderate calibration on the training set while preserving d-AUNBC. \Cref{alg:moderate_calibration} provides an explicit procedure for obtaining such a calibrated model. The first step of \Cref{alg:moderate_calibration} is applying \Cref{alg:improve_aunbc}, which directly improves or preserves d-AUNBC and established \Cref{prop:property of the inproved model}. The subsequent assignment of probabilities achieves calibration while preserving the d-AUNBC. 

\begin{algorithm}\SetKwInOut{Input}{Input}\SetKwInOut{Output}{Output }
    \caption{Achieve Moderate Calibration}\label{alg:moderate_calibration}       
        \KwData{$\{(\bm{x}_j,y_j)\}_{j=1}^N \subseteq \mathbb{R}\times \{0,1\}$}
        \Input{Thresholds $0=p_0 < p_1<\ldots<p_M<p_{M+1}=1$, and model $c:\mathcal{X}\to[0,1]$ }
        \Output{Moderate calibrated model $c^*$ with improved d-AUNBC}
        Apply \Cref{alg:improve_aunbc} to obtain $c^*$ \;
        \For{$i \gets 0$ to $M$}{
            ${O}_i^* \gets \sum_{j=1}^N I\left(c^*(\bm x_j) \in \mathcal{S}_i, y_j=1 \right)$\; \tcp{We use the intervals $\mathcal{S}_i$ defined in \Cref{thm:improve_aunbc}}
            ${N}_i^* \gets \sum_{j=1}^N I\left(c^*(\bm x_j) \in \mathcal{S}_i \right)$\;
            \eIf{$O_i^* = p_{i+1}N_i^*$}{
            $c^*(\bm x) \gets c^*(\bm x) + I(c^*(\bm x) \in \mathcal{S}_i) \cdot (p_{i+1} - c^*(\bm x))$\;
            $q_i \gets p_i$\; \tcp{The value assigned to $q_i$ here is not used by the algorithm, but it is made for the convenience of the subsequent analysis.}
            }
            {
            $q_i \gets O_i^* / N_i^*$ \;
            $c^*(\bm x) \gets c^*(\bm x) + I(c^*(\bm x) \in \mathcal{S}_i) \cdot (q_i - c^*(\bm x))$
            }
        }    
\end{algorithm}

Consequently, the Hosmer--Lemeshow (HL) statistic or the expected calibration error (ECE) \citep{pakdamannaeiniObtainingWellCalibrated2015} of the model will be zero {under the grouping $[0,p_1),[p_1,p_2),\cdots,[p_M,1]$.} Specifically, the HL statistic is defined as:
\begin{equation}
HL = \sum_{i=0}^{M} \frac{(O_i - E_i)^2}{E_i(1-E_i/N_i)},
\end{equation}
where $E_i$ is the expected number of events equals to $q_iN_i$. Under moderate calibration, the observed number of events $O_i$ matches the expected number of events $E_i$ in each group. Similarly, the ECE is defined as
\begin{equation}
ECE = \sum_{\substack{i=0 \\ N_i \not = 0}}^M \frac{N_i}{N} \left| \frac{O_i}{N_i} - e_i \right|,
\end{equation}
where $e_i$ is the mean of the probabilities for the instance in group $i$, which in this case, is equal to the probability $q_i$. Under moderate calibration, the observed event rate equal to the predicted probability in each group, leading to $ECE = 0$. Furthermore, even if the grouping setting is changed and then re-evaluate HL statistic and ECE on the training set, these two values remain zero.

We will use ECE as a measure of calibration in this work, as it provides a direct quantification of calibration performance. In contrast, the HL statistic is known to be sensitive to sample size, which may lead to misleading assessments of calibration.

\begin{figure}
    \centering
    \includegraphics[width=0.8\linewidth]{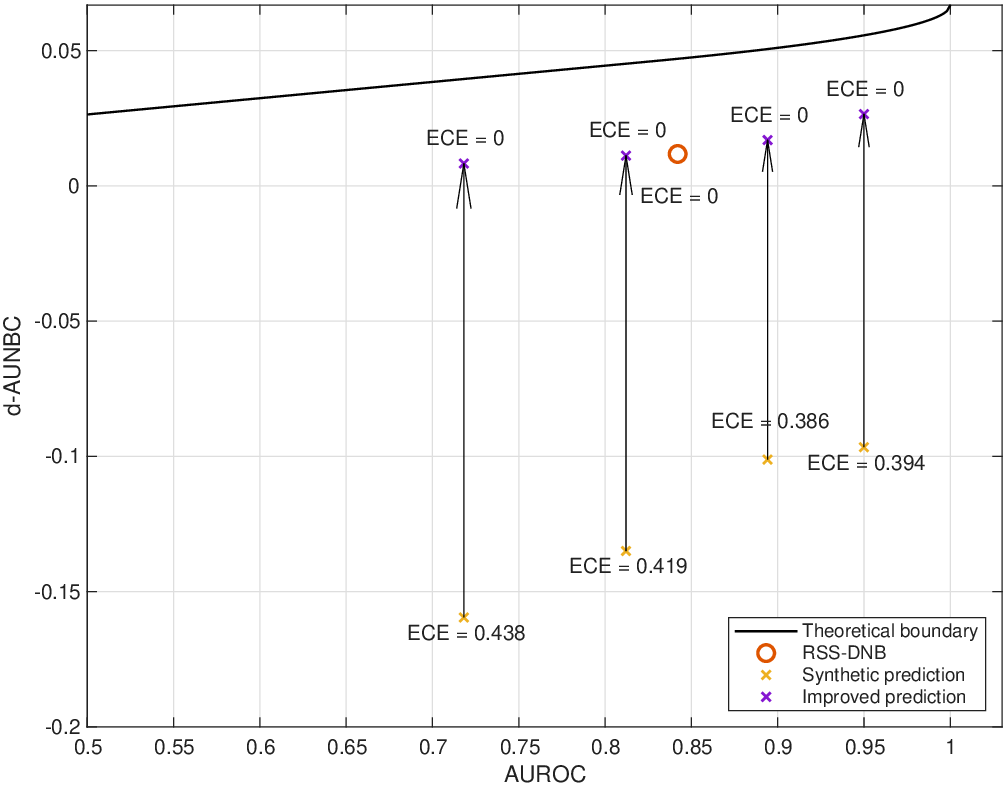}
    \caption{Changes in d-AUNBC, AUROC, and ECE for the random synthetic predictions in \Cref{subsec:method_example} before and after applying \Cref{alg:moderate_calibration}. In this example, d-AUNBC increases and ECE decreases, while conventional AUROC remains unchanged. Arrows indicate the direction of improvement. The solid curve represents the theoretical boundary and the circle denotes the proposed model.}
    \label{fig:improved_aunbc}
\end{figure}

\begin{proposition}\label{prop:calibration}
    Let $0 = p_0 < p_1 < \cdots < p_{M} < p_{M+1} = 1$ be a sequence of thresholds and $c:\mathcal{X}\to [0,1]$ be a risk model. Let $c^*$ denoted the calibrated model obtained from \Cref{alg:moderate_calibration}. Suppose that $\delta > 0$ is sufficiently small, and define $\tilde c(\bm x) = \min\{c^*(\bm x)+\delta \cdot c(\bm x), 1\}$. Then
    \begin{equation}
        \left | \mathrm{ECE}(\tilde c) - \mathrm{ECE}(c^*) \right | \le \delta,\\
    \end{equation}
    \begin{equation}
        \mathrm{AUNBC}(\bm p; \tilde c) = \mathrm{AUNBC}(\bm p; c^*),
    \end{equation}
    and for any $1 \le j_1 \not = j_2\le N$, if $\tilde c(\bm x_{j_1})<1$, $\tilde c(\bm x_{j_2}) < 1$, 
    \begin{equation}
        c(\bm x_{j_1})< c(\bm x_{j_2}) \iff\tilde c(\bm x_{j_1})<\tilde c(\bm x_{j_2}).
    \end{equation}
\end{proposition}
\begin{proof}
    See the appendix.
\end{proof}

We continue to use example in \Cref{subsec:method_example}. For the randomly generated prediction, we apply \Cref{alg:improve_aunbc} (maintain the order of the predictions) to obtain improved prediction, and plot corresponding (AUROC,d-AUNBC) pairs in \Cref{fig:improved_aunbc}. As can be observed, the proposed algorithm significantly improves d-AUNBC while maintaining AUROC and achieves near-perfect calibration (with ECE reduced to approximately zero).

The above results indicate that a risk model d-AUNBC can be adjusted to achieve moderate calibration on the training set while improving or preserving d-AUNBC. Specifically, if a model does not satisfy the property in \Cref{prop:property of the inproved model}, its d-AUNBC can be improved through simple modifications. Then, we will introduce how to solve problem (\ref{ilp:original}).
\subsection{Integer Programming Formulation}
\begin{equation}\label{ilp:formulation}
\begin{aligned}
\min_{\bm{\lambda},\bm{\varphi},\bm{\alpha},\bm{T}} \quad & -\frac{1}{N}\sum_{i=0}^M\sum_{j\in \mathcal{J}^+} \omega_i\varphi_{i,j}+\frac{1}{N}\sum_{i=0}^M \sum_{j\in \mathcal{J}^-}\frac{\omega_i p_i}{1-p_i}\varphi_{i,j} + C_0 \sum_{k=1}^P \alpha_k & \\
\text{s.t.} \quad & H_{j} (1-\varphi_{i,j}) \geq T_i - \sum_{k=1}^{P} x_{j,k} \lambda_k, && i = 0, \ldots, M,\, j \in \mathcal{J}^+, \\
& H_{j} \varphi_{i,j} \geq \gamma + \sum_{k=1}^{P} x_{j,k} \lambda_k - T_i, && i = 0, \ldots, M,\ j \in \mathcal{J}^{-}, \\ 
& -\Lambda_k \alpha_k \leq \lambda_k \leq \Lambda_k \alpha_k, && k = 1, \ldots, P,\\
& T_i \leq T_{i+1}, && i = 0, \ldots, M-1, \\
& \lambda_k \in \mathcal{L}_k, && k = 1, \ldots, P, \\
& \varphi_{i,j} \in \{0,1\}, && i = 0, \ldots, M, \,j = 1,\cdots,N, \\
& \alpha_k \in \{0,1\}, && k = 1, \ldots, P, \\
& T_i \in \mathcal{B}_0, && i = 0, \ldots, M.
\end{aligned}   
\end{equation}

In formulation (\ref{ilp:formulation}), $\varphi_{i,j}$ denotes a binary decision variable, where $\varphi_{i,j}=1$ if the $j$-th sample is predicted as positive under the threshold $p_i$, and $\varphi_{i,j}=0$ otherwise. Specifically, for positive samples, the first constraint forces $\varphi_{i,j}=0$ whenever $\bm x_{j}^{\top}\bm \lambda<T_i$. When $\bm x_{j}^{\top}\bm \lambda \ge T_i$, both $\varphi_{i,j}=0$ and $\varphi_{i,j}=1$ are feasible; however, minimizing the objective function implies that $\varphi_{i,j}=1$ is optimal. For negative samples, the second constraint forces $\varphi_{i,j}=1$ whenever $\bm x_{j}^{\top}\bm \lambda \ge T_i$. When $\bm x_{j}^{\top}\bm \lambda + \gamma \le T_i$, $\gamma$ is a sufficiently small positive number, then both $\varphi_{i,j}=0$ and $\varphi_{i,j}=1$ are feasible; however, minimizing the objective function implies that $\varphi_{i,j}=0$ is optimal. Therefore, $\varphi_{i,j} = I(\bm x_j^\top\bm\lambda \ge T_i)$ for both positive and negative samples. $J^+$ and $J^-$ denote the index sets of all positive and negative samples, $J^+=\left\{j:y_j=1\right\}$ and $J^-=\left\{j:y_j=0\right\}$, respectively. Here, $H_{j}$ is a large positive constant, which can be set as $H_{j} = \max_{\bm\lambda \in \mathcal{L},T_i\in \mathcal{B}_0} \left\{\gamma + \sum_{k=1}^{P} x_{j,k} \lambda_k - T_i\right\}$. We used $T_i \le T_{i+1}$ to preserve the ordering of score interval. Strict inequalities are not necessary. If $T_i = T_{i+1}$, the corresponding interval will be empty and contains no samples. This does not affect validity of the formulation.
The binary decision variable $\alpha_k$ indicates whether the coefficient $\lambda_k$ is nonzero, specifically,
$\alpha_k = 1$ if $\lambda_k \neq 0$, and $\alpha_k = 0$ otherwise. $\Lambda_k$ is the maximum value that $|\lambda_k|$ can reach. We use $\mathcal L_k:= \{\lambda \in \mathbb{Z} : -\Lambda_k \le \lambda \le \Lambda_k\}$ to represent the set of all possible values of $\lambda_k$, and $\mathcal{B}_0:= \{T \in \mathbb{Z} : -T_{\mathrm{max}} \le T \le T_{\mathrm{max}}\}$ to represent the set of all possible values of $T_i$.

In \ref{app:NPhard}, we have shown that our model belongs to a family of NP-hard problems, and the size of the model increases with $N$. Therefore, in addition to solving (\ref{ilp:formulation}) using the solver Gurobi, we propose a simulated annealing algorithm to efficiently search for high-quality solutions by directly optimizing the original formulation (\ref{ilp:original}). The detailed outline of the approach is presented as \Cref{alg:sa} in the appendix. There, \Cref{alg:find_opt_T} provides a subroutine invoked within \Cref{alg:sa}. For predictive models whose outputs are not in the interval $[0,1]$, \Cref{alg:find_opt_T} can also be used to select the appropriate operating points to maximize the d-AUNBC of the model.

\subsection{Learning Capacity and Generalization}
In Problem (\ref{ilp:original}), we restrict the coefficients and intercepts of the linear model to finite sets of integers to simplify the complexity of the model. We will demonstrate that, under appropriate conditions, the proposed model attains a learning capacity comparable to that of general linear models, and we will further derive its generalization bound. A similar conclusion can be found in the work of \citet{ustunSupersparseLinearInteger2016} on SLIM. Following their approach, we adapt these results to our model.
\begin{theorem}[Learning Capacity]\label{thm:learning_capacity}
    Let $\bm{\rho}=\left[\rho_{1}, \cdots, \rho_{P}\right]^\top \in \mathbb{R}^{P}$ denote the coefficients of the baseline linear classifier trained using data $\mathcal{D}_{N}=\left\{\left(\bm{x}_{j}, y_{j}\right)\right\}_{j=1}^{N}$ and $\bm{t}=\left[t_{0}, t_{1}, \ldots, t_{M}\right]^\top \in \mathbb{R}^{M+1}$ satisfy $-\max _{1 \leq j \leq N}\left|\bm{x}_{j}^\top \bm{\rho}\right| \leq t_{0} \leq t_{1} \leq \cdots \leq t_{M} \leq \max _{1 \leq j \leq N}\left|\bm{x}_{j} ^\top \bm{\rho}\right|$ denote intercepts at different risk thresholds $0= p_{0}<p_{1}<\cdots<p_{M}<1$. Let  $\|X\|_{\infty}:=\max _{1 \leq j \leq N}\left\|\bm{x}_{j}\right\|_{1}$ and $\gamma_{\mathrm {min }}:=\frac{\min _{i, j}\left|\bm{x}_{j} ^\top\bm{\rho} -t_{i}\right|}{\|\bm{\rho}\|_{\infty}}$. Consider training a linear classifier with coefficient $\bm{\lambda}= \left[\lambda_{1}, \ldots, \lambda_{P}\right]^\top \in \mathcal{L}=\{-\Lambda,-\Lambda+1, \ldots, \Lambda-1, \Lambda\}^{P}$ and intercept $\bm{T}=\left[T_{0}, T_{1}, \ldots, T_{M}\right]^\top \in \mathcal{B}=\left\{-T_{\max }, -T_{\max }+1,\cdots, T_{\max }-1,T_{\max }\right\}^{M+1}$ at thresholds $p_{0}, p_{1}, \cdots, p_{M}$. If  $\gamma_{\min }>0$, $\Lambda>\frac{\|X\|_{\infty}+1}{2 \gamma_{\min }}$, and $T_{\max } \geq\left\lceil\Lambda\|X\|_{\infty}\right\rceil$, then there exist $\bm\lambda \in \mathcal{L}$ and $\bm{T} \in \mathcal{B}$ such that
\begin{equation}
\begin{aligned}
    &\sum_{i=0}^{M}\left(\omega_{i} \sum_{j=1}^{N} I\left( \bm{x}_{j}^\top\bm{\lambda}\ge T_{i},y_j=1\right)
    -\frac{\omega_{i}p_i}{1-p_{i}} \sum_{j=1}^{N} I\left(\bm{x}_{j}^\top\bm{\lambda} \geq T_{i}, y_{j}=0\right)\right) \\
    \ge &\sum_{i=0}^{M}\left(\omega_{i} \sum_{j=1}^{N} I\left( \bm{x}_{j}^\top\bm{\rho}\ge t_{i},y_j=1\right)
    -\frac{\omega_{i}p_i}{1-p_{i}} \sum_{j=1}^{N} I\left( \bm{x}_{j} ^\top\bm{\rho} \geq t_{i}, y_{j}=0\right)\right).
\end{aligned}
\end{equation}
\end{theorem}
\begin{proof}
    See the appendix.
\end{proof}

These findings suggest that, as long as $\Lambda$ and $T_{\max}$ are sufficiently large, the coefficients and intercepts of any linear model can be converted to integers without reducing the weighted sum of net benefits. It should be noted that the bound given in \Cref{thm:learning_capacity} is a sufficient condition, but not a necessary one. Consequently, the theorem is intended to provide a qualitative expressiveness result rather than a practical guideline for selecting coefficient bounds. Furthermore, \Cref{cor:learning_capacity} indicates that if we are willing to sacrifice a small amount of net benefit, it is possible to choose relatively smaller values for $\Lambda$ and $T_{\max}$.
\begin{corollary}\label{cor:learning_capacity}
    Let $\bm{\rho} = [\rho_1, \ldots, \rho_p]^\top \in \mathbb{R}^p$ denote the coefficients of a baseline linear classifier trained using data $\mathcal{D}_N = \{(x_j, y_j)\}_{j=1}^N$ and $\bm{t} = [t_0, t_1, \ldots, t_M]^\top \in \mathbb{R}^{M+1}$ satisfy $-\max_{1 \leq j \leq N} |\bm x_j ^\top\bm \rho| \leq t_0 \leq t_1 \leq \cdots \leq t_M \leq \max_{1 \leq j \leq N} |\bm x_j^\top \bm \rho|$ denote intercepts at different risk thresholds $0 = p_0 < p_1 < \cdots < p_M < 1$. Let $\gamma_{(k)}$ denote the $k$-th smallest value in $\left\{\min_{0\le i \le M}\frac{|\bm x_j ^\top\bm \rho - t_i|}{\| \bm \rho \|_\infty}\right\}_{1 \le j \le N}$, $\mathcal{J}_{(k)} := \left\{ j: \min_{0\le i \le M}\frac{|\bm x_j ^\top\bm \rho - t_i|}{\| \bm \rho \|_\infty}  \ge \gamma_{(k)},\ j = 1, 2, \ldots, N \right\}$, $\| \bm x \|_{(k),\infty} := \max_{j \in \mathcal{J}_{(k)}} \| \bm x_j \|_1$. Consider training a linear classifier with coefficients $\bm \lambda = [\lambda_1, \ldots, \lambda_P]^\top \in \mathcal L = \{ -\Lambda_{(k)}, -\Lambda_{(k)}+1, \ldots, \Lambda_{(k)}-1, \Lambda_{(k)} \}^P$ and intercepts $\bm T = [T_0, T_1, \ldots, T_M]^\top \in \mathcal{B} = \left\{-T_{(k) }, -T_{(k)) }+1,\cdots, T_{(k) }-1,T_{(k) }\right\}^{M+1}$ at risk thresholds $p_0, p_1, \ldots, p_M$. If $\gamma_{(k)} > 0$, $\Lambda_{(k)} > \frac{\| x \|_{(k),\infty} + 1}{2\gamma_{(k)}}$, and $T_{(k)} \geq \left\lceil \Lambda_{(k)} \| \bm x \|_{(k),\infty} \right\rceil$, then there exist $\bm \lambda \in \mathcal L$ and intercepts $\bm T \in \mathcal{B}$ such that
\begin{equation}
\begin{aligned}
    &\sum_{i=0}^{M}\left(\omega_{i} \sum_{j=1}^{N} I\left( \bm{x}_{j}^\top\bm{\lambda}\ge T_{i},y_j=1\right)
    -\frac{\omega_{i}p_i}{1-p_{i}} \sum_{j=1}^{N} I\left(\bm{x}_{j}^\top\bm{\lambda} \geq T_{i}, y_{j}=0\right)\right) \\
    \ge &\sum_{i=0}^{M}\left(\omega_{i} \sum_{j=1}^{N} I\left( \bm{x}_{j}^\top\bm{\rho}\ge t_{i},y_j=1\right)
    -\frac{\omega_{i}p_i}{1-p_{i}} \sum_{j=1}^{N} I\left( \bm{x}_{j} ^\top\bm{\rho} \geq t_{i}, y_{j}=0\right)\right) -  (k-1)\sum_{i=0}^M\frac{\omega_i}{1-p_i}.
\end{aligned}
\end{equation}
\end{corollary}

\begin{proof}
    See the appendix.
\end{proof}
Theorem \ref{thm:learning_capacity} and Corollary \ref{cor:learning_capacity} characterize the learning capacity of the proposed model on the training set. Next, we will demonstrate its generalization, that is, the expected performance on all possible values of $(x,y) \in \mathcal{X}\times\mathcal{Y}$. We use $R_N(\bm \lambda,\bm T)$ and $R(\bm \lambda,\bm T)$ to represent the weighted sum of empirical and expected negative net benefit,respectively. Formally,
\begin{align}
    R_N(\bm \lambda,\bm T) = - \frac{1}{N}\sum_{i=0}^M \sum_{j=1}^N \omega_i {I}(\bm x_j^\top \bm \lambda - T_i \ge 0,y_i=1) + \frac{1}{N}\sum_{i=0}^M \sum_{j=1}^N \frac{\omega_i p_i}{1-p_i} {I}(\bm x_j^\top \bm \lambda - T_i \ge 0,y_i=0),
\end{align}
\begin{align}
    R(\bm \lambda,\bm T) = - \sum_{i=0}^M \omega_i \mathbb{E}_{\mathcal X, \mathcal Y} \left[ {I}(\bm x_j^\top \bm \lambda - T_i \ge 0,y_i=1)\right]  + \sum_{i=0}^M \frac{\omega_i p_i}{1-p_i} \mathbb{E}_{\mathcal X, \mathcal Y} \left[ {I}(\bm x_j^\top \bm \lambda - T_i \ge 0,y_i=0)\right] .
\end{align}
Theorem \ref{thm:generalization} shows an upper bound on the difference between $R(\bm \lambda,\bm T)$ and $R_N(\bm \lambda,\bm T)$. 
\begin{theorem}\label{thm:generalization}
Let $\Lambda_k \in \mathbb{Z}_+$, $k=1,2,\ldots,P$, and let $T_{\max}\in \mathbb{Z}_+$. $\mathcal{L}=\{\bm \lambda \in \mathbb{Z}^P:-\Lambda_k \le \lambda_k \le \Lambda_k,\,k=1,2,\ldots,P\}$, $\mathcal{B}_0=\{T \in \mathbb{Z}:-T_{\max} \le T \le T_{\max}\}$, $\mathcal{B}=\{\bm T \in \mathbb{Z}^{M+1}:T_i \in \mathcal{B}_0,\,i=0,1,\ldots,M\}$. Then, for any $\bm\lambda \in \mathcal{L}$, $\bm T \in {\mathcal{B}}$, given a small $\delta > 0$, with probability at least $1 - 2(M + 1)\delta$, we have
\begin{align}
\sup_{\bm \lambda \in \mathcal{L},\bm T \in \mathcal{B}} \left\{ R(\bm \lambda,\bm T) - R_{N}(\bm \lambda,\bm T) \right\} \le \sum_{i=0}^{M} \frac{ \omega_i}{1 - p_i} \sqrt{\frac{\ln|\mathcal{L}|+\ln|{\mathcal{B}}_0| - \ln \delta}{2N}} ,
\end{align}
where $|\mathcal{L}| = \prod_{k=1}^P (2\,\Lambda_k + 1)$, and $|{\mathcal{B}}_0|=2\,T_{\max}+1$.
\end{theorem}
\begin{proof}
    See the appendix.
\end{proof}

\Cref{thm:generalization} provides a generalization bound for the proposed model over the finite parameter sets $\mathcal{L}$ and $\mathcal{B}$. It shows that the expected negative net benefit $R(\bm \lambda,\bm T)$ can be controlled by the empirical version $R_{N}(\bm \lambda,\bm T)$, plus a term that depends logarithmically on the size of $\mathcal{L}$ and $\mathcal{B}_0$, and decreases at the rate $\mathcal{O}(1/\sqrt{N})$. This shows that the finiteness of $\mathcal{L}$ and $\mathcal{B}$ not only enhances the interpretability of the model but also provides explicit control over generalization. 
\section{Experiments} \label{sec:experiments}
In this section, we present numerical experiments based on publicly available datasets to compare the predictive performance, decision utility, and model sparsity of RSS-DNB with other baseline models. The purpose of this section is to demonstrate that RSS-DNB, while explicitly optimizing d-AUNBC, can achieve comparable discrimination, calibration, and sparsity.
\subsection{Experimental setup}
We ran experiments on eight datasets from the UCI Machine Learning Repository \citep{kellyUCIMachineLearning}. Following the preprocessing procedure of \citet{ustunSupersparseLinearInteger2016}, we binarized all category variables and some continuous variables, removed all samples with missing values, and partitioned each dataset into ten folds for cross-validation. Table \ref{tab:dataset} 

\begin{table}[htbp]
    \caption{The processed datasets used in the study.}
    \begin{tabularx}{\textwidth}{lrrrX}
    \toprule
    Dataset & $N$ & $P$ & $N^+/N$ & Task\\
    \midrule
    adult \citep{beckerAdult1996} & 32,561 & 36 & 24.1\% & Predict whether annual income of an individual exceeds \$50,000\\
    bankruptcy \citep{kimDiscoveryExpertsDecision2003} & 250 & 6 & 57.2\% & Bankruptcy prediction based on qualitative parameters provided by experts\\
    breastcancer \citep{wolbergMultisurfaceMethodPattern1990} & 683 & 9 & 35.0\% & Predict whether a breast tumor is malignant based on cytological characteristics\\
    haberman \citep{habermanGeneralizedResidualsLogLinear1976} & 306 & 3 & 73.5\% & Predicting the survival of breast cancer surgery patients \\
    heart \citep{detranoInternationalApplicationNew1989} & 303 & 32 & 45.9\% & Predicting whether a patient has a high risk of coronary artery disease \\
    mammo \citep{elterPredictionBreastCancer2007} & 961 & 14 & 46.3\% & Predict whether a mammographic mass is malignant \\ 
    mushroom \citep{schlimmerMushroom1987} & 8,124 & 113 & 48.2\% &  Determine whether a mushroom is poisonous\\
    spambase \citep{hopkinsSpambase1999} & 4,601 & 57 & 39.4\% & Determine whether an email is spam\\
    \bottomrule
    \end{tabularx}    
    \label{tab:dataset}
\end{table}

For each dataset, the models were trained on nine folds and evaluated on the remaining fold. Performance metrics were averaged over the 10 folds. We report the mean AUROC, AUNBC, and ECE on both the training and test sets, together with their standard deviations, as well as the average model size and its range. The AUROC, AUNBC, and ECE are used to evaluate models' discrimination, utility, and calibration, respectively, while the size serves as a measure of models' sparsity. Throughout \Cref{sec:experiments,sec:application} and the related empirical appendix, AUROC refers to the conventional AUROC based on the complete ranking of prediction scores, whereas AUNBC refers to the discretized AUNBC defined in \Cref{def:aunbc}. Unless otherwise specified, AUNBC is evaluated on the threshold grid $p_i = \frac{i}{10}$, $i=0,1,\ldots,9$ and interval-length weights $\omega_i=\frac{1}{10}$. ECE is computed using the same prespecified threshold grid. The model size is defined as the number of nonzero coefficients for linear models (excluding the intercept) and as the number of nodes for decision tree models. 

We considered a range of models in our experiments, including sparse linear scoring models (RSS-DNB, RSS-DNB with simulated annealing (RSS-DNB-SA), SLIM, and RISKSLIM), baseline interpretable models (logistic regression, LASSO-regularized logistic regression, and decision tree), and the gradient boosting model XGBoost. All methods were trained and evaluated using the same data preprocessing procedure, the same cross-validation folds, and the same evaluation protocol.

For the RSS-DNB model, the $\ell_0$-penalty parameter was set as $C_0 = 0.1\cdot N^+/NP$, and all coefficients were restricted to $\{-10, \ldots, 10\}$. The RSS-DNB-SA model was trained using Algorithm~\ref{alg:sa}. The $\ell_0$-penalty parameter and the coefficient range were the same as those for RSS-DNB. The initial temperature was set as $10^{-3}$ with a cooling rate of $10^{-6}$, and the minimum temperature was set as 0. At each temperature, the number of iterations was set as 10. \Cref{alg:moderate_calibration} was applied after training RSS-DNB and RSS-DNB-SA. We also conducted ablation studies to evaluate the effects of the sparsity penalty $C_0$ and \Cref{alg:moderate_calibration}, and compare the performance of other baseline models before and after applying \Cref{alg:moderate_calibration}. The detailed experimental results are provided in \ref{app:ablation}. For the SLIM model, we adopted the settings recommended in the original paper \citep{ustunSupersparseLinearInteger2016}. Specifically, the $\ell_0$-penalty parameter was set as $C_0 = 0.9 / N P$ and the $\ell_1$-penalty parameter was set as $\epsilon = C_0 / 10$, where $N$ and $P$ denote the sample size and the number of features, respectively. The coefficients were restricted to $\{-10, \ldots, 10\}$, and the intercept was restricted to $\{-100, \ldots, 100\}$. The RISKSLIM model was solved using the initialization procedure and the LCAP algorithm proposed in the original work \citep{ustunLearningOptimizedRisk2019}. Following the paper, the regularization parameter was set as $C_0 = 10^{-6}$, and the intercept term was constrained to $\{-100, \ldots, 100\}$. To align the coefficient scale with the other models, we allowed the coefficients to take values in $\{-10, \ldots, 10\}$, although the original implementation restricted them to $\{-5, \ldots, 5\}$. For the XGBoost model, we used a maximum tree depth of 4, a learning rate of 0.1, and 100 boosting iterations. To obtain calibrated probability estimates, Platt scaling was applied to the predicted scores after training.

For both RSS-DNB and RSS-DNB-SA, the optimization objective was constructed using uniformly spaced thresholds $p_i=\frac{i}{10}$, $i=0,1,\ldots,9$, with equal weights $\omega_i=\frac{1}{10}$. Unless otherwise specified, this setting was adopted throughout the experiments. To evaluate the sensitivity of the proposed methods to the threshold gridding and weighting scheme, we additionally considered three threshold setting during the training: (i) a finer uniform grid ($p_i = \frac{i}{20}$); (ii) a coarser uniform grid ($p_i = \frac{i}{5}$); and (iii) a randomly generated threshold setting. For the randomly generated threshold grid, we further considered two weighting schemes: equal weights and interval-length weights ($\omega_i = p_{i+1} - p_i$). The corresponding results are reported in \ref{app: sensitivity_analysis}.

The optimization problems underlying the RSS-DNB, SLIM, and RISKSLIM models were solved using Gurobi Optimizer 11.0.3 via the MATLAB (R2025a) interface. All problems were free from additional constraints, such as the number of non-zero coefficients, and the solution time for each optimization problem was limited to 10 minutes. When Gurobi optimizer terminated at the time limit without proving optimality, we used its best feasible incumbent. Logistic regression, LASSO-regularized logistic regression, and decision tree were trained using the corresponding built-in MATLAB functions. The XGBoost model and Platt scaling were implemented in Python 3.12.13 using the 'xgboost' and 'scikit-learn' libraries. Statistical significance analysis were performed in R (version 4.5.2) using the 'PMCMRplus' package, where the Friedman test followed by Nemenyi test was conducted.
\subsection{Results}
We summarize the results in \Cref{tab:performance} and show the ROC curves, calibration plots, and decision curves of all models on each dataset in \Cref{fig:roc_part1,fig:roc_part2,fig:calibration_part1,fig:calibration_part2,fig:dca_part1,fig:dca_part2}, respectively. All reported curves are generated using out-of-fold predictions from 10-fold cross-validation. In addition, since RSS-DNB, SLIM, and RISKSLIM are mix-integer linear programming (MILP)-based models, we report their computational statistics in \ref{app:additional} (\Cref{tab:computational_statistics}), including runtime, optimality gap, solver termination status, and number of solved instances. \Cref{tab:computational_statistics} shows that the MILP-based methods reached the time limit without proving optimality on several datasets. Therefore, the corresponding performance results reflect feasible incumbents rather than certified optima.

As shown in \Cref{tab:performance} and \Cref{fig:roc_part1,fig:roc_part2}, the proposed RSS-DNB models achieved competitive AUROC across the eight test sets. Importantly, optimizing net benefit did not result in a substantial loss of discrimination, indicating that the proposed approach maintains strong ranking performance while targeting decision-oriented objectives. 

The two RSS-DNB models consistently achieved perfect calibration on the training sets (ECE = 0), which is in accordance with \Cref{cor:calibration}. This result confirms that the proposed model construction procedure explicitly enforces calibration under the specified conditions. On the test sets, the RSS-DNB models generally demonstrated improved calibration compared with baseline methods, as reflected in both lower ECE and visually better alignment in calibration plots (\Cref{fig:calibration_part1,fig:calibration_part2}). These findings suggest that the calibration advantages are not limited to the training data but also translate into improved out-of-sample reliability.

In terms of utility performance, the RSS-DNB models achieved net benefit levels that were comparable to baseline models across a broad range of thresholds (\Cref{fig:dca_part1,fig:dca_part2}). While no single model dominated across all datasets and threshold ranges, the proposed approach consistently provided competitive or superior net benefit within the target decision region.

To assess whether the differences among competing methods are statistically significant, we performed the Friedman test followed by Nemenyi test across all datasets. The analysis was conducted separately for AUROC, AUNBC, ECE, and model size. For model size, only linear models (RSS-DNB, RSS-DNB-SA, logistic, LASSO, SLIM, and RISKSLIM) were included because their complexity is measured consistently by the number of non-zero coefficients. For each metric, the Friedman test rejected the null hypothesis that all methods perform equivalently, with significant differences observed for AUROC ($\chi^2=29.00$, $df=7$, $p<0.001$), AUNBC ($\chi^2=18.36$, $df=7$, $p=0.010$), ECE ($\chi^2=25.93$, $df=7$, $p<0.001$), and model size ($\chi^2=17.46$, $df=5$, $p=0.004$). However, the subsequent Nemenyi test showed that statistical differences depended on the metrics. Specifically, the proposed methods did not show statistically significant differences compare to baseline models in term of AUROC and AUNBC. However, RSS-DNB achieved significantly better calibration than LASSO and XGBoost in terms of ECE ($p=0.020$ for both comparisons), and both RSS-DNB and RSS-DNB-SA produced significantly smaller model sizes than logistic regression ($p=0.039$ and $p=0.001$, respectively). These results show that the proposed methods maintain competitive predictive and decision utility performance while providing advantages in calibration and sparsity for specific comparisons. Detailed results of the Friedman tests and Nemenyi tests are provided in \ref{app:statistical_analysis}.

Notably, despite being trained with a decision-oriented objective, the RSS-DNB models did not sacrifice discrimination or calibration while enforcing sparsity constraint to achieve utility optimization. Although the proposed method does not uniformly outperform all alternatives on every metric, the framework provides a principled mechanism to directly align model training with downstream decision-making, ensuring that utility considerations are formally incorporated rather than treated as a post hoc evaluation criterion.

\begin{figure}[htbp]
\centering

\begin{subfigure}{0.45\textwidth}
  \centering
  \includegraphics[width=\linewidth]{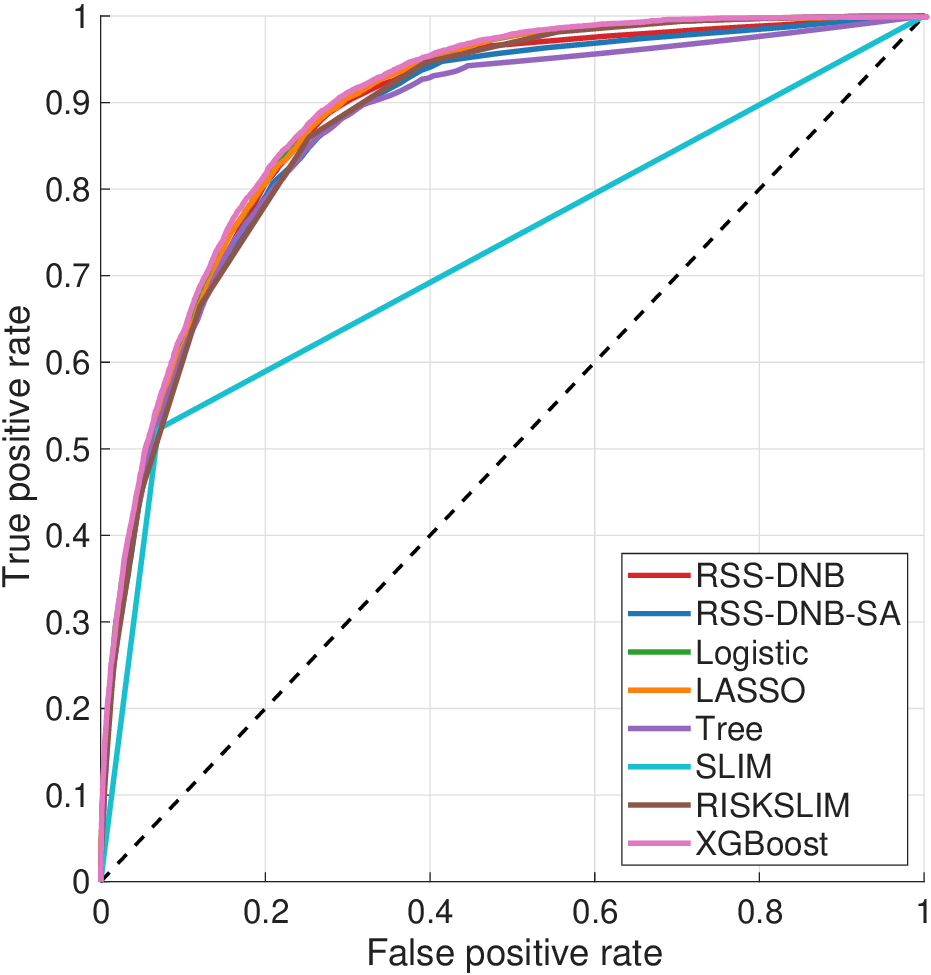}
  \caption{adult}
\end{subfigure}
\hfill
\begin{subfigure}{0.45\textwidth}
  \centering
  \includegraphics[width=\linewidth]{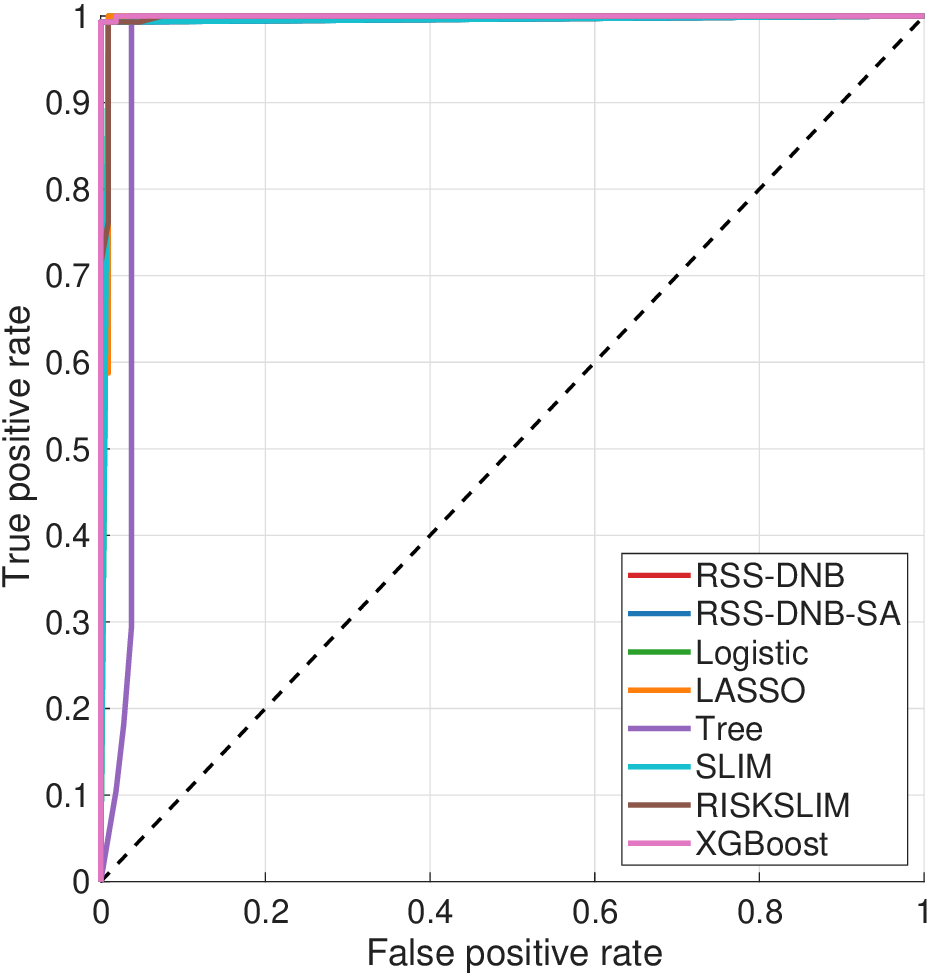}
  \caption{bankruptcy}
\end{subfigure}

\vspace{0.5cm}

\begin{subfigure}{0.45\textwidth}
  \centering
  \includegraphics[width=\linewidth]{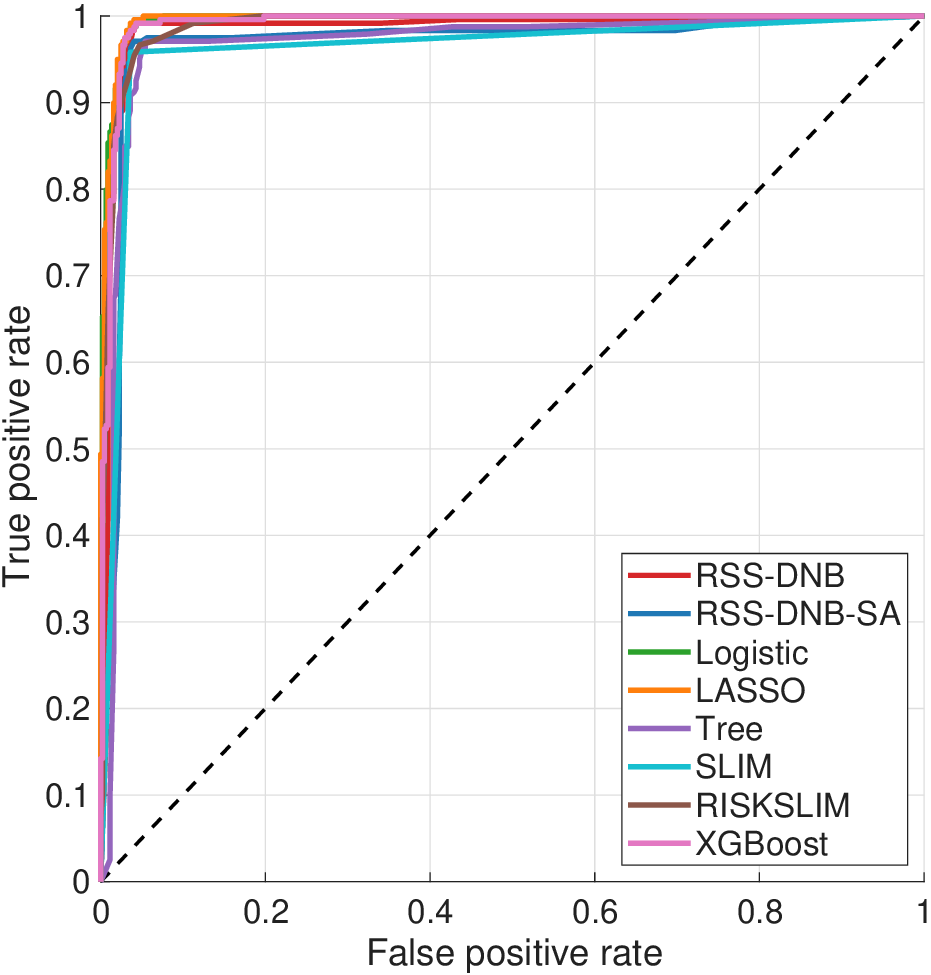}
  \caption{breastcancer}
\end{subfigure}
\hfill
\begin{subfigure}{0.45\textwidth}
  \centering
  \includegraphics[width=\linewidth]{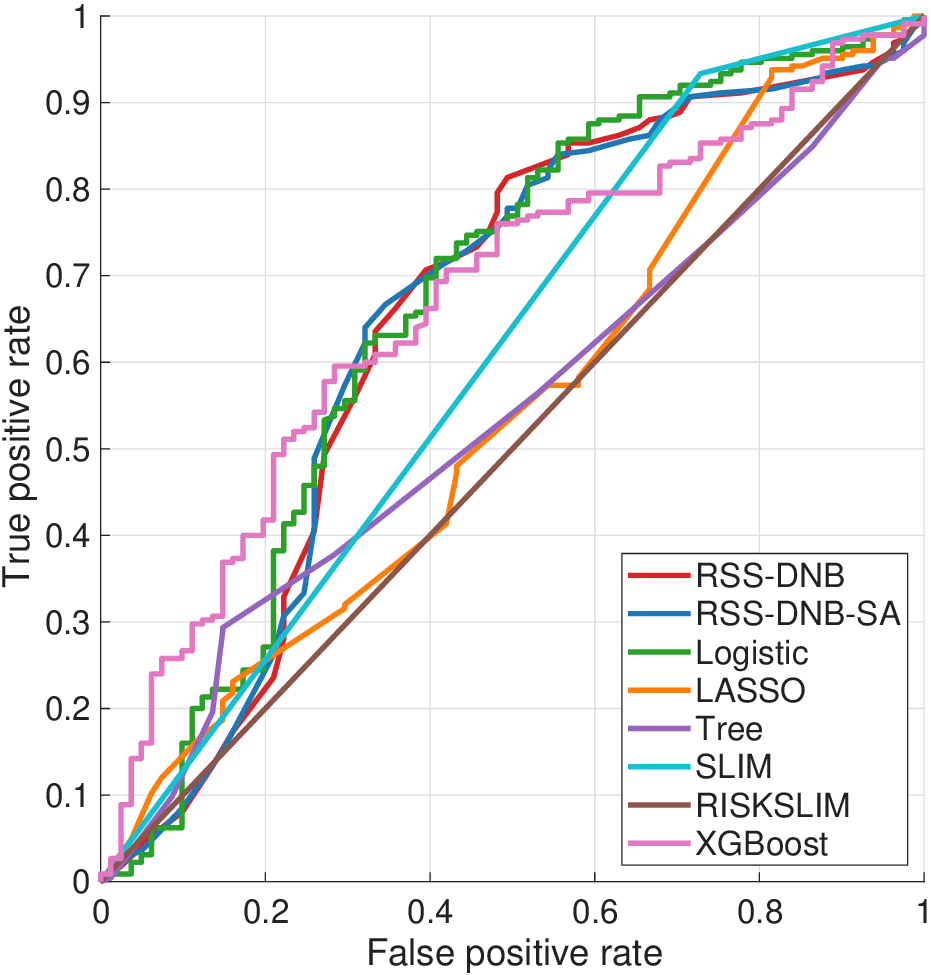}
  \caption{haberman}
\end{subfigure}

\caption{ROC curves on the first four datasets. Each panel shows the ROC curves of all models evaluated on the corresponding dataset.}
\label{fig:roc_part1}

\end{figure}

\begin{figure}[htbp]
\centering

\begin{subfigure}{0.45\textwidth}
  \centering
  \includegraphics[width=\linewidth]{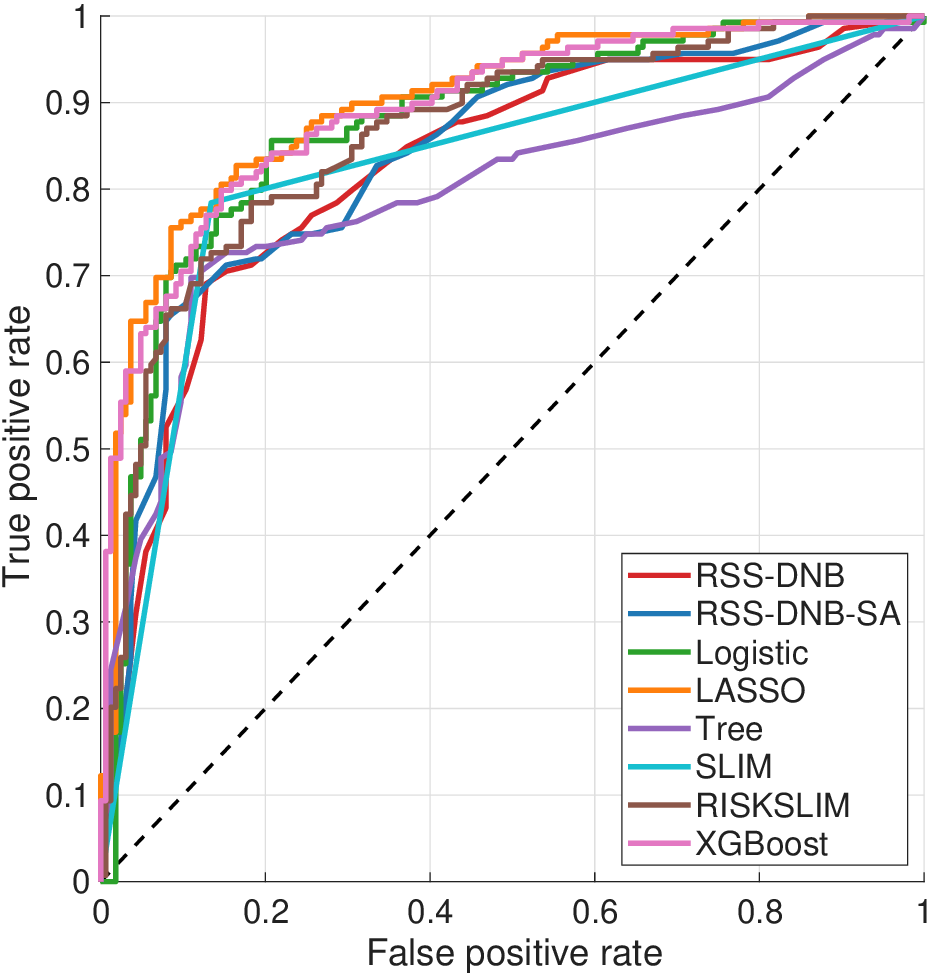}
  \caption{heart}
\end{subfigure}
\hfill
\begin{subfigure}{0.45\textwidth}
  \centering
  \includegraphics[width=\linewidth]{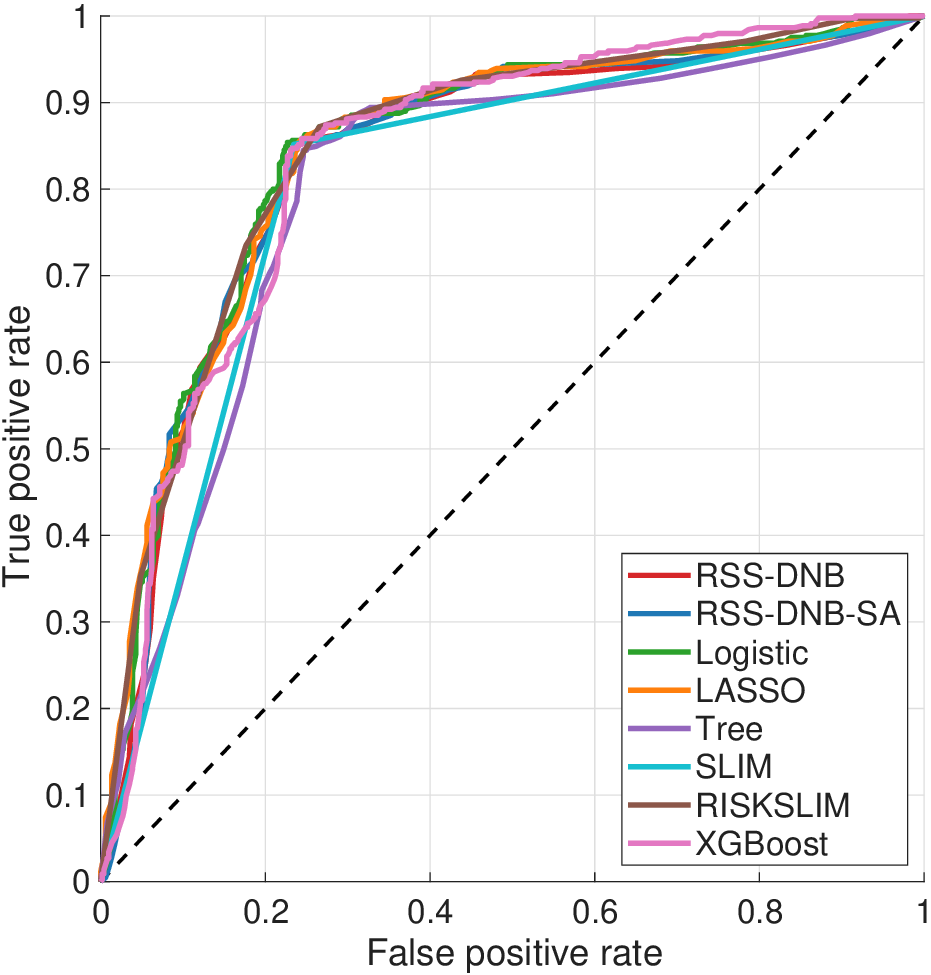}
  \caption{mammo}
\end{subfigure}

\vspace{0.5cm}

\begin{subfigure}{0.45\textwidth}
  \centering
  \includegraphics[width=\linewidth]{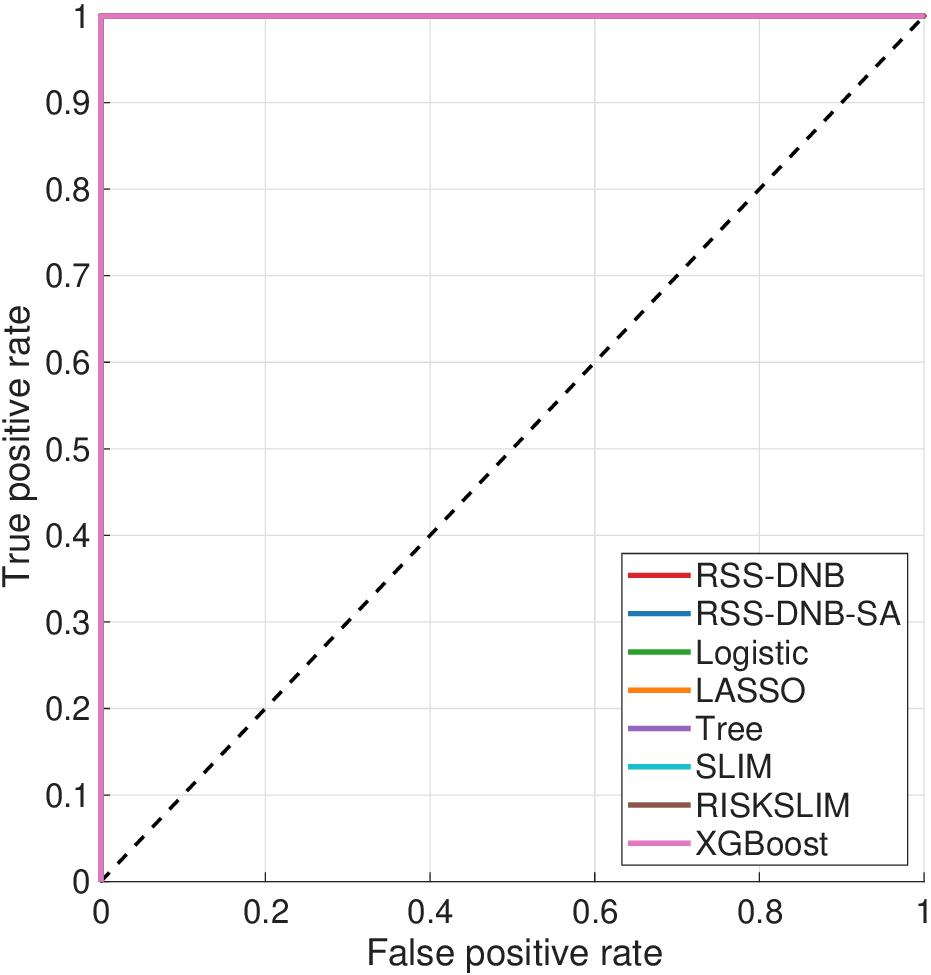}
  \caption{mushroom}
\end{subfigure}
\hfill
\begin{subfigure}{0.45\textwidth}
  \centering
  \includegraphics[width=\linewidth]{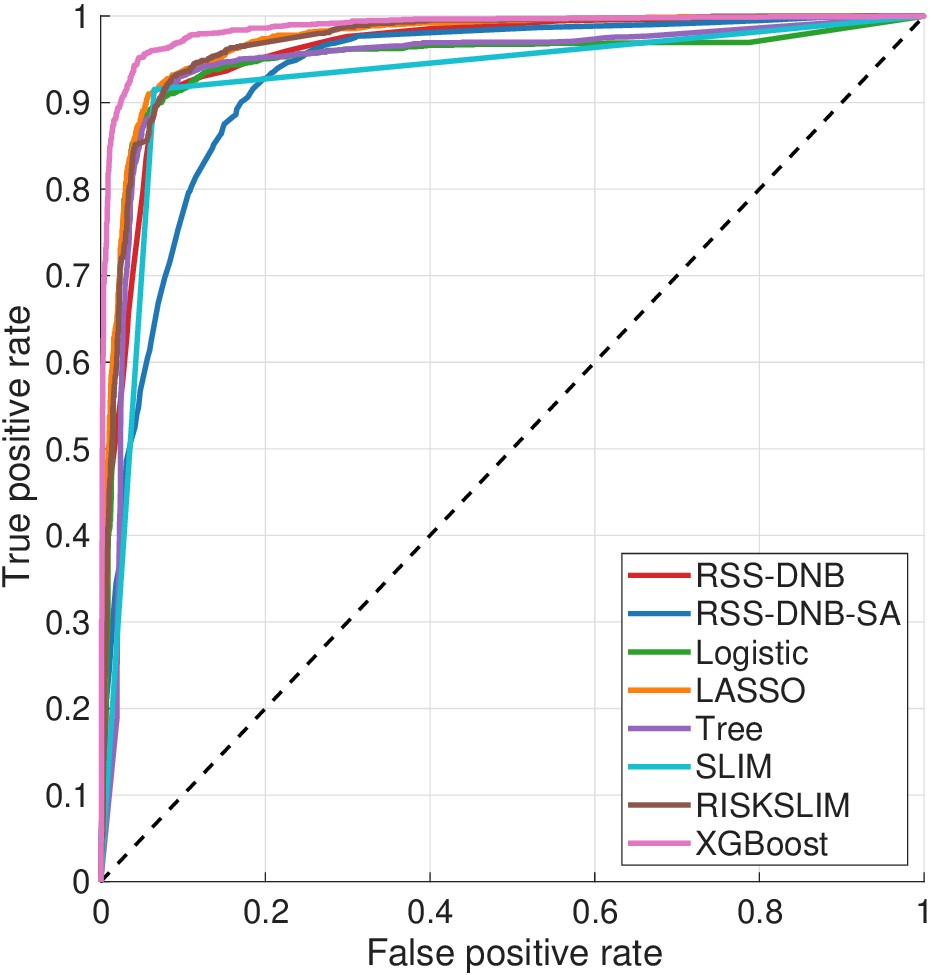}
  \caption{spambase}
\end{subfigure}

\caption{ROC curves on the remaining four datasets.}
\label{fig:roc_part2}

\end{figure}

\begin{figure}[htbp]
\centering

\begin{subfigure}{0.45\textwidth}
  \centering
  \includegraphics[width=\linewidth]{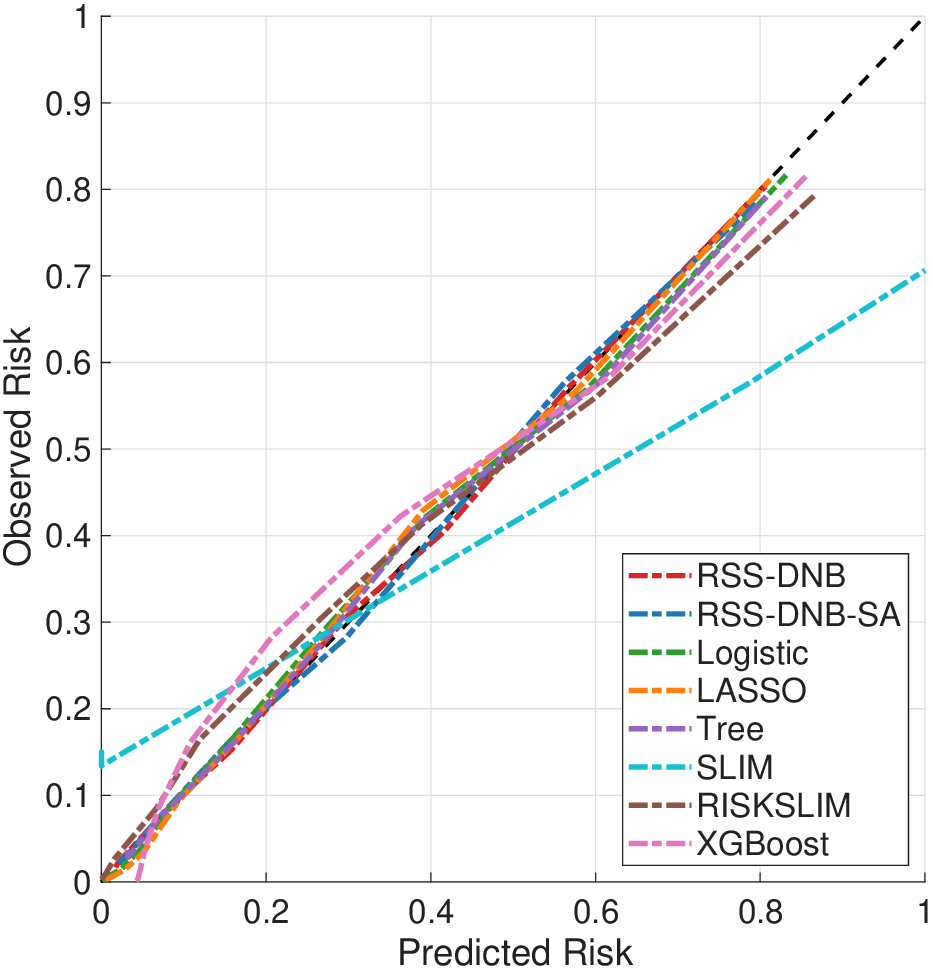}
  \caption{adult}
\end{subfigure}
\hfill
\begin{subfigure}{0.45\textwidth}
  \centering
  \includegraphics[width=\linewidth]{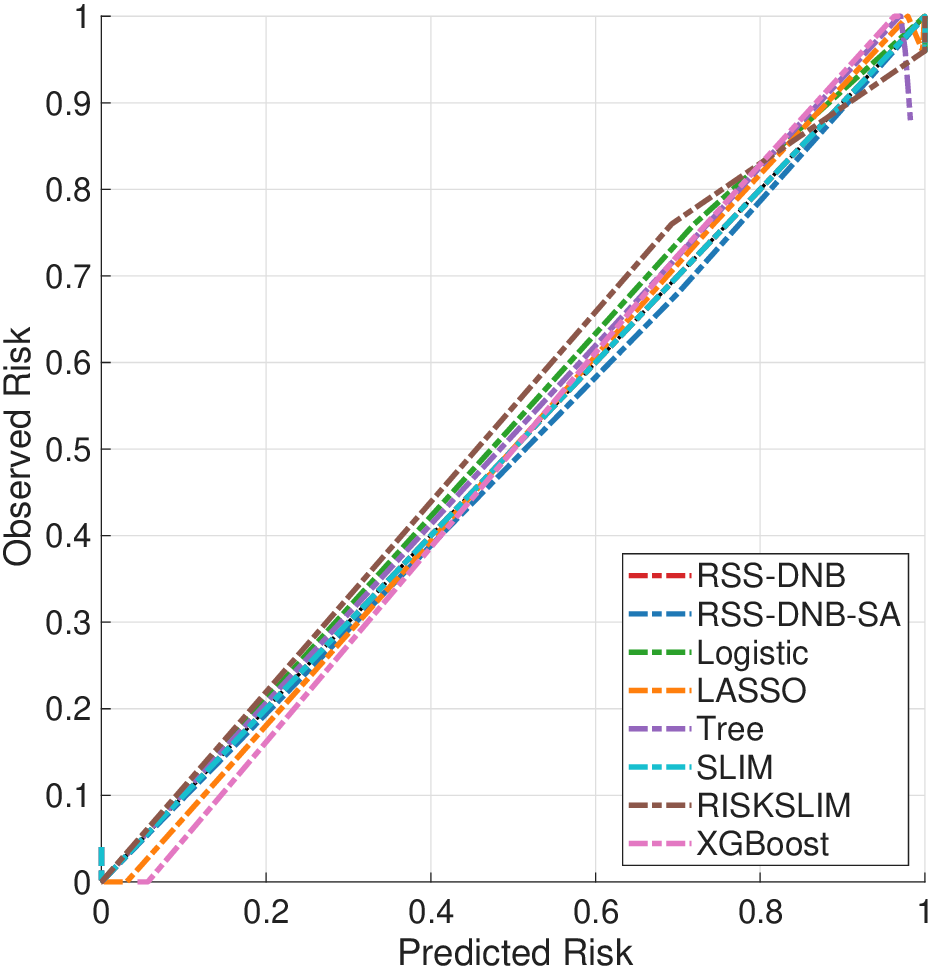}
  \caption{bankruptcy}
\end{subfigure}

\vspace{0.5cm}

\begin{subfigure}{0.45\textwidth}
  \centering
  \includegraphics[width=\linewidth]{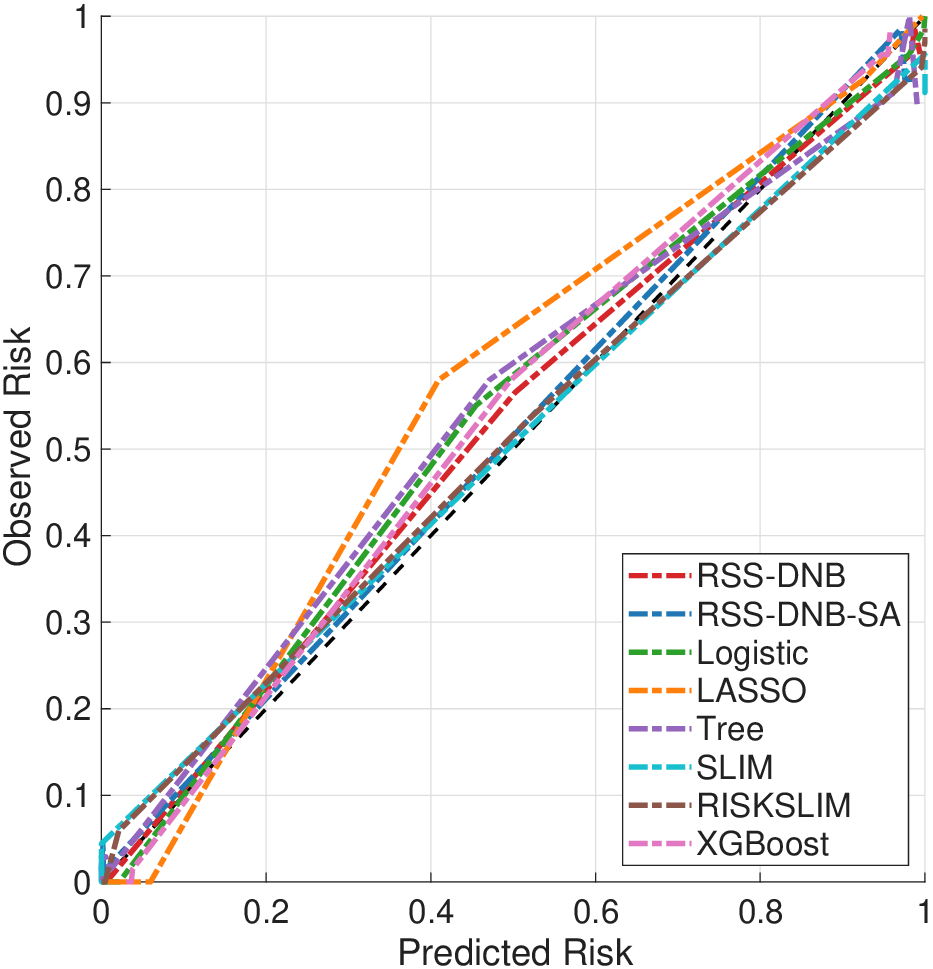}
  \caption{breastcancer}
\end{subfigure}
\hfill
\begin{subfigure}{0.45\textwidth}
  \centering
  \includegraphics[width=\linewidth]{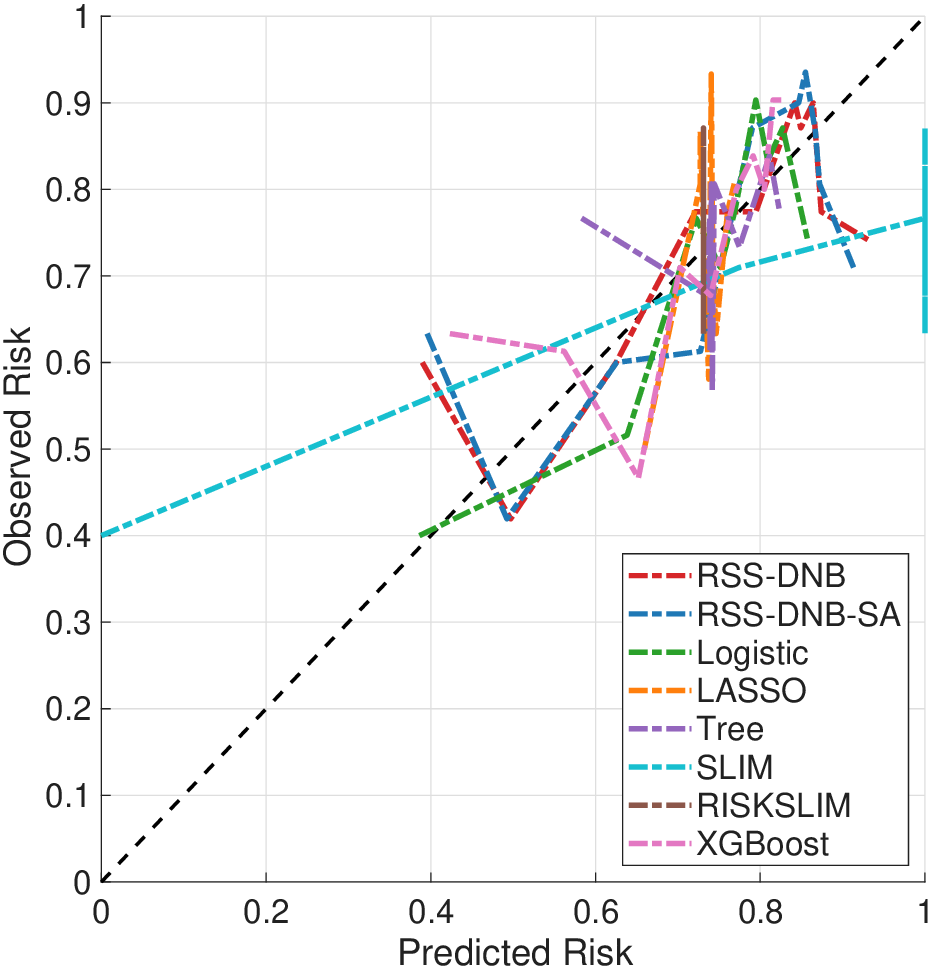}
  \caption{haberman}
\end{subfigure}

\caption{Calibration curves on the first four datasets. Each panel shows the Calibration curves of all models evaluated on the corresponding dataset.}
\label{fig:calibration_part1}

\end{figure}

\begin{figure}[htbp]
\centering

\begin{subfigure}{0.45\textwidth}
  \centering
  \includegraphics[width=\linewidth]{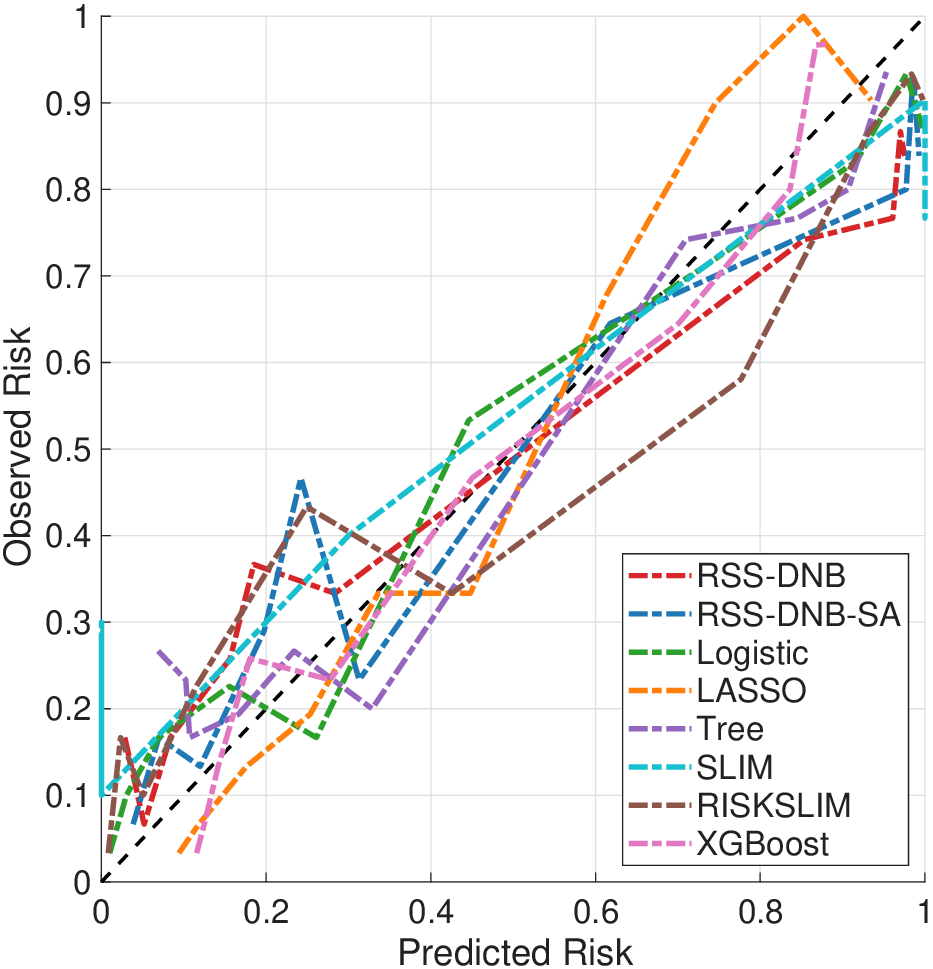}
  \caption{heart}
\end{subfigure}
\hfill
\begin{subfigure}{0.45\textwidth}
  \centering
  \includegraphics[width=\linewidth]{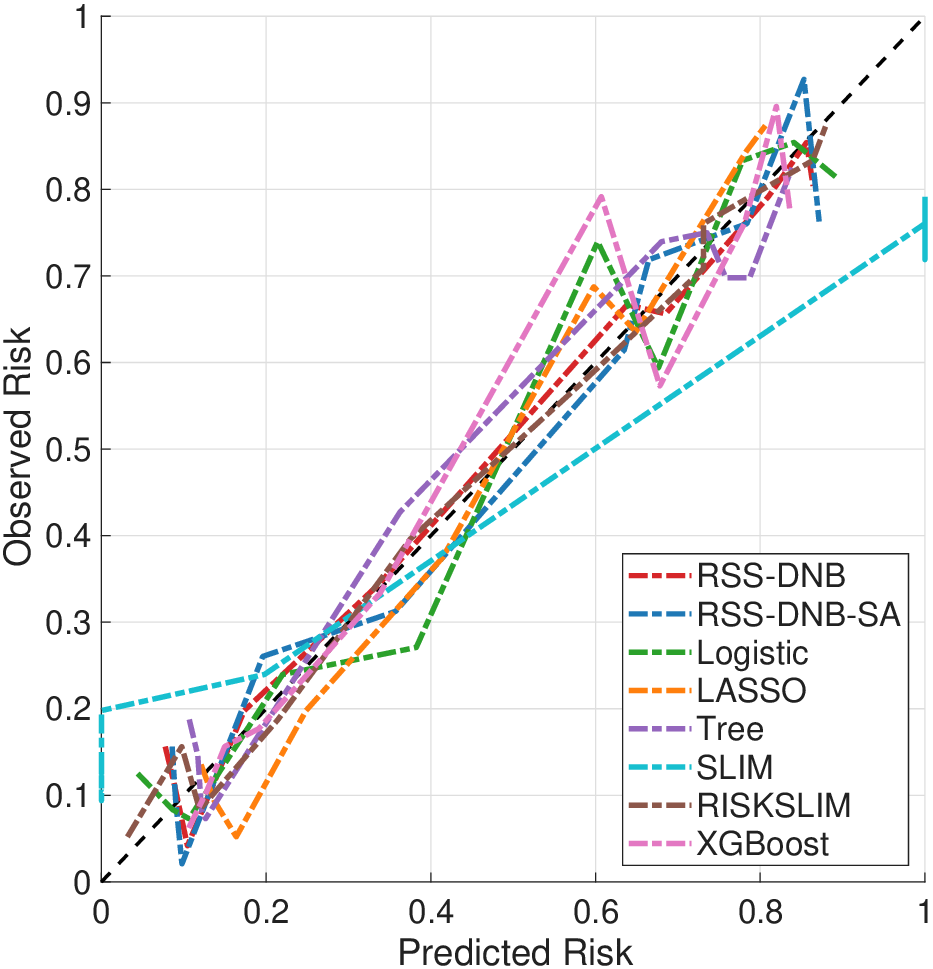}
  \caption{mammo}
\end{subfigure}

\vspace{0.5cm}

\begin{subfigure}{0.45\textwidth}
  \centering
  \includegraphics[width=\linewidth]{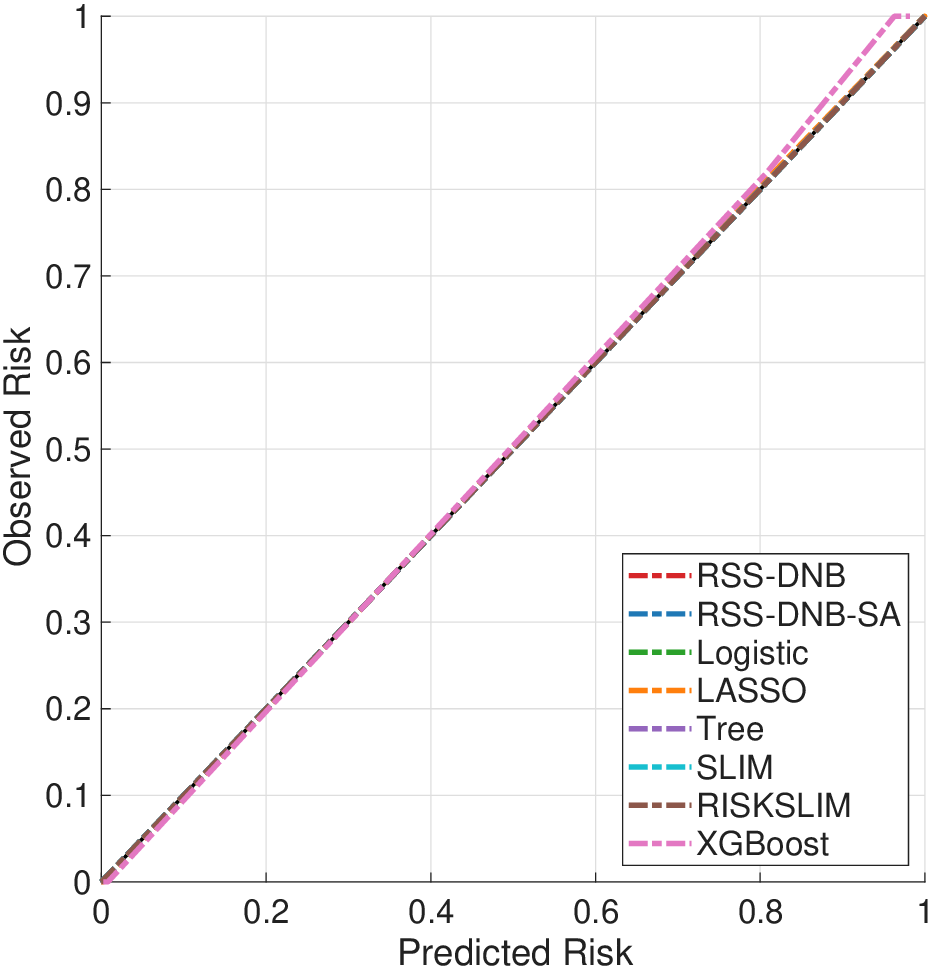}
  \caption{mushroom}
\end{subfigure}
\hfill
\begin{subfigure}{0.45\textwidth}
  \centering
  \includegraphics[width=\linewidth]{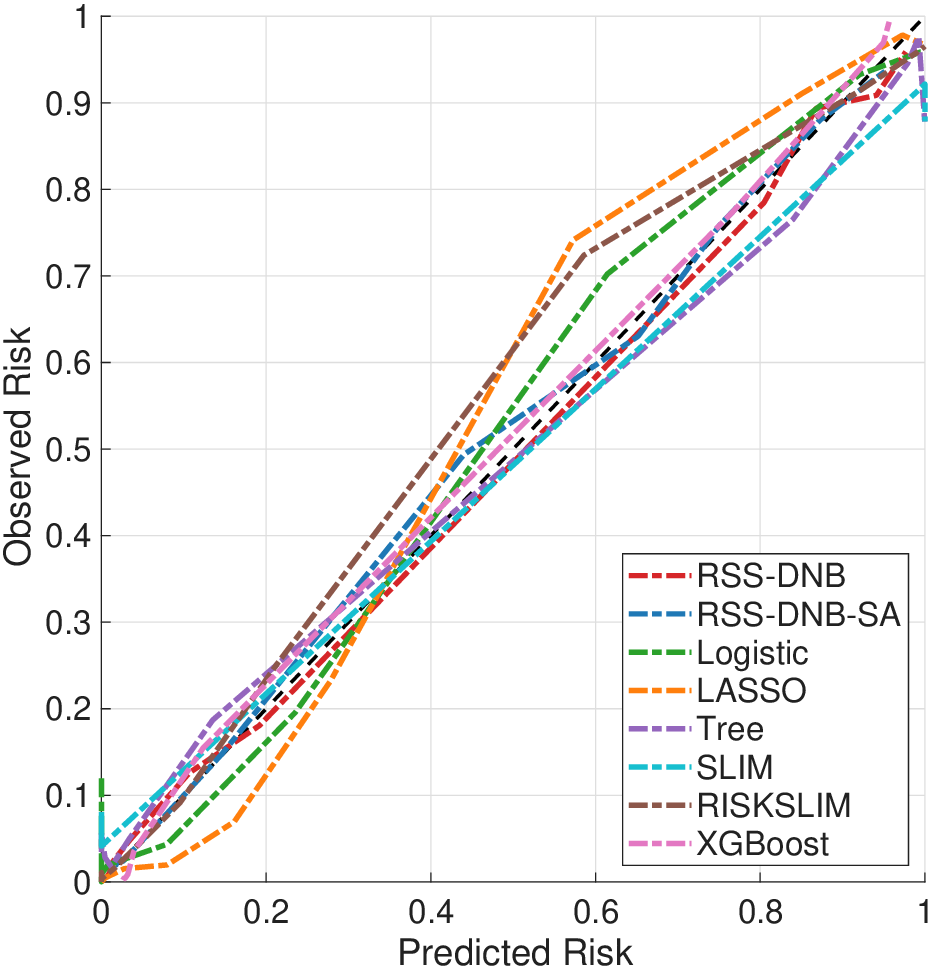}
  \caption{spambase}
\end{subfigure}

\caption{Calibration curves on the remaining four datasets.}
\label{fig:calibration_part2}

\end{figure}

\begin{figure}[htbp]
\centering

\begin{subfigure}{0.45\textwidth}
  \centering
  \includegraphics[width=\linewidth]{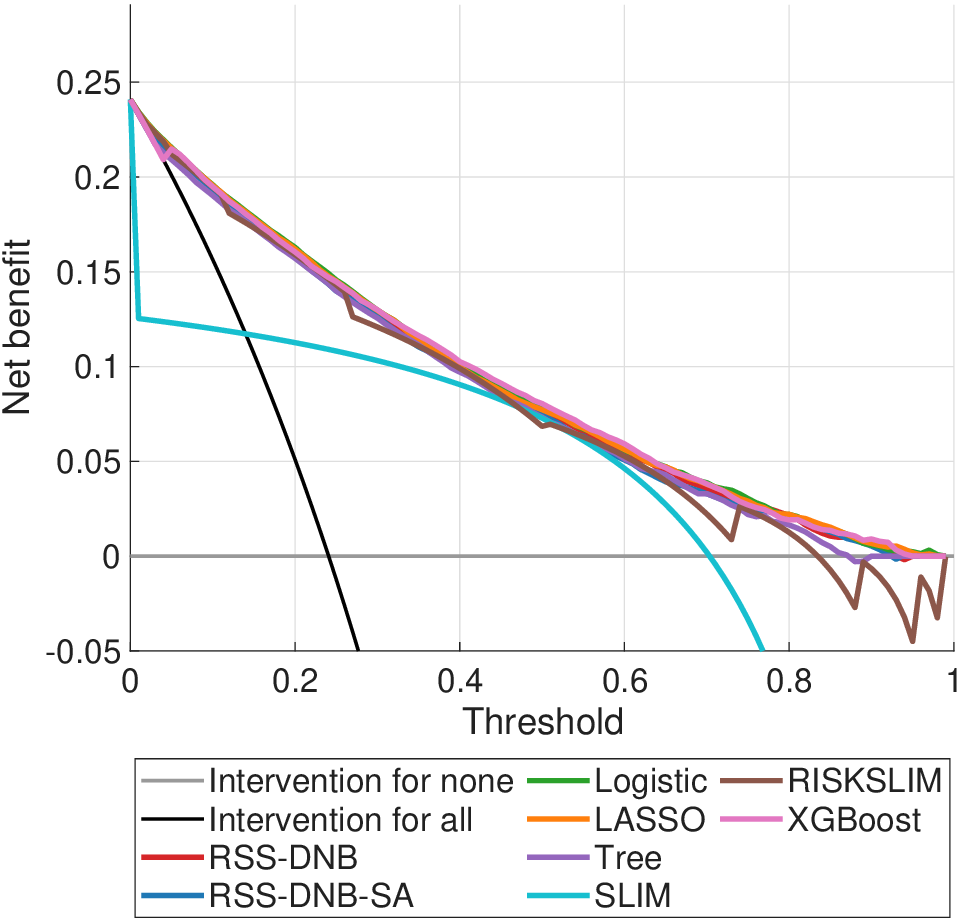}
  \caption{adult}
\end{subfigure}
\hfill
\begin{subfigure}{0.45\textwidth}
  \centering
  \includegraphics[width=\linewidth]{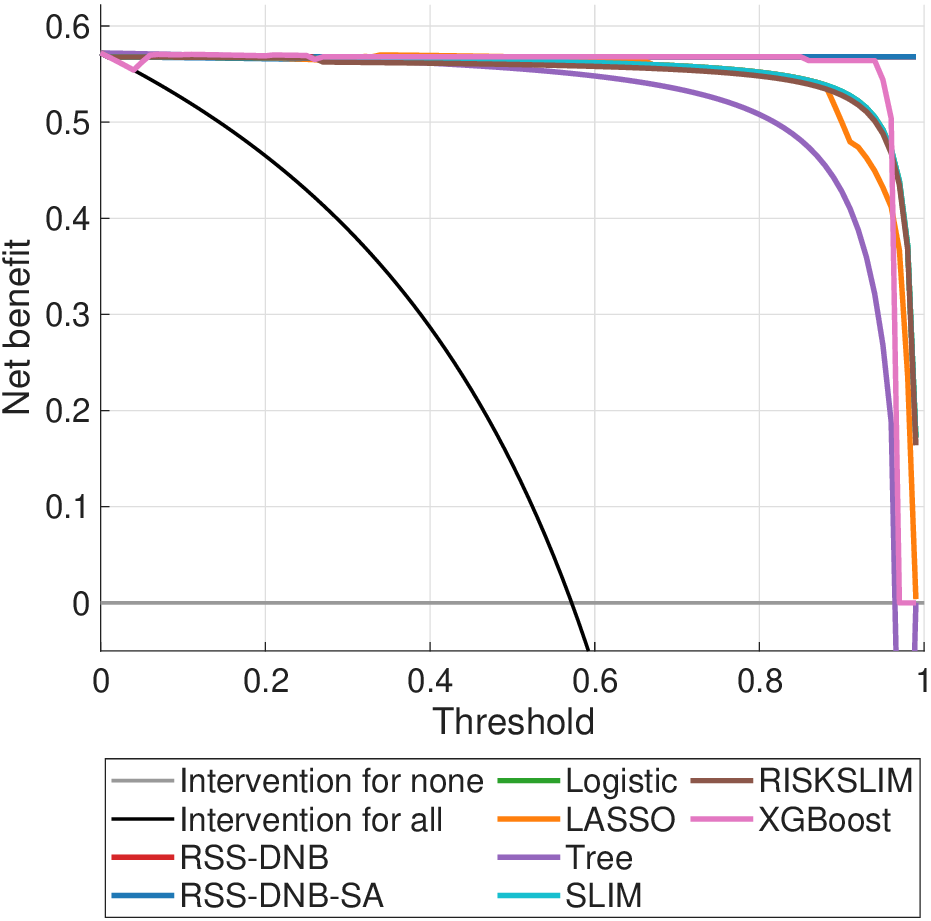}
  \caption{bankruptcy}
\end{subfigure}

\vspace{0.5cm}

\begin{subfigure}{0.45\textwidth}
  \centering
  \includegraphics[width=\linewidth]{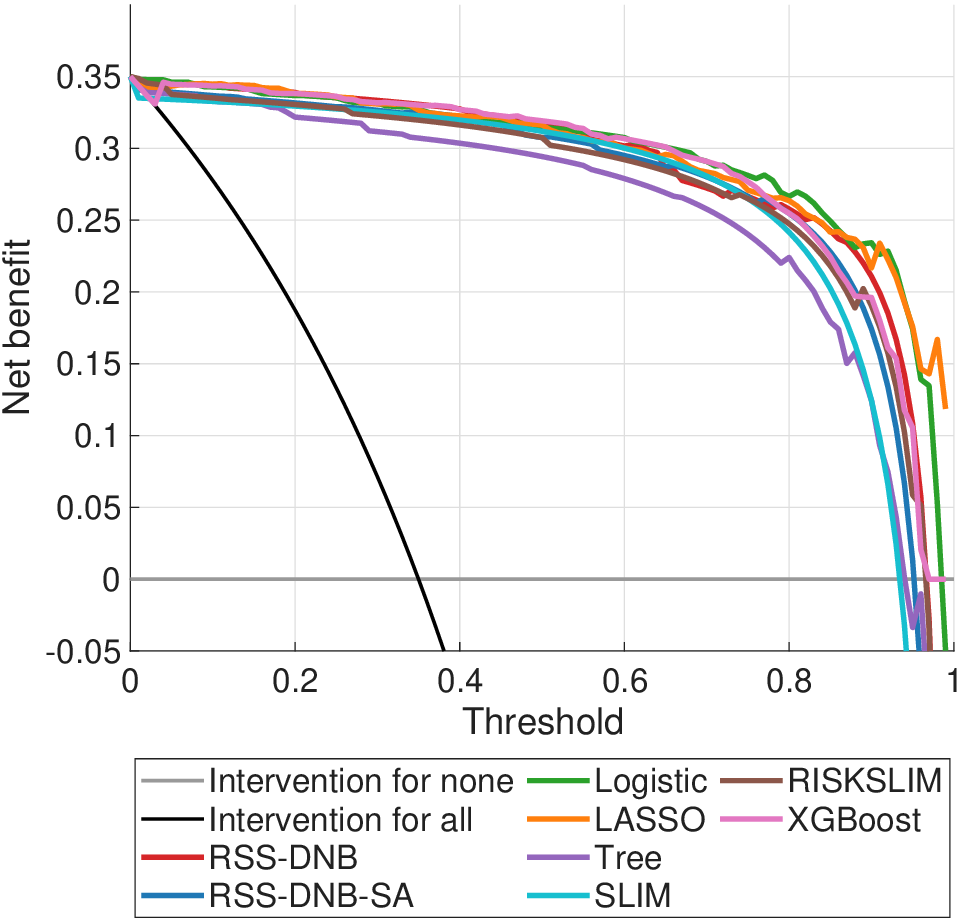}
  \caption{breastcancer}
\end{subfigure}
\hfill
\begin{subfigure}{0.45\textwidth}
  \centering
  \includegraphics[width=\linewidth]{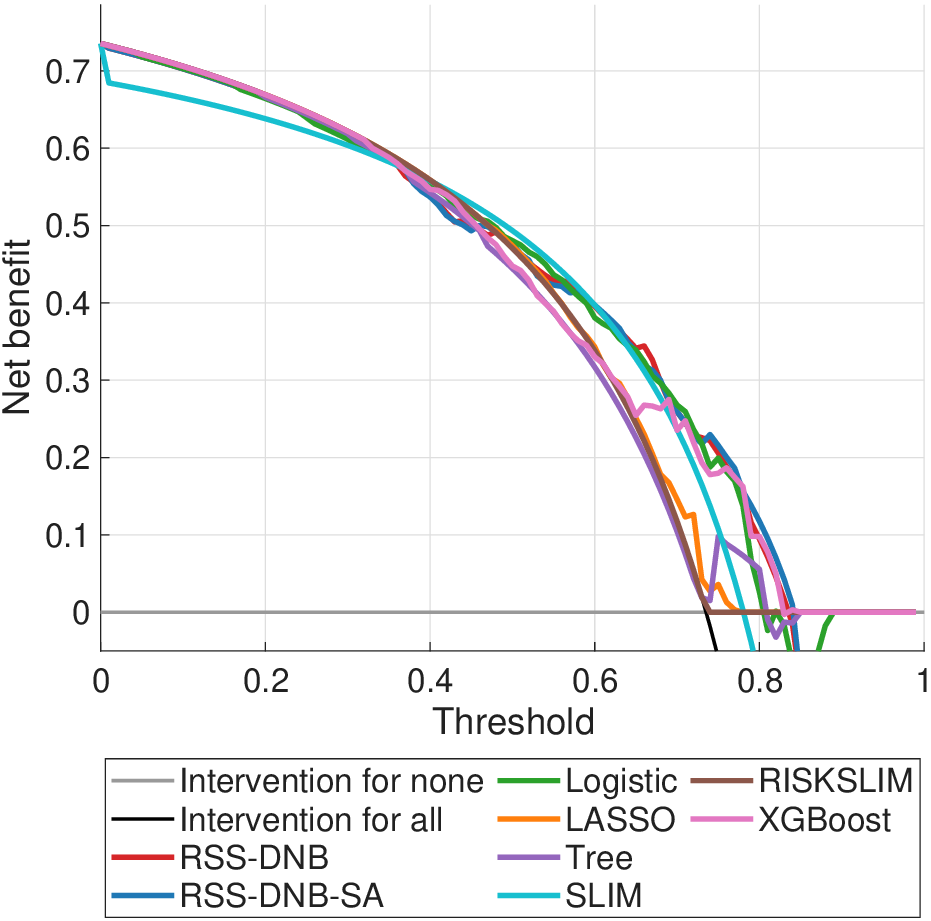}
  \caption{haberman}
\end{subfigure}

\caption{Decision curves on the first four datasets. Each panel shows the decision curves of all models evaluated on the corresponding dataset.}
\label{fig:dca_part1}

\end{figure}

\begin{figure}[htbp]
\centering

\begin{subfigure}{0.45\textwidth}
  \centering
  \includegraphics[width=\linewidth]{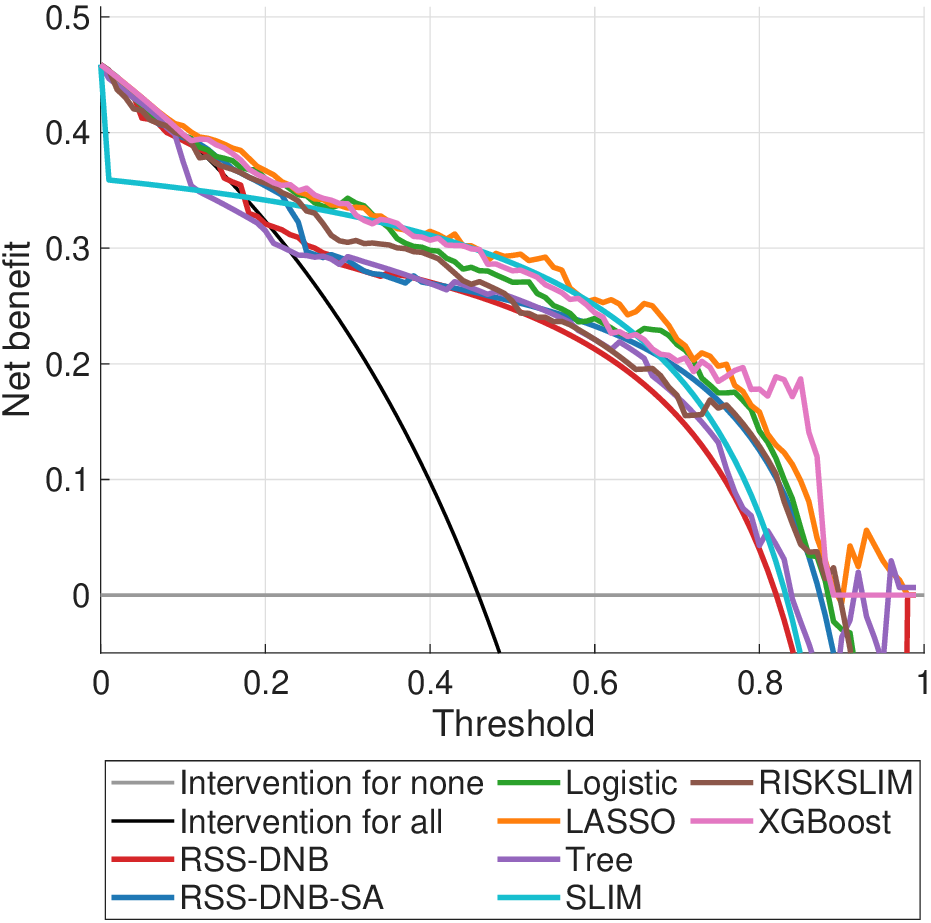}
  \caption{heart}
\end{subfigure}
\hfill
\begin{subfigure}{0.45\textwidth}
  \centering
  \includegraphics[width=\linewidth]{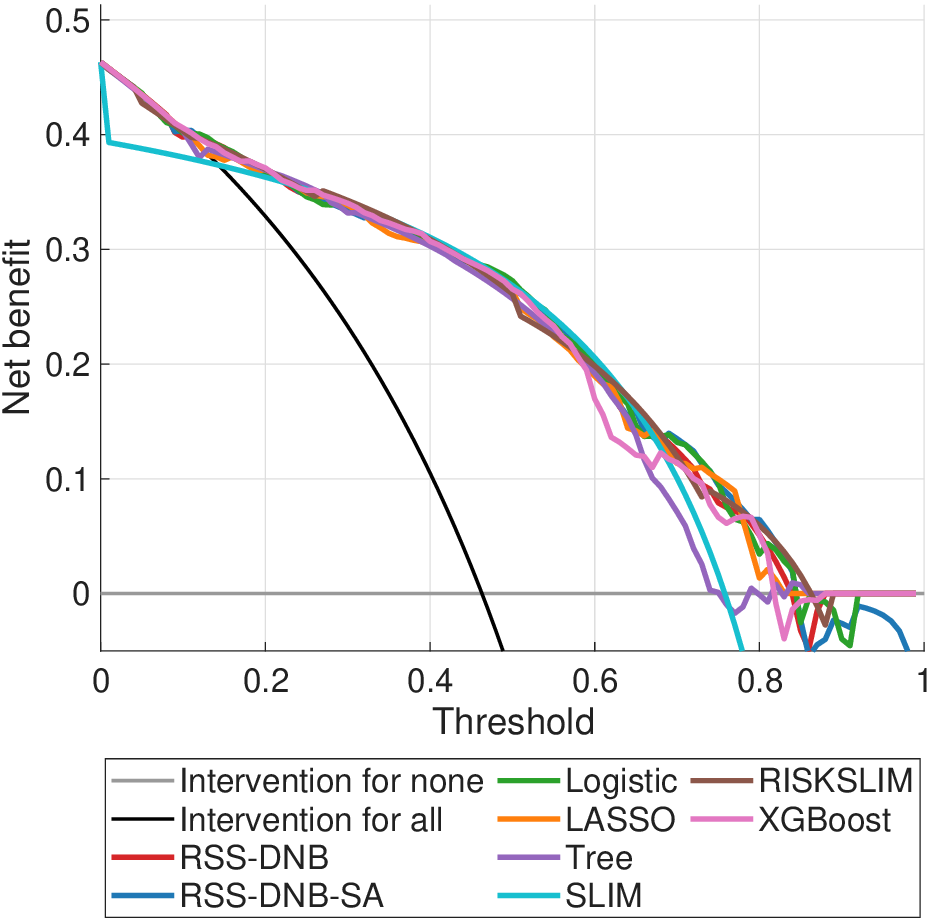}
  \caption{mammo}
\end{subfigure}

\vspace{0.5cm}

\begin{subfigure}{0.45\textwidth}
  \centering
  \includegraphics[width=\linewidth]{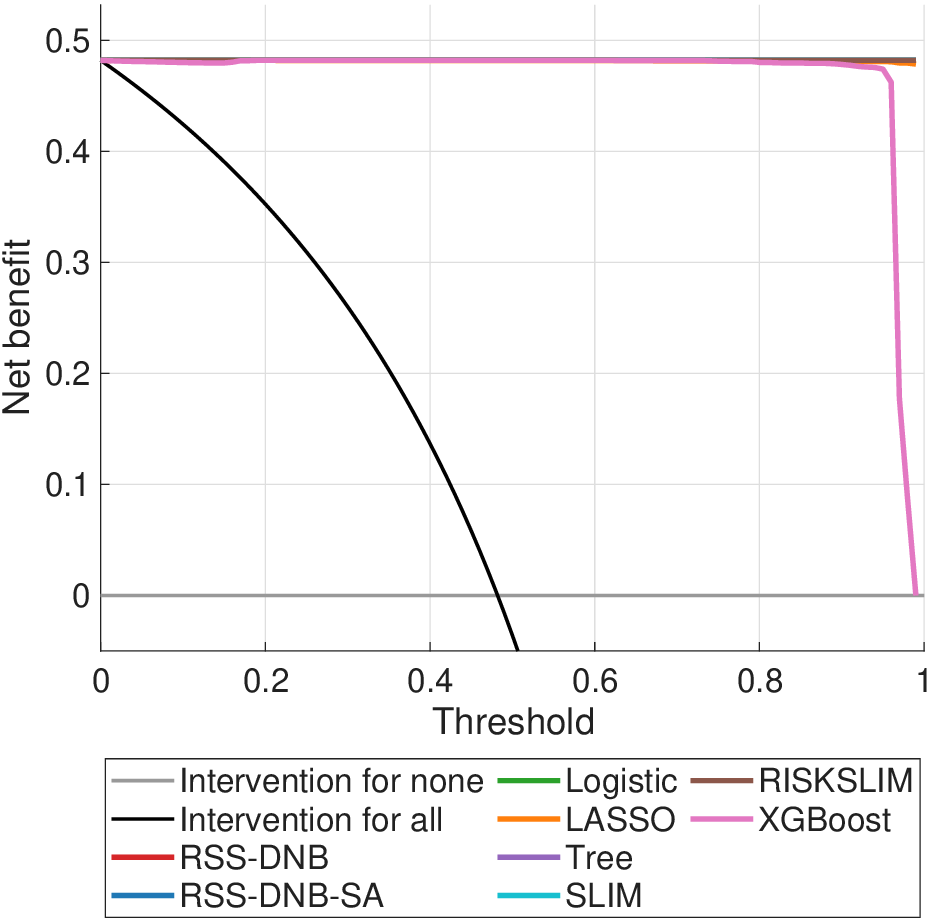}
  \caption{mushroom}
\end{subfigure}
\hfill
\begin{subfigure}{0.45\textwidth}
  \centering
  \includegraphics[width=\linewidth]{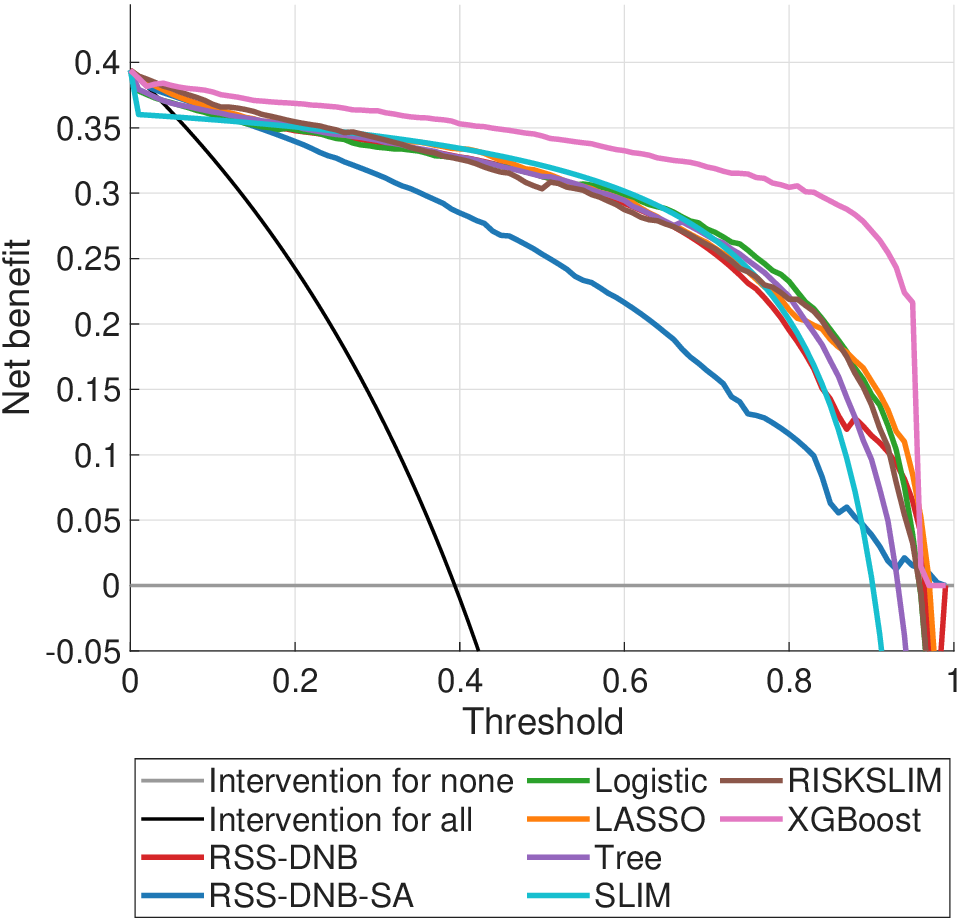}
  \caption{spambase}
\end{subfigure}

\caption{Decision curves on the remaining four datasets.}
\label{fig:dca_part2}

\end{figure}

\begin{landscape}
\footnotesize
\setlength{\tabcolsep}{4pt}
\renewcommand{\arraystretch}{1.1}
\begin{longtable}{crcccccccc}
\multicolumn{10}{l}{\textbf{Table \thetable\ :}  Performance of all models in term of  discrimination, calibration, utility, and model size on across eight datasets.} \label{tab:performance} \\
\toprule
Dataset & Metric & RSS-DNB & RSS-DNB-SA & Logistic & LASSO & Decision tree  & SLIM & RISKSLIM & XGBoost\\
\midrule
\endfirsthead

\multicolumn{10}{l}{\textbf{Table \thetable\ :}  continued} \\
\toprule
Dataset & Metric & RSS-DNB & RSS-DNB-SA & Logistic & LASSO & Decision tree  & SLIM & RISKSLIM & XGBoost \\
\midrule
\endhead

\bottomrule
\multicolumn{10}{l}{\footnotesize Notes: All values are reported as mean $\pm$ standard deviation over 10-fold cross-validation. Size is reported as mean (minimum--maximum) across the 10 folds. The best } \\
\multicolumn{10}{l}{\footnotesize and second-best results are highlighted in bold and underline, respectively.}\\
\endlastfoot
adult&Train AUROC&0.881 $\pm$ 0.002 & 0.872 $\pm$ 0.003 & \underline{0.891 $\pm$ 0.001} & 0.890 $\pm$ 0.001 & 0.873 $\pm$ 0.002 & 0.732 $\pm$ 0.007 & 0.882 $\pm$ 0.003 & \textbf{0.900 $\pm$ 0.001}\\
&Test AUROC&0.880 $\pm$ 0.007 & 0.872 $\pm$ 0.008 & \underline{0.891 $\pm$ 0.006} & 0.889 $\pm$ 0.006 & 0.869 $\pm$ 0.006 & 0.726 $\pm$ 0.011 & 0.881 $\pm$ 0.007 & \textbf{0.894 $\pm$ 0.005}\\
&Train ECE&\textbf{0.000 $\pm$ 0.000} & \textbf{0.000 $\pm$ 0.000} & 0.010 $\pm$ 0.001 & 0.012 $\pm$ 0.002 & \textbf{0.000 $\pm$ 0.000} & 0.113 $\pm$ 0.005 & 0.030 $\pm$ 0.007 & 0.033 $\pm$ 0.001\\
&Test ECE&\textbf{0.013 $\pm$ 0.003} & \underline{0.014 $\pm$ 0.004} & 0.016 $\pm$ 0.004 & 0.017 $\pm$ 0.004 & 0.016 $\pm$ 0.004 & 0.115 $\pm$ 0.007 & 0.029 $\pm$ 0.007 & 0.036 $\pm$ 0.005\\
&Train AUNBC&0.102 $\pm$ 0.000 & 0.101 $\pm$ 0.001 & \underline{0.104 $\pm$ 0.000} & 0.103 $\pm$ 0.000 & 0.101 $\pm$ 0.001 & 0.040 $\pm$ 0.007 & 0.097 $\pm$ 0.001 & \textbf{0.107 $\pm$ 0.000}\\
&Test AUNBC&0.102 $\pm$ 0.003 & 0.100 $\pm$ 0.003 & \underline{0.103 $\pm$ 0.003} & 0.103 $\pm$ 0.003 & 0.098 $\pm$ 0.003 & 0.035 $\pm$ 0.011 & 0.096 $\pm$ 0.004 & \textbf{0.103 $\pm$ 0.003}\\
&Size&22.2 (19-24) & \textbf{18.6 (15-22)} & 32.0 (32-32) & 26.3 (23-30) & 156.0 (113-199) & \underline{22.0 (12-28)} & 23.0 (14-28) & -\\
\midrule
bankruptcy&Train AUROC&\textbf{1.000 $\pm$ 0.000} & \underline{1.000 $\pm$ 0.000} & \textbf{1.000 $\pm$ 0.000} & \textbf{1.000 $\pm$ 0.000} & 0.981 $\pm$ 0.004 & \textbf{1.000 $\pm$ 0.000} & \textbf{1.000 $\pm$ 0.000} & \textbf{1.000 $\pm$ 0.000}\\
&Test AUROC&0.997 $\pm$ 0.011 & 0.997 $\pm$ 0.011 & 0.990 $\pm$ 0.032 & \underline{0.998 $\pm$ 0.006} & 0.981 $\pm$ 0.034 & 0.987 $\pm$ 0.032 & 0.996 $\pm$ 0.011 & \textbf{1.000 $\pm$ 0.000}\\
&Train ECE&\textbf{0.000 $\pm$ 0.000} & \textbf{0.000 $\pm$ 0.000} & \textbf{0.000 $\pm$ 0.000} & 0.015 $\pm$ 0.007 & \textbf{0.000 $\pm$ 0.000} & \textbf{0.000 $\pm$ 0.000} & 0.001 $\pm$ 0.001 & 0.024 $\pm$ 0.001\\
&Test ECE&0.004 $\pm$ 0.013 & 0.006 $\pm$ 0.013 & \underline{0.004 $\pm$ 0.013} & 0.016 $\pm$ 0.012 & \textbf{0.000 $\pm$ 0.000} & 0.004 $\pm$ 0.013 & 0.007 $\pm$ 0.015 & 0.027 $\pm$ 0.015\\
&Train AUNBC&\textbf{0.572 $\pm$ 0.016} & 0.569 $\pm$ 0.016 & \textbf{0.572 $\pm$ 0.016} & 0.567 $\pm$ 0.014 & 0.541 $\pm$ 0.017 & \textbf{0.572 $\pm$ 0.016} & \underline{0.572 $\pm$ 0.016} & 0.570 $\pm$ 0.016\\
&Test AUNBC&\textbf{0.568 $\pm$ 0.147} & \textbf{0.568 $\pm$ 0.147} & 0.561 $\pm$ 0.135 & 0.559 $\pm$ 0.137 & 0.541 $\pm$ 0.157 & 0.561 $\pm$ 0.135 & 0.558 $\pm$ 0.139 & \underline{0.568 $\pm$ 0.150}\\
&Size&\underline{2.9 (2-3)} & \textbf{2.1 (2-3)} & 6.0 (6-6) & 4.1 (4-5) & 3.0 (3-3) & \underline{2.9 (2-3)} & 5.2 (5-6) & -\\
\midrule
breastcancer&Train AUROC&0.992 $\pm$ 0.002 & 0.990 $\pm$ 0.002 & \underline{0.996 $\pm$ 0.001} & 0.996 $\pm$ 0.001 & 0.989 $\pm$ 0.006 & 0.985 $\pm$ 0.002 & 0.993 $\pm$ 0.002 & \textbf{1.000 $\pm$ 0.000}\\
&Test AUROC&0.985 $\pm$ 0.015 & 0.972 $\pm$ 0.030 & \textbf{0.995 $\pm$ 0.005} & \underline{0.995 $\pm$ 0.006} & 0.964 $\pm$ 0.025 & 0.960 $\pm$ 0.031 & 0.990 $\pm$ 0.009 & 0.992 $\pm$ 0.010\\
&Train ECE&\textbf{0.000 $\pm$ 0.000} & \textbf{0.000 $\pm$ 0.000} & 0.016 $\pm$ 0.002 & 0.033 $\pm$ 0.007 & \textbf{0.000 $\pm$ 0.000} & 0.003 $\pm$ 0.001 & 0.013 $\pm$ 0.004 & 0.034 $\pm$ 0.004\\
&Test ECE&\underline{0.022 $\pm$ 0.010} & 0.022 $\pm$ 0.013 & 0.032 $\pm$ 0.010 & 0.048 $\pm$ 0.015 & 0.028 $\pm$ 0.016 & \textbf{0.015 $\pm$ 0.017} & 0.031 $\pm$ 0.013 & 0.041 $\pm$ 0.011\\
&Train AUNBC&\underline{0.325 $\pm$ 0.006} & 0.320 $\pm$ 0.007 & 0.318 $\pm$ 0.006 & 0.314 $\pm$ 0.006 & 0.317 $\pm$ 0.012 & 0.321 $\pm$ 0.005 & 0.306 $\pm$ 0.007 & \textbf{0.342 $\pm$ 0.006}\\
&Test AUNBC&0.305 $\pm$ 0.057 & 0.298 $\pm$ 0.066 & \textbf{0.310 $\pm$ 0.054} & \underline{0.307 $\pm$ 0.053} & 0.280 $\pm$ 0.054 & 0.292 $\pm$ 0.064 & 0.297 $\pm$ 0.061 & 0.306 $\pm$ 0.057\\
&Size&6.5 (5-7) & \textbf{3.5 (3-5)} & 9.0 (9-9) & 8.5 (8-9) & 18.2 (11-29) & 6.3 (5-8) & \underline{3.8 (3-6)} & -\\
\midrule
haberman&Train AUROC&\underline{0.733 $\pm$ 0.015} & 0.728 $\pm$ 0.011 & 0.701 $\pm$ 0.009 & 0.584 $\pm$ 0.108 & 0.554 $\pm$ 0.087 & 0.637 $\pm$ 0.015 & 0.500 $\pm$ 0.000 & \textbf{0.930 $\pm$ 0.035}\\
&Test AUROC&0.662 $\pm$ 0.119 & \underline{0.667 $\pm$ 0.138} & \textbf{0.684 $\pm$ 0.113} & 0.571 $\pm$ 0.094 & 0.516 $\pm$ 0.091 & 0.600 $\pm$ 0.093 & 0.500 $\pm$ 0.000 & 0.666 $\pm$ 0.123\\
&Train ECE&\textbf{0.000 $\pm$ 0.000} & \textbf{0.000 $\pm$ 0.000} & 0.054 $\pm$ 0.011 & 0.026 $\pm$ 0.034 & \textbf{0.000 $\pm$ 0.000} & 0.047 $\pm$ 0.006 & 0.008 $\pm$ 0.004 & 0.177 $\pm$ 0.020\\
&Test ECE&0.118 $\pm$ 0.061 & 0.123 $\pm$ 0.049 & 0.127 $\pm$ 0.043 & 0.079 $\pm$ 0.038 & 0.073 $\pm$ 0.044 & \textbf{0.049 $\pm$ 0.035} & \underline{0.061 $\pm$ 0.037} & 0.154 $\pm$ 0.029\\
&Train AUNBC&\underline{0.476 $\pm$ 0.012} & 0.474 $\pm$ 0.011 & 0.452 $\pm$ 0.014 & 0.427 $\pm$ 0.013 & 0.435 $\pm$ 0.014 & 0.355 $\pm$ 0.018 & 0.422 $\pm$ 0.012 & \textbf{0.519 $\pm$ 0.023}\\
&Test AUNBC&0.440 $\pm$ 0.107 & \underline{0.442 $\pm$ 0.111} & \textbf{0.442 $\pm$ 0.100} & 0.425 $\pm$ 0.105 & 0.418 $\pm$ 0.110 & 0.319 $\pm$ 0.171 & 0.421 $\pm$ 0.106 & 0.439 $\pm$ 0.104\\
&Size&2.4 (2-3) & 1.9 (1-2) & 3.0 (3-3) & \underline{0.4 (0-1)} & 2.6 (1-11) & 3.0 (3-3) & \textbf{0.0 (0-0)} & -\\
\midrule
heartdisease&Train AUROC&0.926 $\pm$ 0.008 & 0.916 $\pm$ 0.010 & \underline{0.938 $\pm$ 0.005} & 0.924 $\pm$ 0.006 & 0.887 $\pm$ 0.030 & 0.919 $\pm$ 0.008 & 0.930 $\pm$ 0.007 & \textbf{1.000 $\pm$ 0.000}\\
&Test AUROC&0.819 $\pm$ 0.086 & 0.841 $\pm$ 0.077 & 0.870 $\pm$ 0.067 & \textbf{0.897 $\pm$ 0.069} & 0.785 $\pm$ 0.109 & 0.823 $\pm$ 0.076 & 0.853 $\pm$ 0.066 & \underline{0.892 $\pm$ 0.055}\\
&Train ECE&\textbf{0.000 $\pm$ 0.000} & \textbf{0.000 $\pm$ 0.000} & 0.031 $\pm$ 0.007 & 0.088 $\pm$ 0.014 & \textbf{0.000 $\pm$ 0.000} & 0.050 $\pm$ 0.005 & 0.037 $\pm$ 0.009 & 0.172 $\pm$ 0.006\\
&Test ECE&\underline{0.079 $\pm$ 0.033} & \textbf{0.079 $\pm$ 0.052} & 0.137 $\pm$ 0.055 & 0.158 $\pm$ 0.030 & 0.120 $\pm$ 0.043 & 0.099 $\pm$ 0.031 & 0.132 $\pm$ 0.048 & 0.130 $\pm$ 0.038\\
&Train AUNBC&\underline{0.359 $\pm$ 0.011} & 0.352 $\pm$ 0.008 & 0.332 $\pm$ 0.010 & 0.305 $\pm$ 0.009 & 0.301 $\pm$ 0.021 & 0.359 $\pm$ 0.016 & 0.326 $\pm$ 0.012 & \textbf{0.378 $\pm$ 0.009}\\
&Test AUNBC&0.206 $\pm$ 0.145 & 0.248 $\pm$ 0.128 & 0.269 $\pm$ 0.089 & \textbf{0.280 $\pm$ 0.079} & 0.236 $\pm$ 0.133 & 0.228 $\pm$ 0.163 & 0.257 $\pm$ 0.106 & \underline{0.276 $\pm$ 0.089}\\
&Size&16.3 (12-22) & 23.5 (20-26) & 24.9 (24-25) & \textbf{11.0 (9-13)} & \underline{13.2 (5-31)} & 15.5 (13-17) & 20.2 (14-29) & -\\
\midrule
mammo&Train AUROC&0.859 $\pm$ 0.003 & 0.852 $\pm$ 0.006 & \underline{0.860 $\pm$ 0.004} & 0.852 $\pm$ 0.005 & 0.827 $\pm$ 0.017 & 0.820 $\pm$ 0.003 & 0.852 $\pm$ 0.004 & \textbf{0.877 $\pm$ 0.004}\\
&Test AUROC&0.845 $\pm$ 0.031 & 0.847 $\pm$ 0.033 & \textbf{0.851 $\pm$ 0.035} & 0.848 $\pm$ 0.039 & 0.812 $\pm$ 0.038 & 0.809 $\pm$ 0.029 & \underline{0.848 $\pm$ 0.034} & 0.845 $\pm$ 0.037\\
&Train ECE&\textbf{0.000 $\pm$ 0.000} & \textbf{0.000 $\pm$ 0.000} & 0.024 $\pm$ 0.004 & 0.056 $\pm$ 0.014 & \textbf{0.000 $\pm$ 0.000} & 0.062 $\pm$ 0.002 & 0.026 $\pm$ 0.005 & 0.040 $\pm$ 0.003\\
&Test ECE&0.078 $\pm$ 0.019 & 0.078 $\pm$ 0.020 & 0.095 $\pm$ 0.025 & 0.086 $\pm$ 0.023 & \textbf{0.060 $\pm$ 0.029} & \underline{0.069 $\pm$ 0.023} & 0.078 $\pm$ 0.021 & 0.097 $\pm$ 0.034\\
&Train AUNBC&\underline{0.262 $\pm$ 0.006} & 0.257 $\pm$ 0.006 & 0.256 $\pm$ 0.006 & 0.248 $\pm$ 0.007 & 0.247 $\pm$ 0.008 & 0.173 $\pm$ 0.009 & 0.254 $\pm$ 0.006 & \textbf{0.266 $\pm$ 0.005}\\
&Test AUNBC&\underline{0.252 $\pm$ 0.052} & 0.251 $\pm$ 0.052 & 0.248 $\pm$ 0.053 & 0.246 $\pm$ 0.048 & 0.239 $\pm$ 0.051 & 0.158 $\pm$ 0.088 & \textbf{0.253 $\pm$ 0.055} & 0.249 $\pm$ 0.051\\
&Size&6.5 (6-7) & \textbf{4.6 (4-6)} & 11.0 (11-11) & \underline{4.8 (3-6)} & 9.4 (5-19) & 9.5 (9-11) & 5.9 (5-9) & -\\
\midrule
mushroom&Train AUROC&\textbf{1.000 $\pm$ 0.000} & \textbf{1.000 $\pm$ 0.000} & \textbf{1.000 $\pm$ 0.000} & \textbf{1.000 $\pm$ 0.000} & \textbf{1.000 $\pm$ 0.000} & \textbf{1.000 $\pm$ 0.000} & \textbf{1.000 $\pm$ 0.000} & \textbf{1.000 $\pm$ 0.000}\\
&Test AUROC&\textbf{1.000 $\pm$ 0.000} & \textbf{1.000 $\pm$ 0.000} & \textbf{1.000 $\pm$ 0.000} & \textbf{1.000 $\pm$ 0.000} & \textbf{1.000 $\pm$ 0.000} & \textbf{1.000 $\pm$ 0.000} & \textbf{1.000 $\pm$ 0.000} & \textbf{1.000 $\pm$ 0.000}\\
&Train ECE&\textbf{0.000 $\pm$ 0.000} & \textbf{0.000 $\pm$ 0.000} & \textbf{0.000 $\pm$ 0.000} & 0.001 $\pm$ 0.000 & \textbf{0.000 $\pm$ 0.000} & \textbf{0.000 $\pm$ 0.000} & 0.000 $\pm$ 0.000 & 0.005 $\pm$ 0.001\\
&Test ECE&\textbf{0.000 $\pm$ 0.000} & \textbf{0.000 $\pm$ 0.000} & \textbf{0.000 $\pm$ 0.000} & 0.001 $\pm$ 0.001 & \textbf{0.000 $\pm$ 0.000} & \textbf{0.000 $\pm$ 0.000} & 0.000 $\pm$ 0.000 & 0.005 $\pm$ 0.001\\
&Train AUNBC&\textbf{0.482 $\pm$ 0.002} & \textbf{0.482 $\pm$ 0.002} & \textbf{0.482 $\pm$ 0.002} & \underline{0.482 $\pm$ 0.002} & \textbf{0.482 $\pm$ 0.002} & \textbf{0.482 $\pm$ 0.002} & \textbf{0.482 $\pm$ 0.002} & 0.481 $\pm$ 0.002\\
&Test AUNBC&\textbf{0.482 $\pm$ 0.017} & \textbf{0.482 $\pm$ 0.017} & \textbf{0.482 $\pm$ 0.017} & \underline{0.482 $\pm$ 0.017} & \textbf{0.482 $\pm$ 0.017} & \textbf{0.482 $\pm$ 0.017} & \textbf{0.482 $\pm$ 0.017} & 0.481 $\pm$ 0.017\\
&Size&\underline{22.3 (19-43)} & 35.3 (30-43) & 39.2 (35-45) & 25.8 (22-27) & 24.6 (23-27) & \textbf{8.8 (8-10)} & 44.5 (41-48) & -\\
\midrule
spambase&Train AUROC&0.954 $\pm$ 0.010 & 0.917 $\pm$ 0.037 & 0.965 $\pm$ 0.039 & 0.974 $\pm$ 0.001 & \underline{0.989 $\pm$ 0.008} & 0.943 $\pm$ 0.004 & 0.973 $\pm$ 0.002 & \textbf{0.994 $\pm$ 0.000}\\
&Test AUROC&0.951 $\pm$ 0.017 & 0.919 $\pm$ 0.038 & 0.959 $\pm$ 0.045 & \underline{0.970 $\pm$ 0.009} & 0.945 $\pm$ 0.015 & 0.925 $\pm$ 0.012 & 0.969 $\pm$ 0.009 & \textbf{0.987 $\pm$ 0.003}\\
&Train ECE&\textbf{0.000 $\pm$ 0.000} & \textbf{0.000 $\pm$ 0.000} & 0.030 $\pm$ 0.025 & 0.046 $\pm$ 0.004 & \textbf{0.000 $\pm$ 0.000} & 0.026 $\pm$ 0.004 & 0.027 $\pm$ 0.008 & 0.019 $\pm$ 0.002\\
&Test ECE&\textbf{0.011 $\pm$ 0.008} & 0.028 $\pm$ 0.009 & 0.041 $\pm$ 0.029 & 0.050 $\pm$ 0.012 & 0.029 $\pm$ 0.009 & 0.033 $\pm$ 0.007 & 0.039 $\pm$ 0.012 & \underline{0.027 $\pm$ 0.006}\\
&Train AUNBC&0.298 $\pm$ 0.008 & 0.249 $\pm$ 0.037 & 0.313 $\pm$ 0.016 & 0.307 $\pm$ 0.004 & \underline{0.356 $\pm$ 0.016} & 0.315 $\pm$ 0.005 & 0.306 $\pm$ 0.004 & \textbf{0.361 $\pm$ 0.002}\\
&Test AUNBC&0.295 $\pm$ 0.020 & 0.249 $\pm$ 0.044 & 0.303 $\pm$ 0.026 & \underline{0.303 $\pm$ 0.019} & 0.297 $\pm$ 0.023 & 0.288 $\pm$ 0.028 & 0.299 $\pm$ 0.020 & \textbf{0.343 $\pm$ 0.016}\\
&Size&\underline{33.2 (28-37)} & \textbf{32.1 (23-39)} & 57.0 (57-57) & 51.3 (47-54) & 219.2 (113-329) & 39.2 (35-44) & 35.5 (32-38) & -\\
\end{longtable}
\end{landscape}
\clearpage
\section{Case Study: Give Me Some Credit} \label{sec:application}
The illustrative application in \Cref{subsec:method_example} demonstrated how RSS-DNB constructs an interpretable risk scoring system and how the resulting scorecard can be used to support decision making. In this section, we further evaluate the empirical performance of the proposed approach on the Give Me Some Credit dataset. This case study aims to assess the generalization ability of RSS-DNB on unseen data from the perspectives of discrimination, calibration, and decision utility. We compare RSS-DNB with several representative interpretable and a non-interpretable method under a common experimental setting.

\subsection{Experimental setup}
To evaluate the practical performance of RSS-DNB, we conducted a case study on the Give Me Some Credit dataset. The dataset was randomly divided into a training set (75\%) and a test set (25\%). All preprocessing steps described in \Cref{subsec:method_example} were fitted using the training set and subsequently applied to the test data.

Following the experimental design in \Cref{sec:experiments}, we compared RSS-DNB and RSS-DNB-SA with the same set of baseline models, including Logistic Regression, LASSO, Decision Tree, SLIM, RISKSLIM, and XGBoost. These models cover traditional statistical methods, interpretable machine learning approaches, sparse integer scoring systems, and a widely used gradient boosting method. All models were trained on the training set and evaluated on the test set.

Following the measurement convention, the models were evaluated from three aspects: AUROC was used to assess discrimination ability, ECE was used to evaluate calibration performance, and d-AUNBC was used as measure of decision utility. In addition, the number of non-zero predictors or decision tree nodes was reported to quantify model complexity.

\subsection{Results and Observations}
\Cref{tab:results_give} summarizes the performance of RSS-DNB, RSS-DNB-SA and other models on the test set. Overall, RSS-DNB and RSS-DNB-SA achieved competitive performance across discrimination, calibration, and decision utility while maintaining a sparse model structure. This is consistent with the conclusions of the previous Section.
\begin{table}[htbp]
    \centering
    \renewcommand{\arraystretch}{1.2}
    \caption{Performance comparison on the Give Me Some Credit dataset.}
    \label{tab:results_give}
    \begin{tabular}{ccccc}
    \toprule
    Model & AUROC & ECE & AUNBC & Size\\
    \midrule
    RSS-DNB	& \underline{0.8413} & \underline{0.0015}& \underline{0.0115} & 4\\
    RSS-DNB-SA & 0.8362 & \textbf{0.0009} & \underline{0.0115} & \textbf{2}	\\
    Logistic & \textbf{0.8415} & 0.0037 & \textbf{0.0116} & 8	\\
    Lasso & 0.8381 & 0.0046 & \textbf{0.0116} & \underline{3}	\\
    Decision tree & 0.7656 & \underline{0.0015} & \textbf{0.0116} & 7	\\
    SLIM & 0.5144 & 0.0647 & 0.0037 & \underline{3}	\\
    RISKSLIM & 0.8136 & 0.0121 & 0.0099 & 4	\\
    XGBoost & 0.8388 & 0.0108 & 0.0103 & -\\
    \bottomrule
    \end{tabular}
\end{table}

An important observation from the experimental results is that multiple models with different structures achieved similar performance. In particular, logistic regression, LASSO, and the proposed RSS-DNB model exhibited similar levels of discrimination, calibration, and decision utility, despite large differences in model complexity and coefficient structure. This phenomenon is related to the so-called Rashomon set, which refers to the existence of a large set of models that achieve near-optimal performance on a given dataset \citep{rudinStopExplainingBlack2019}. Within this set, we can often find at least one model that is inherently interpretable. From this perspective, the results suggest that when predictive performance is comparable, it is preferable to select models that are more interpretable and easier to use in practice. The RSS-DNB models have a sparse structure and small integer coefficients, which can provide transparent and meaningful representations, making them particularly suitable for decision support in real-world environments.
\section{Conclusion}
In this work, we studied the problem of developing risk scoring systems for decision-making, with a focus on optimizing decision utility rather than conventional predictive metrics. Existing approaches primarily emphasize discrimination and calibration, but optimizing these metrics alone does not guarantee improved decision utility. Therefore, We proposed the RSS-DNB model, a sparse integer linear model formulated to maximize a weighted sum of net benefit over a range of thresholds across specified thresholds minus a sparsity penalty. We established theoretical connections between net benefit, discrimination, and calibration. In particular, we proved that there exists a lower bound on the discrimination of the model (measured by d-AUROC), which is controlled by the model utility (measured by d-AUNBC). Specifically, this lower bound increases with d-AUNBC, implying that optimizing model utility will not result in models with poor discrimination. Furthermore, by leveraging the relationship between net benefit and calibration, we developed an algorithm that improves the calibration of a given model while simultaneously increasing its AUNBC, without compromising its discrimination performance. We also provided guarantees on the learning capacity and generalization performance of the proposed model.

Empirical results in both public datasets and a large-scale credit risk dataset demonstrate that the proposed method, while trained with a decision utility objective, achieves competitive discrimination and calibration relative to baseline models. This observation is qualitatively consistent with our theoretical findings. Moreover, the sparse linear structure with integer coefficients enhances interpretability, while the integer programming framework allows the incorporation of various operational constraints, which facilitates practical deployment. As a result, the resulting scoring system is both transparent and readily usable in real decision-making contexts.

This study has several limitations. First, the proposed RSS-DNB model is based on a sparse linear structure with integer coefficients. While this structure has strong inherent interpretability, it cannot capture the complex nonlinear relationships and interactions among predictors.

Second, under our integer linear programming formulation, the number of decision variables grows with both sample size and the number of decision thresholds, resulting in large-scale optimization problem. In practice, solving this problem exactly may require substantial computational time, making it impractical. Heuristic methods, such as the simulated annealing algorithm adopted in this work, offer faster solutions, they cannot guarantee a global optimum or even a satisfactory result. Therefore, the development of more efficient optimization algorithms remains an important direction for future research. {Future research may investigate stronger heuristic baselines, such as coordinate descent, local search, randomized rounding, and greedy score optimization, to further improve the scalability and solution quality of RSS-DNB for larger-scale problems. }Additionally, exploring alternative formulations that reduce the problem complexity may further enhance computational efficiency.

Third, the performance of RSS-DNB is affected by the setting of the decision threshold grid used to approximate the area under the net benefit curve. Furthermore, the weighting scheme in objective function is determined by the spacing between adjacent thresholds, may not always fully reflect the specific application preferences. The sensitivity analysis shows comparable performance under several settings, while also revealing notable differences for some datasets. The optimal selection of thresholds and weights may depend on the application environment, and future research could explore adaptive or data-driven threshold selection strategies.

Finally, although RSS-DNB was evaluated on multiple benchmark datasets and a large-scale case study, external validation using independent cohorts or prospective studies is still necessary before practical deployment. Future studies should explore the generalization and robustness of RSS-DNB across different application domains.
\section*{Declarations}
    \paragraph{Funding.} This research did not receive funding.
    \paragraph{Competing interests.}  The authors declare no competing interests.
    \paragraph{Data availability.} The data used in this study are publicly available. The datasets used in \Cref{sec:experiments} are obtained from the UCI Machine Learning Repository, and the corresponding preprocessed versions used in this study are available at: \url{https://static-content.springer.com/esm/art%3A10.1007%2Fs10994-015-5528-6/MediaObjects/10994_2015_5528_MOESM1_ESM.zip}. The credit risk dataset Give Me Some Credit are available at: \url{https://www.kaggle.com/c/GiveMeSomeCredit}.
    \paragraph{Code availability.} The code for the proposed method is available at: \url{https://github.com/ccwwhh100/RSS_DNB}.
    \paragraph{Author contributions.} W.C. wrote the main manuscript text. S.I.B. wrote the introduction and provided conceptual guidance for the illustrative application and case study. All authors reviewed the manuscript.
\clearpage
\appendix

\renewcommand{\thesection}{Appendix \Alph{section}}

\section{Proofs of Main Results}
\subsection*{NP-hardness Proof} \label{app:NPhard}
\begin{definition}[Bounded Integer Minimum-Disagreement Halfspace (MDH)]
Let $D_N = \{(\bm{x}_j, y_j)\}_{j=1}^N$ be a labeled dataset with feature vectors $\bm{x}_j \in \mathbb{R}^P$ and labels $y_j \in \{0,1\}$. Let $L, \Gamma \in \mathbb{Z}^+$ be bounds on coefficients and intercept. The MDH problem asks for $\bm{\lambda} \in \mathbb{Z}^P$ and $T \in \mathbb{Z}$ with $-L \le \lambda_k \le L$ for all $k=1, \dots, P$ and $-\Gamma \le T \le \Gamma$ that minimize the number of misclassifications:
\begin{equation}
E(\lambda, T) := \sum_{j=1}^N \Big(
I(\bm{x}_j^\top \lambda < T, y_j = 1)
+
I(\bm{x}_j^\top \lambda \ge T, y_j = 0)
\Big).
\label{eq:MDH}
\end{equation}
\end{definition}

It is known that directly minimizing the 0-1 loss (i.e., the number of misclassifications) for linear or halfspace classifiers is $\mathcal{NP}$-hard, including in discrete settings with integer-weighted classifiers; see, for example, hardness results for noisy/agnostic halfspace learning \citep{HoeffgenSimonVanHorn1995,AroraBabaiSternSweedyk1997} and for direct 0--1 loss optimization in binary classification \citep{NguyenSanner2013}. The SLIM work of \cite{ustunSupersparseLinearInteger2016} likewise formulates bounded-integer, 0--1 loss scoring systems as mixed-integer optimization problems and treats their training as computationally challenging, but it does not contain a standalone $\mathcal{NP}$-hardness proof. Taken together, these references show that minimum-disagreement and discrete 0-1 loss problems for integer-weighted linear classifiers are $\mathcal{NP}$-hard in general.

\begin{theorem}
The optimization problem \eqref{ilp:original} is $\mathcal{NP}$-hard.
\end{theorem}
\begin{proof}
We reduce MDH to a special instance of the risk-scoring problem. Fix an MDH instance $D_N$ and bounds $L,\Gamma$. Define
\[
\mathcal{L} := [-L, L]^P \cap \mathbb{Z}^P, \qquad \mathcal{B} := [-\Gamma, \Gamma]^2 \cap \mathbb{Z}^2.
\]
Consider a single operating threshold ($M = 1$) with risk thresholds $p_0 = 0$, $p_1 = \tfrac{1}{2}$ and weights $\omega_0 = 0$, $\omega_1 = 1$. Then, only the $i=1$ term contributes to the net-benefit objective. The problem \eqref{ilp:original} reduces to \begin{align*}
\min_{\bm{\lambda}, T_1} \quad &
J(\lambda, T_1) = - \frac{1}{N}\big(\mathrm{TP}_1 - \mathrm{FP}_1\big) + C_0 \|\bm{\lambda}\|_0, \\
\text{s.t.} \quad & \mathrm{TP}_1 = \sum_{j=1}^N
I(\bm{x}_j^\top \bm{\lambda} \ge T_1, y_j = 1),
\\
& \mathrm{FP}_1 = \sum_{j=1}^N I(\bm{x}_j^\top \bm{\lambda} \ge T_1, y_j = 0),
\\
& \bm{\lambda} \in \mathcal{L},\quad T_1 \in [-\Gamma, \Gamma] \cap \mathbb{Z},
\end{align*}
where $\|\bm{\lambda}\|_0$ is the number of nonzero entries of $\bm{\lambda}$. Let $N^+ := \sum_{j=1}^N I(y_j = 1)$ and define
\[
\mathrm{FN}_1 := \sum_{j=1}^N I(\bm{x}_j^\top \bm{\lambda} < T_1, y_j = 1).
\]
Then, $\mathrm{TP}_1 + \mathrm{FN}_1 = N^+$, so $\mathrm{TP}_1 = N^+ - \mathrm{FN}_1$ and
\begin{align*}
- \big(\mathrm{TP}_1 - \mathrm{FP}_1\big)
& = - \big(N^+ - \mathrm{FN}_1 - \mathrm{FP}_1\big) \\
& = \mathrm{FN}_1 + \mathrm{FP}_1 - N^+.
\end{align*}
The sum $\mathrm{FN}_1 + \mathrm{FP}_1$ is exactly the MDH error $E(\bm{\lambda}, T_1)$ in \eqref{eq:MDH}. Therefore,
\begin{equation}
J(\bm{\lambda}, T_1) = \frac{E(\bm{\lambda}, T_1)}{N} - \frac{N^+}{N} + C_0 \|\bm{\lambda}\|_0.
\label{eq:J}
\end{equation}
Since $E(\bm{\lambda}, T_1) \in \{0,1,\dots,N\}$, the term $E(\bm{\lambda},T_1)/N$ changes only in increments of $1/N$. Assume $0 < C_0 < \frac{1}{N(P+1)}$. Because $\bm{\lambda} \in \mathbb{Z}^P$ and $\|\bm{\lambda}\|_0 \le P$, we have
\begin{equation}
C_0 \|\bm{\lambda}\|_0 \le C_0 P < \frac{P}{N(P+1)} < \frac{1}{N}.
\label{eq:penalty}
\end{equation}
Consider two feasible solutions $(\bm{\lambda}_A,T_A)$ and $(\bm{\lambda}_B,T_B)$ with $E(\bm{\lambda}_A,T_A) \le E(\bm{\lambda}_B,T_B) - 1$. From \eqref{eq:J} and \eqref{eq:penalty}, we have
\begin{align*}
J(\bm{\lambda}_B,T_B) - J(\bm{\lambda}_A,T_A)
& = \frac{E(\bm{\lambda}_B,T_B) - E(\bm{\lambda}_A,T_A)}{N} + C_0\big(\|\bm{\lambda}_B\|_0 - \|\bm{\lambda}_A\|_0\big) \\
& \ge \frac{1}{N} - C_0 P \\
& > \frac{1}{N} - \frac{1}{N} = 0.
\end{align*}
Thus any unit decrease in the MDH error $E(\bm{\lambda},T_1)$ strictly improves the objective, regardless of changes in the sparsity term. Minimizing $J(\bm{\lambda},T_1)$ is therefore equivalent to:
(i) minimizing $E(\bm{\lambda},T_1)$ over $\bm{\lambda} \in \mathcal{L}$ and $T_1 \in [-\Gamma,\Gamma] \cap \mathbb{Z}$, and (ii) among all minimizers of $E$, selecting one with minimal $\|\bm{\lambda}\|_0$. Step (i) is exactly the MDH problem. The mapping from an MDH instance $(D_N,L,\Gamma)$ to the corresponding risk-scoring instance is computable in time polynomial in $N$, $P$, $L$, and $\Gamma$. Since MDH is $\mathcal{NP}$-hard (as a discrete minimum-disagreement halfspace problem in the sense explained above) and our problem contains MDH as a special case, the optimization problem \eqref{ilp:original} is $\mathcal{NP}$-hard as well.
\end{proof}
\subsection*{Proof of \Cref{thm:auroc_aunbc}}
\begin{proof}
   First, since ${TP}_i \ge 0$ and ${FP}_i \le N^-$ for all $i=1,2,\cdots,M$, we have 
    \begin{equation}
        \begin{aligned}
            \mathrm{AUNBC}(\bm p;c) &\ge  \frac{1}{N} \left(N^+p_1+\sum_{i=1}^M \left(p_{i+1} - p_i\right) \left(0 - N^-\cdot \frac{p_i}{1-p_i}\right)\right)\\
            &=a_0p_1-(1-a_0)P_M,
        \end{aligned}
 \end{equation}
    where $a_0=\frac{N^+}{N}$, $P_M=\sum_{i=1}^M \frac{(p_{i+1}-p_i)p_i}{1-p_i}$.
    
    Let $b_0 = \frac{N^-}{N}$, and $a_{M+1}=b_{M+1}=0$. Vectors $\bm{a}=(a_1,\cdots,a{_M})$ and $\bm{b}=(b_1,\cdots,b_M)$ satisfy $a_i = \frac{{TP}_i}{N}$ and $b_i = \frac{{FP}_i}{N}$, $i=1,2,\cdots,M$. Since $0=p_0<p_1<\cdots<p_M<1$, we have $a_0 \ge a_1 \ge \cdots \ge a_M \ge 0$ and $b_0 \ge b_1 \ge \cdots \ge b_M \ge 0$. AUNBC$(\bm p;c)$ and AUROC$(\bm p;c)$ can be represented by $F(\bm{a},\bm{b}) := \sum_{i=0}^M \left(p_{i+1} - p_i \right)\left(a_i-b_i\cdot\frac{p_i}{1-p_i}\right)$ and $G(\bm{a},\bm{b}) := \frac{1}{a_0 b_0} \sum_{i=0}^M(b_i-b_{i+1})a_i$, respectively. Let
    \begin{align}
        \mathcal{K}:=\left\{1 \le k \le M \mid b_k \ge \frac{(1-p_k)\left(1-G(\bm{a},\bm{b})\right)a_0 b_0}{\sum_{i=1}^M (a_{i-1}-a_i)(1-p_i)} \right\}.
    \end{align}
    Then, $\mathcal{K}$ is a non-empty set, otherwise for all $1\le k \le M$, $b_k<\frac{(1-p_k)\left(1-G(\bm{a},\bm{b})\right)a_0 b_0}{\sum_{i=1}^M (a_{i-1}-a_i)(1-p_i)}$ and
    \begin{equation}\label{eq:contradiction}
    \begin{aligned}
        G(\bm{a},\bm{b}) a_0 b_0 &= \sum_{i=0}^M(b_i-b_{i+1})a_i \\
        &= a_0 b_0 - \sum_{i=1}^M(a_{i-1}-a_{i})b_i\\
        &> a_0 b_0 - \sum_{i=1}^M(a_{i-1}-a_{i}) \frac{(1-p_i)\left(1-G(\bm{a},\bm{b})\right)a_0 b_0}{\sum_{j=1}^M (a_{j-1}-a_j)(1-p_j)}\\
        &= a_0 b_0 - \frac{\left(1-G(\bm{a},\bm{b})\right)a_0 b_0}{\sum_{j=1}^M (a_{j-1}-a_j)(1-p_j)}\sum_{i=1}^M (a_{i-1}-a_i)(1-p_i) \\
        &= G(\bm{a},\bm{b}) a_0 b_0.
    \end{aligned}
    \end{equation}
    
    There is a contradiction in (\ref{eq:contradiction}). Thus, there exists at least one integer $1 \le K \le M$ such that $b_K \ge \frac{(1-p_K)\left(1-G(\bm{a},\bm{b})\right)a_0 b_0}{\sum_{i=1}^M (a_{i-1}-a_i)(1-p_i)}$.
    Let vectors $\bm{a}'=(a_1',\cdots,a_M')$ and $\bm{b}'=(b_1',\cdots,b_M')$ satisfy $a_1'=\cdots=a_{K-1}'=a_0$, $a_K'=\cdots=a_M'=\frac{\sum_{i=0}^M a_i(p_{i+1}-p_i)-a_0 p_K}{1-p_K}$, $b_1'=\cdots=b_K'=\frac{(1-p_K)\left(1-G(\bm{a},\bm{b})\right)a_0 b_0}{\sum_{i=1}^M (a_{i-1}-a_i)(1-p_i)} \le b_K$, $b_{K+1}'=\cdots=b_M'=0$. Then, we have
    \begin{equation}
    \begin{aligned}
        a_0 - a_K' &= \frac{a_0-\sum_{i=0}^M a_i(p_{i+1}-p_i)}{1-p_K} \\
        &= \frac{a_0-\sum_{i=1}^M (a_{i-1}-a_i)p_i-a_M}{1-p_K}\\
        &= \frac{\sum_{i=1}^M (a_{i-1}-a_i)(1-p_i)}{1-p_K}\\
        &= \frac{\left(1-G(\bm{a},\bm{b})\right)a_0 b_0}{b_1'} \ge 0,
    \end{aligned}        
    \end{equation}
    \begin{equation}
        \begin{aligned}
            G(\bm{a}',\bm{b}') &= \frac{1}{a_0 b_0}\left((b_0-b_1')a_0+b_K'a_K'\right) \\
            &= \frac{1}{a_0 b_0} \left(a_0 b_0 - b_1'(a_0-a_K')\right) \\
            &= G(\bm{a},\bm{b}),
        \end{aligned}
    \end{equation}
    and
    \begin{equation}
        \begin{aligned}
            F(\bm{a}',\bm{b}') &= \sum_{i=0}^M (p_{i+1}-p_i)\left(a_i'-b_i'\cdot\frac{p_i}{1-p_i}\right) \\
            &= p_K a_0 +(1-p_K)a_K'-\sum_{i=0}^K (p_{i+1}-p_i)b_i'\cdot\frac{p_i}{1-p_i}\\
            &= \sum_{i=0}^M a_i(p_{i+1}-p_i) - b_1'\sum_{i=1}^K \frac{(p_{i+1}-p_i)p_i}{1-p_i}\\
            &\ge \sum_{i=0}^M a_i(p_{i+1}-p_i) - b_K\sum_{i=1}^K \frac{(p_{i+1}-p_i)p_i}{1-p_i} \\
            &\ge \sum_{i=0}^M a_i(p_{i+1}-p_i) - \sum_{i=1}^K \frac{b_i(p_{i+1}-p_i)p_i}{1-p_i} \\
            &\ge \sum_{i=0}^M a_i(p_{i+1}-p_i) - \sum_{i=0}^M \frac{b_i(p_{i+1}-p_i)p_i}{1-p_i} \\
            &= F(\bm{a},\bm{b}).
        \end{aligned}
    \end{equation}
    Moreover, since $a_K' = a_0 - \frac{\left(1-G(\bm{a},\bm{b})\right)a_0 b_0}{b_1'}$, we have
    \begin{equation}
        \begin{aligned}
            F(\bm{a}',\bm{b}') &= p_K a_0 +(1-p_K)a_K'-\sum_{i=0}^K \frac{b_i'(p_{i+1}-p_i)p_i}{1-p_i} \\
            &= a_0 - \frac{(1-p_K)a_0b_0\left(1-G(\bm{a},\bm{b})\right)}{b_1'}-b_1'P_K,
        \end{aligned}
    \end{equation}
    where $P_K = \sum_{i=0}^K \frac{(p_{i+1}-p_i)p_i}{1-p_i}$.
    By the AM--GM inequality, for any $x,y \ge 0$, $x+y\ge2\sqrt{xy}$, applying this with $x=\frac{(1-p_K)a_0b_0\left(1-G(\bm{a},\bm{b})\right)}{b_1'}$ and $y=b_1'P_K$, we obtain
    \begin{equation}
        \begin{aligned}
            F(\bm{a}',\bm{b}') &\le a_0 - 2\sqrt{P_K(1-p_K)a_0b_0\left(1-G(\bm{a},\bm{b})\right)}.
        \end{aligned}
    \end{equation}
    Equality holds if and only if $b_1'=\sqrt{\frac{(1-p_K)a_0b_0(1-G(\bm{a},\bm{b}))}{P_K}}$. Since the constraint $b_1'\le b_0$ must also be satisfied, this implies the condition $G(a,b)>1-\frac{b_0P_K}{(1-p_K)a_0}$. If this condition does not hold, the maximum of $F(\bm{a}',\bm{b}')$ is attained at $b_1' = b_0$, in which case
    \begin{align}
        F(\bm{a}',\bm{b}') \le a_0(1-p_K)G(\bm{a},\bm{b}) + a_0p_K - b_0P_K.
    \end{align}
    Therefore, we have
    \begin{equation}
        \begin{aligned}
            \mathrm{AUNBC}(\bm p;c) &= F(\bm{a},\bm{b}) \le F(\bm{a}',\bm{b}') \\
            &\le \left\{
            \begin{array}{cc}
                a_0(1-p_K)G(\bm{a},\bm{b}) +a_0p_K - b_0P_K, & \text{if }G(\bm{a},\bm{b}) \le 1- \frac{b_0P_K}{(1-p_K)a_0}; \\
                a_0 - 2\sqrt{P_K(1-p_K)a_0b_0(1-G(\bm{a},\bm{b})}, & \text{otherwise}
            \end{array}\right.\\
            &=A_K(G(\bm{a},\bm{b});a_0)\\
            &=A_K\left(\mathrm{AUROC}(\bm p;c);\frac{N^+}{N}\right) \le \max_{1\le k\le M}A_k\left(\mathrm{AUROC}(\bm p;c);\frac{N^+}{N}\right).
        \end{aligned}
    \end{equation}
\end{proof}
\subsection*{Proof of \Cref{cor:auroc_aunbc}}
\begin{proof} 
    The function $A_k(x;a_0)$ defined in (\ref{eq:Ak}) is continuous in $[0,1]$ and strictly monotonically increases with respect to $x$ for each $k \in \{1,2,\cdots,M\}$. Then we can define the inverse  function $B_k(y;a_0)$ of $A_k(x;a_0)$, as shown in (\ref{eq:Bk}). In addition, since $\mathrm{AUROC}(\bm p;c) \in [0,1]$, (\ref{eq:auroc}) holds true. Thus,
    \begin{equation}
        \begin{aligned}
            \mathrm{AUROC}(\bm p;c) \ge \min_{1\le k \le M} A_k^{-1}(\mathrm{AUNBC}(\bm p;c);a_0)=\min_{1\le k \le M} B_k(\mathrm{AUNBC}(\bm p;c);a_0).
        \end{aligned}
    \end{equation}
    Since $\mathrm{AUROC}(\bm p;c) \ge 0$, we finally have
    \begin{align}
        1 \ge \mathrm{AUROC}(\bm p;c) \ge \max\left\{\min_{1\le k \le M} B_k(\mathrm{AUNBC}(\bm p;c);a_0), 0\right\}.
    \end{align}
\end{proof}
\subsection*{Proof of Theorem~\ref{thm:improve_aunbc}}
\begin{proof} 
For any risk model $c:\mathcal{X}\to [0,1]$, we use $\mathrm{TP}_i(c)$ and $\mathrm{FP}_i(c)$, respectively, to denote the number of true positives and the number of false positives predicted by $c$ above threshold $p_i$; use $N_i(c)$ and $O_i(c)$, respectively, to denote the total number of samples and the number of true positives predicted by $c$ in the interval $\mathcal{S}_i$. Then 
\begin{align}
  &N_i(c) = \mathrm{TP}_i(c) - \mathrm{TP}_{i+1}(c) +\mathrm{FP}_i(c) -\mathrm{FP}_{i+1}(c),\\
  &O_i(c) = \mathrm{TP}_i(c) - \mathrm{TP}_{i+1}(c).
\end{align}
For the model $c'_k$, we have 
\begin{align}
    \mathrm{TP}_i(c_k') = \sum_{j=1}^N I(c_k'(\bm{x}_j) \ge p_i,y_j=1) = \left\{
        \begin{array}{cc}
        \mathrm{TP}_{k}(c),  & \text{if } i=k+1;\\
        \mathrm{TP}_i(c),    & \text{otherwise,}
        \end{array}
        \right.\\
    \mathrm{FP}_i(c_k') = \sum_{j=1}^N I(c_k'(\bm{x}_j) \ge p_i,y_j=0) = \left\{
        \begin{array}{cc}
        \mathrm{FP}_{k}(c),  & \text{if } i=k+1;\\
        \mathrm{FP}_i(c),    & \text{otherwise.}
        \end{array}
        \right.
\end{align}
Then 
\begin{equation}
    \begin{aligned}
    & p_{k+1}N_k(c) < O_k(c)\\
    \Leftrightarrow \quad & p_{k+1}\left(\mathrm{TP}_k(c)-\mathrm{TP}_{k+1}(c)+\mathrm{FP}_k(c)-\mathrm{FP}_{k+1}(c) \right)<\left(\mathrm{TP}_k(c)-\mathrm{TP}_{k+1}(c) \right)\\
    \Leftrightarrow \quad &  \mathrm{TP}_{k+1}(c) - \mathrm{FP}_{k+1}(c) \cdot \frac{p_{k+1}}{1-p_{k+1}}  < \mathrm{TP}_k(c)-\mathrm{FP}_k(c)\cdot \frac{p_{k+1}}{1-p_{k+1}}\\
    \Leftrightarrow \quad &\mathrm{TP}_{k+1}(c) - \mathrm{FP}_{k+1}(c) \cdot \frac{p_{k+1}}{1-p_{k+1}} < \mathrm{TP}_{k+1}(c_k') - \mathrm{FP}_{k+1}(c_k') \cdot \frac{p_{k+1}}{1-p_{k+1}}\\
    \Leftrightarrow \quad & \sum_{i=0}^M \mathrm{TP}_{i}(c) - \mathrm{FP}_{i}(c) \cdot \frac{p_{i}}{1-p_{i}} < \sum_{i=0}^M \mathrm{TP}_{i}(c_k') - \mathrm{FP}_{i}(c_k') \cdot \frac{p_{i}}{1-p_{i}}.
    \end{aligned}
\end{equation}
Similarly, for the model $c_k''$ 
\begin{align}
    \mathrm{TP}_i(c_k'') = \sum_{j=1}^N I(c_k''(\bm{x}_j) \ge p_i,y_j=1) = \left\{
        \begin{array}{cc}
        \mathrm{TP}_{k+1}(c),  & \text{if } i=k;\\
        \mathrm{TP}_i(c),    & \text{otherwise.}
        \end{array}
        \right.\\
    \mathrm{FP}_i(c_k'') = \sum_{j=1}^N I(c_k''(\bm{x}_j) \ge p_i,y_j=0) = \left\{
        \begin{array}{cc}
        \mathrm{FP}_{k+1}(c),  & \text{if } i=k;\\
        \mathrm{FP}_i(c),    & \text{otherwise.}
        \end{array}
        \right.
\end{align}
Then
\begin{equation}
    \begin{aligned}
    & p_{k}N_k(c) > O_k(c)\\
    \Leftrightarrow \quad & p_{k}\left(\mathrm{TP}_k(c)-\mathrm{TP}_{k+1}(c)+\mathrm{FP}_k(c)-\mathrm{FP}_{k+1}(c) \right) > \left(\mathrm{TP}_k(c)-\mathrm{TP}_{k+1}(c) \right)\\
    \Leftrightarrow \quad &  \mathrm{TP}_{k+1}(c) - \mathrm{FP}_{k+1}(c) \cdot \frac{p_{k}}{1-p_{k}}  > \mathrm{TP}_k(c)-\mathrm{FP}_k(c)\cdot \frac{p_{k}}{1-p_{k}}\\
    \Leftrightarrow \quad &\mathrm{TP}_{k}(c_k'') - \mathrm{FP}_{k}(c_k'') \cdot \frac{p_{k}}{1-p_{k}} > \mathrm{TP}_{k}(c) - \mathrm{FP}_{k}(c) \cdot \frac{p_{k}}{1-p_{kre}}\\
    \Leftrightarrow \quad & \sum_{i=0}^M \mathrm{TP}_{i}(c_k'') - \mathrm{FP}_{i}(c_k'') \cdot \frac{p_{i}}{1-p_{i}} > \sum_{i=0}^M \mathrm{TP}_{i}(c) - \mathrm{FP}_{i}(c) \cdot \frac{p_{i}}{1-p_{i}}.
    \end{aligned}
\end{equation}
\end{proof}
\subsection*{Proof of \Cref{prop:property of the inproved model}}
\begin{proof}
We prove the proposition by analyzing the two for-loops in \Cref{alg:improve_aunbc}. For the first loop, let $O_k^{1,s}$ and $N_k^{1,s}$ denote the values of $O_k$ and $N_k$ after completing $s$ iterations of the first loop, where $s=0,\ldots,M$, with $s=0$ denoting the state before the loop begins. After the iteration with $i=k$, the values of $O_k^{1,k+1}$ and $N_k^{1,k+1}$ either become zero, if the condition $O_k^{1,k}>p_{k+1}N_k^{1,k}$ holds, or remain unchanged otherwise. In both cases, $O_k^{1,k+1} \le p_{k+1}N_k^{1,k+1}$ holds after the iteration with $i=k$. Moreover, since the interval $\mathcal{S}_k$ is never revisited in the remaining iterations of the first loop, we have
    \begin{equation}
        O_k^{1,s} = O_k^{1,k+1},\quad N_k^{1,s} = N_k^{1,k+1},
    \end{equation}
    for all $s \ge k+1$, and hence
    \begin{equation}
        O_k^{1,s} \le p_{k+1}N_k^{1,s}.
    \end{equation}
\indent For the second loop, let $O_k^{2,t}$ and $N_k^{2,t}$ denote the values of $O_k$ and $N_k$ after completing $t$ iterations of the second loop, where $t=0,\ldots,M$, with $t=0$ denoting the state before the second loop begins. Then for each $k=0,1,\cdots,M$, we have $O_k^{1,M} = O_k^{2,0}$ and $N_k^{1,M} = N_k^{2,0}$. Similarly, after the iteration with $i=k$, the values of $O_k^{2,M-k+1}$ and $N_k^{2,M-k+1}$ either become zero, if the condition $O_k^{2,k} < p_{k}N_k^{2,k}$ holds, or remain unchanged otherwise. In both cases, $O_k^{2,M-k+1} \ge p_{k}N_k^{2,M-k+1}$ holds after the iteration with $i=k$. Moreover, since the interval $\mathcal{S}_k$ is never revisited in the remaining iterations of the second loop, we have
    \begin{equation}
        O_k^{2,t} = O_k^{2,M-k+1},\quad N_k^{2,t} = N_k^{2,M-k+1},
    \end{equation}
    for all $t \le M-k+1$, and hence
    \begin{equation}
        O_k^{2,t} \ge p_{k}N_k^{2,t}.
    \end{equation}
    Then, we prove by backward induction that the inequality $O_k^{2,t} \le p_{k+1}N_k^{2,t}$ is preserved throughout the second loop. Before the second loop begins, we have
    \begin{equation}
        O_k^{2,0} \le p_{k+1} N_k^{2,0}, \quad k=0,1,\cdots,M.
    \end{equation}
    Suppose that these inequalities hold before iteration with $i=k$, and let $\mathcal{S}_{k}$ be the interval processed in this iteration.

    If $O_k^{2,M-k} < p_{k}N_k^{2,M-k}$ holds, we have
    \begin{equation}
        0 = O_k^{2,M-k+1} \le p_{k+1}N_k^{2,M-k+1} =0,
    \end{equation}
    and
    \begin{equation}
    O_{k-1}^{2,M-k+1} = O_{k-1}^{2,M-k}+O_{k}^{2,M-k},\quad N_{k-1}^{2,M-k+1} = N_{k-1}^{2,M-k}+N_{k}^{2,M-k}
    \end{equation}
    Since the values of $O_{k-1}^{2,t}$ and $N_{k-1}^{2,t}$ do not change when $t \le M-k$, we have
    \begin{equation}
        O_{k-1}^{2,M-k} = O_{k-1}^{2,0} = O_{k-1}^{1,M}, \quad N_{k-1}^{2,M-k} = N_{k-1}^{2,0} = N_{k-1}^{1,M}.
    \end{equation}
     Therefore,
     \begin{equation}
         O_{k-1}^{2,M-k} \le p_k N_{k-1}^{2,M-k}.
     \end{equation}
     Combine with $O_k^{2,M-k} < p_{k}N_k^{2,M-k}$, we have
     \begin{equation}
         O_{k-1}^{2,M-k+1} \le p_k N_{k-1}^{2,M-k+1}.
     \end{equation}
     Otherwise, $O_k^{2,M-k} \ge p_{k}N_k^{2,M-k}$ holds, we have
     \begin{equation}
         O_k^{2,M-k+1} = O_k^{2,M-k} \le p_{k+1} N_k^{2,M-k} = p_{k+1} N_k^{2,M-k+1},
     \end{equation}
     and \begin{equation}
         O_{k-1}^{2,M-k+1} = O_{k-1}^{2,M-k} \le p_{k} N_{k-1}^{2,M-k} = p_{k} N_{k-1}^{2,M-k+1}.
     \end{equation}
     Therefore, after the iteration with $i=k$, the inequalities for both $\mathcal{S}_k$ and $\mathcal{S}_{k-1}$ are satisfied. Applying backward induction completes the proof. Hence the conclusion of \Cref{prop:property of the inproved model} holds.
\end{proof}
\subsection*{Proof of \Cref{cor:calibration}}
\begin{proof} 
Since for each $i=0,1,\cdots,M$, 
\begin{equation}
     p_i N_i^* \le O_i^* \le p_{i+1}N_i^*,
\end{equation}
we can choose $q_i\in [p_i,p_{i+1}]$, such that 
\begin{align}
    q_{i} N_i^*=O_i^*.
\end{align}
If $q_i\in \mathcal{S}_i$, $i = 0,1,\cdots,M$, then the conclusion (1) holds. Otherwise, $q_{k} =p_{k+1}$ for some $k \in \{ 0,1,\cdots,M-1\}$, then let $c_k':\mathcal{X}\to[0,1]$ satisfies
\begin{equation}
    \begin{aligned}
        c_k'(\bm{x})=\left\{
        \begin{array}{cc}
        p_{k+1} & \text{if }c^*(\bm {x}) \in \mathcal{S}_k;\\
        c^*(\bm {x})    & \text{otherwise.}
        \end{array}
        \right.
    \end{aligned}
\end{equation}
It is easy to verify that $c_k'$ has same value of AUNBC as $c^*$ and
\begin{align}
    \mathrm{TP}_k(c'_k) - \mathrm{TP}_{k+1}(c'_k) &= \mathrm{FP}_k(c'_k) -\mathrm{FP}_{k+1}(c'_k) = 0,\\
    q_{k}'  \left(\mathrm{TP}_k(c'_k) - \mathrm{TP}_{k+1}(c'_k) +\mathrm{FP}_k(c'_k) -\mathrm{FP}_{k+1}(c'_k)\right)  &= \mathrm{TP}_k(c'_k) - \mathrm{TP}_{k+1}(c'_k) \quad \forall\, q_k' \in [p_k,p_{k+1}),\\
    q_{i}  \left(\mathrm{TP}_i(c'_k) - \mathrm{TP}_{i+1}(c'_k) +\mathrm{FP}_i(c'_k) -\mathrm{FP}_{i+1}(c'_k)\right)  &= \mathrm{TP}_i(c'_k) - \mathrm{TP}_{i+1}(c'_k) \quad \forall\, i \not = k.
\end{align}
Then we can replace the $c^*$ with $c_k'$ and repeat above operation until all $q_i$ fall within the interval $\mathcal{S}_i$.

Therefore, we can always find a model that maximizes AUNBC and assign it a suitable output $q_i$ on the interval $\mathcal{S}_i$, $\forall \, i \in \{0,1,\cdots,M\}$, to achieve moderate calibration.
\end{proof}
\subsection*{Proof of \Cref{prop:calibration}}
\begin{proof}
    First, Supposing $\tilde c(\bm x_{j_1}), \tilde c(\bm x_{j_2})<1$, we can write $\tilde c(\bm x_{j})=c^*(\bm x_{j}) + \delta \cdot c(\bm x_{j})$ for $j\in\{j_1,j_2\}$. To establish the equivalence $c(\bm x_{j_1})< c(\bm x_{j_2}) \iff \tilde c(\bm x_{j_1})<\tilde c(\bm x_{j_2})$, we analyze the following two scenarios based on \Cref{alg:improve_aunbc}:
    \begin{enumerate}
        \item If $c(\bm x_{j_1})< c(\bm x_{j_2})$, then $c^*(\bm x_{j_1}) \le c^*(\bm x_{j_2})$ holds, which implies $\tilde c(\bm x_{j_1}) < \tilde c(\bm x_{j_2})$.
        \item If $c(\bm x_{j_1}) \ge c(\bm x_{j_2})$, then $c^*(\bm x_{j_1}) \ge c^*(\bm x_{j_2})$ holds, which implies $\tilde c(\bm x_{j_1}) \ge \tilde c(\bm x_{j_2})$.
    \end{enumerate}
    Summarizing both case completes the proof of the equivalence.\\
    \indent Then, we define the intervals $\mathcal{S}_i = [p_i,p_{i+1})$ for $i=0,1,\cdots,M-1$ and $\mathcal{S}_M=[p_M,1]$. For each $i=0,1,\cdots,M$, we use $N_i(c)$ and $O_i(c)$ to denote the number of samples and the number of positive samples whose predicted risk by $c$ lies in the interval $\mathcal{S}_i$, respectively. According to the \Cref{alg:moderate_calibration}, 
    \begin{equation}
        c^*(\bm x) \in \left \{0,q_0,q_1,\cdots,q_M, 1\right\},
    \end{equation}
    and $0 \le q_0 < p_1 \le q_1 < p_2\le\cdots < p_M \le q_M \le 1$. Then for each $i=0,1,\cdots,M$,
    \begin{equation}
    \begin{aligned}
        N_i(c^*) &= \sum_{j=1}^N I(c^*(\bm x_j) \in \mathcal{S}_i) = \sum_{j=1}^N I(c^*(\bm x_j) = q_i), \\
        O_i(c^*) &= \sum_{j=1}^N I(c^*(\bm x_j) \in \mathcal{S}_i, y_j=1) = \sum_{j=1}^N I(c^*(\bm x_j) = q_i,y_j=1) = q_iN_i(c^*). 
    \end{aligned}
    \end{equation}
    And for each $i=0,1,\cdots,M-1$, if $\delta < p_{i+1}-q_i$ holds, we have $p_i \le q_i + \delta \cdot c(\bm x)<p_{i+1}<1$. Suppose that $\delta < \min_{0\le i \le M-1} \{p_{i+1} - q_i\}$, we have
    \begin{equation}
    \begin{aligned}
        N_i(\tilde c) & = \sum_{j=1}^N I(\tilde c(\bm x_j) \in \mathcal{S}_i) \\
        & = \sum_{j=1}^N I(c^*(\bm x_j)+\delta\cdot c(\bm x_j) \in \mathcal{S}_i)\\
        &= \sum_{j=1}^N I(c^*(\bm x_j) \in \mathcal{S}_i) = N_i(c^*),
    \end{aligned}
    \end{equation}
    and similarly,
    \begin{equation}
        O_i(\tilde c) = O_i(c^*).
    \end{equation}
    As for $i=M$, we can get the same results for $N_M(\tilde c)$ and $O_M(\tilde c)$ if $q_M < 1$ and $\delta < 1-q_M$; otherwise, for all $1\le j \le N$ such that $c^*(\bm x_j) \in \mathcal{S}_M$, $c^*(\bm x_j)=1$, then we have $\tilde c(\bm x_j) = 1$. Therefore, $N_M(\tilde c) = N_M(c^*)$ and $O_M(\tilde c) = O_M(c^*)$ still hold. Then we can compute ECE value for $\tilde c$,
    \begin{equation}
    \begin{aligned}
    ECE(\tilde c) &= \sum_{\substack{i=0 \\ N_i(\tilde c) \not = 0}}^M \frac{N_i(\tilde c)}{N} \left | \frac{O_i(\tilde c)}{N_i(\tilde c)} - e_i(\tilde c) \right |  \\
    & = \sum_{\substack{i=0 \\ N_i(c^*) \not = 0}}^M \frac{N_i(c^*)}{N} \left | \frac{O_i(c^*)}{N_i(c^*)} - e_i(\tilde c) \right |\\
    & = \sum_{\substack{i=0 \\ N_i(c^*) \not = 0}}^M \frac{N_i(c^*)}{N} \left | q_i - \left(q_i + \frac{\sum_{j=1}^N \min \{\delta \cdot c(\bm x_j),1-q_i\} \cdot I(c^*(\bm x_j)\in \mathcal{S}_j)}{N_i(c^*)}  \right) \right | \\
    & \le \sum_{\substack{i=0 \\ N_i(c^*) \not = 0}}^M \frac{1}{N} \left |\sum_{j=1}^N \delta \right | = \delta.
    \end{aligned}
    \end{equation}
    Since $\mathrm{ECE}(c^*)= 0$, we have 
    \begin{equation}
        \left | \mathrm{ECE}(\tilde c) - \mathrm{ECE}(c^*) \right | \le \delta.\\
    \end{equation}
    \indent Finally, for each $i=0,1,\cdots,M$, we use $\mathrm{TP}_i(c)$ and $\mathrm{FP}_i(c)$ to denote the number of true positive and the number of false positive samples, respectively, whose predicted risk by $c$ are at least $p_i$. By the definition, we have 
    \begin{equation}
        \mathrm{TP}_i(c) = \sum_{k = i}^M O_k(c),\quad \mathrm{FP}_i(c) = \sum_{k = i}^M (N_k(c) - O_k(c)).
    \end{equation}
    Therefore, for each $i=0,1,\cdots,M$, 
    \begin{equation}
        \mathrm{TP}_i(\tilde c) = \mathrm{TP}_i(c^*), \quad \mathrm{FP}_i(\tilde c) = \mathrm{FP}_i(c^*).
    \end{equation}
    \begin{equation}
        \implies \mathrm{AUNBC}(\bm p; c^*) = \mathrm{AUNBC}(\bm p;\tilde c).
    \end{equation}
\end{proof}
\subsection*{Proof of \Cref{thm:learning_capacity}}
\begin{proof} 
For each $r \in \{1,2,\cdots,P\}$ and $i \in \{0,1,\cdots,M\}$, we can determine the integers $\lambda_r$ and $T_i$ by:

\begin{align}
\frac{\rho_r}{\|\bm \rho\|_\infty} \Lambda - \frac{1}{2} & < \lambda_r \leq \frac{\rho_r}{\|\bm \rho\|_\infty} \Lambda + \frac{1}{2}, \\
\frac{t_i}{\|\bm \rho\|_\infty} \Lambda - \frac{1}{2} & < T_i \leq \frac{t_i}{\|\bm \rho\|_\infty} \Lambda + \frac{1}{2}.
\end{align}

Due to:

\begin{align}
\lambda_r & \leq \frac{\rho_r}{\|\bm \rho\|_\infty} \Lambda + \frac{1}{2} < \Lambda + 1, \\
\lambda_r & > \frac{\rho_r}{\|\bm \rho\|_\infty} \Lambda - \frac{1}{2} > -\Lambda - 1,
\end{align}

we have $-\Lambda \leq \lambda_r \leq \Lambda$. Furthermore,

\begin{align}
\|\bm \rho\|_\infty \|X\|_\infty & = \|\bm \rho\|_\infty \max_{1 \leq j \leq N} \|\bm x_j\|_1 \geq \max_{1 \leq j \leq N} |\bm x_j \bm \rho| \geq |t_i|,\quad \forall i=0,1,\cdots,M \\
\Rightarrow T_i & > \frac{t_i}{\|\bm \rho\|_\infty} \Lambda - \frac{1}{2} \geq \frac{-|t_i|}{\|\bm \rho\|_\infty}\Lambda - \frac{1}{2} > -[\Lambda \|X\|_\infty] - 1, \\
T_i & \leq \frac{t_i}{\|\bm \rho\|_\infty} \Lambda + \frac{1}{2} \leq \frac{|t_i|}{\|\bm \rho\|_\infty}\Lambda + \frac{1}{2} < [\Lambda \|X\|_\infty] + 1,
\end{align}

we have $-T_{\text{max}} \leq T_i \leq T_{\text{max}}$.

Next, we will prove for any $i\in \{0,1,\cdots,M\}$, $J \subseteq \{1,2,\ldots,N\}$ and $J' \subseteq \{1,2,\ldots,N\}$ that

\begin{align}
\{j \in J: \bm x_j ^\top \bm \lambda < T_i\} & = \{j \in J: \bm x_j ^\top \bm \rho < t_i\}, \\
\{j \in J': \bm x_j ^\top \bm \lambda \geq T_i\} & = \{j \in J': \bm x_j ^\top \bm \rho \geq t_i\}.
\end{align}

It is equivalent to prove that for any $i = 0,1,\cdots,M$ and $j=1,2,\cdots,N$, the signs of $\frac{\bm x_j ^\top \bm \lambda - T_i}{\Lambda}$ and $\frac{\bm x_j ^\top \bm\rho - t_i}{\|\bm \rho\|_\infty}$ are always the same. Comparing the terms $\frac{\bm x_j ^\top \bm \lambda - T_i}{\Lambda}$ and $\frac{\bm x_j ^\top \bm \rho - t_i}{\|\bm \rho\|_\infty}$, we get

\begin{align}
\left| \frac{\bm x_j ^\top \bm \lambda - T_i}{\Lambda} - \frac{\bm x_j ^\top \bm \rho - t_i}{\|\bm \rho\|_\infty} \right| & \leq \left|\bm  x_j ^\top \left( \frac{\bm \lambda}{\Lambda} - \frac{\bm \rho}{\|\bm \rho\|_\infty} \right) \right| + \left| \frac{T_i}{\Lambda} - \frac{t_i}{\|\bm \rho\|_\infty} \right| \\
& \leq \frac{1}{2\Lambda} \|\bm  x_j \|_1 + \frac{1}{2\Lambda} \leq \frac{\|X\|_\infty + 1}{2\Lambda} < \gamma_{\text{min}} = \min_{i,j} \frac{|\bm x_j ^\top \bm \rho - t_i|}{\|\bm \rho\|_\infty}.
\end{align}
For the case, where $\bm x_j ^\top \bm \rho - t_i > 0$:
\begin{align}
\frac{\bm x_j ^\top \bm \rho - t_i}{\|\bm \rho\|_\infty} - \frac{\bm x_j ^\top \bm \lambda - T_i}{\Lambda} & \leq \left| \frac{\bm x_j ^\top \bm \lambda - T_i}{\Lambda} - \frac{\bm x_j ^\top \bm \rho - t_i}{\|\bm \rho\|_\infty} \right| < \min_{i,j} \frac{|\bm x_j ^\top \bm \rho - t_i|}{\|\bm \rho\|_\infty}, \\
\Rightarrow \frac{\bm x_j ^\top \bm \lambda - T_i}{\Lambda} & > \frac{\bm x_j ^\top \bm \rho - t_i}{\|\bm \rho\|_\infty} - \min_{i,j} \frac{|\bm x_j ^\top \bm \rho - t_i|}{\|\bm \rho\|_\infty} \geq 0.
\end{align}
For the case, where $\bm x_j ^\top \bm \rho - t_i < 0$:
\begin{align}
\frac{\bm x_j ^\top \bm \rho - t_i}{\|\bm \rho\|_\infty} - \frac{\bm x_j ^\top \bm \lambda - T_i}{\Lambda} & \geq -\left| \frac{\bm x_j ^\top \bm \lambda - T_i}{\Lambda} - \frac{\bm x_j ^\top \bm \rho - t_i}{\|\bm \rho\|_\infty} \right| > -\min_{i,j} \frac{|\bm x_j ^\top \bm \rho - t_i|}{\|\bm \rho\|_\infty}, \\
\Rightarrow \frac{\bm x_j ^\top \bm \lambda - T_i}{\Lambda} & < \frac{\bm x_j ^\top \bm \rho - t_i}{\|\bm \rho\|_\infty} + \min_{i,j} \frac{|\bm x_j ^\top \bm \rho - t_i|}{\|\bm \rho\|_\infty} \leq 0.
\end{align}
\end{proof}
\subsection*{Proof of \Cref{cor:learning_capacity}}
\begin{proof} 
If we only consider a subset $\mathcal{D}_{(k)}$ of dataset $\mathcal{D}$, where $\mathcal{D}_{(k)} = \left\{\left(\bm{x}_{j}, y_{j}\right)\right\}_{j\in \mathcal J_{(k)}}$, by Theorem (\ref{thm:learning_capacity}), we have

\begin{equation}
    \begin{aligned}
            &\sum_{i=0}^{M}\left(\omega_{i} \sum_{j\in \mathcal J_{(k)}} I\left( \bm{x}_{j}^\top \bm{\lambda}\ge T_{i},y_j=1\right)
            -\frac{\omega_{i}p_i}{1-p_{i}} \sum_{j\in \mathcal J_{(k)}} I\left(\bm{x}_{j}^\top \bm{\lambda} \geq T_{i}, y_{j}=0\right)\right) \\ 
            \ge &\sum_{i=0}^{M}\left(\omega_{i} \sum_{j\in \mathcal J_{(k)}} I\left( \bm{x}_{j}^\top \bm{\rho}\ge t_{i},y_j=1\right)
            -\frac{\omega_{i}p_i}{1-p_{i}} \sum_{j\in \mathcal J_{(k)}} I\left( \bm{x}_{j} ^\top \bm{\rho} \geq t_{i}, y_{j}=0\right)\right).
    \end{aligned}
\end{equation}
From the definition of $\mathcal{J}_{(k)}$, we know that $|\mathcal J_{(k)}|=N-(k-1)$. Thus, we have
\begin{equation}
\begin{aligned}
    &\sum_{i=0}^{M}\left(\omega_{i} \sum_{j \not \in \mathcal J_{(k)}} I\left( \bm{x}_{j}^\top \bm{\lambda}\ge T_{i},y_j=1\right)
    -\frac{\omega_{i}p_i}{1-p_{i}} \sum_{j \not \in \mathcal J_{(k)}} I\left(\bm{x}_{j}^\top \bm{\lambda} \geq T_{i}, y_{j}=0\right)\right) \\
    \ge &\sum_{i=0}^{M}\left(0
    -\frac{\omega_{i}p_i}{1-p_{i}}(k-1) \right) =-(k-1)\sum_{i=0}^M\frac{\omega_ip_i}{1-p_i}.
\end{aligned}    
\end{equation}
and
\begin{equation}
\begin{aligned}
    &\sum_{i=0}^{M}\left(\omega_{i} \sum_{j \not \in \mathcal J_{(k)}} I\left( \bm{x}_{j}^\top \bm{\rho}\ge t_{i},y_j=1\right)
    -\frac{\omega_{i}p_i}{1-p_{i}} \sum_{j \not \in \mathcal J_{(k)}} I\left( \bm{x}_{j}^\top  \bm{\rho} \geq t_{i}, y_{j}=0\right)\right) \\
    \le &\sum_{i=0}^{M}\left(\omega_{i} (k-1)-0\right) = k-1.
\end{aligned}    
\end{equation}
Finally, we can get
\begin{equation}
    \begin{aligned}
    &\sum_{i=0}^{M}\left(\omega_{i} \sum_{j=1}^{N} I\left( \bm{x}_{j}^\top \bm{\lambda}\ge T_{i},y_j=1\right)
    -\frac{\omega_{i}p_i}{1-p_{i}} \sum_{j=1}^{N} I\left(\bm{x}_{j}^\top \bm{\lambda} \geq T_{i}, y_{j}=0\right)\right) \\
    \ge &\sum_{i=0}^{M}\left(\omega_{i} \sum_{j=1}^{N} I\left( \bm{x}_{j}^\top \bm{\rho}\ge t_{i},y_j=1\right)
    -\frac{\omega_{i}p_i}{1-p_{i}} \sum_{j=1}^{N} I\left( \bm{x}_{j} ^\top \bm{\rho} \geq t_{i}, y_{j}=0\right)\right) - (k-1)\sum_{i=0}^M\frac{\omega_i}{1-p_i}.
    \end{aligned}
\end{equation}
\end{proof}
\subsection*{Proof of Theorem~\ref{thm:generalization}}
\begin{proof} 
$R_N(\bm \lambda,\bm T)$ and $R(\bm \lambda,\bm T))$ can be decomposed as:

\begin{align}
R_N(\bm \lambda,\bm T) & =  \sum_{i=0}^M \omega_i R_{1,N}(\bm \lambda, T_i) + \sum_{i=0}^M \omega_i R_{2,N}(\bm \lambda,T_i) , \\
R(\bm \lambda,\bm T) & = \sum_{i=0}^M \omega_i R_{1}(\bm \lambda,T_i) + \sum_{i=0}^M \omega_i R_2(\bm \lambda,T_i),
\end{align}

where
\begin{align}
R_{1,N}(\bm \lambda,T_i) & := -\frac{1}{N} \sum_{j=1}^N I(\bm x_j ^\top \bm \lambda -T_i \ge 0,y_j = 1), \\
R_{2,N}(\bm \lambda,T_i) & := \frac{p_i}{N(1-p_i)} \sum_{j=1}^N I(\bm x_j ^\top \bm \lambda -T_i \ge 0,y_j = 0), \\
R_{1}(\bm \lambda,T_i) & := -\mathbb{E}_{\mathcal {X},\mathcal{Y}} [I(\bm x_j ^\top \bm \lambda -T_i \ge 0,y_j = 1)], \\
R_{2}(\bm \lambda,T_i) & := \frac{p_i}{1-p_i} \mathbb{E}_{\mathcal X,\mathcal Y}[I(\bm x_j ^\top \bm \lambda -T_i \ge 0,y_j = 0)].
\end{align}

According to Hoeffding's inequality, for all $\epsilon > 0$,

\begin{align}
\mathbb{P}(R_{1}(\bm \lambda,T_i) - R_{1,N}(\bm \lambda,T_i) \geq \epsilon) & \leq \exp(-2N\epsilon^2),
\end{align}
and
\begin{equation}
    \begin{aligned}
        &\mathbb{P}(R_{2}(\bm \lambda,T_i) - R_{2,N}(\bm \lambda,T_i) \geq \epsilon) \\
        = & \mathbb{P} \left( \frac{p_i}{1 - p_i} \left( \mathbb{E}_{\mathcal X,\mathcal Y}[I(\bm x_j ^\top \bm \lambda -T_i \ge 0,y_j = 0)] - \frac{1}{N} \sum_{j=1}^N I(\bm x_j ^\top \bm \lambda -T_i \ge 0,y_j = 0) \right) \geq \epsilon \right) \\
        \leq &\exp(\frac{-2N(1 - p_i)^2\epsilon^2}{p_i^2}).
    \end{aligned}
\end{equation}

More generally,
\begin{equation}
\begin{aligned}
&\mathbb{P}(R_{1}(\bm \lambda,T_i) - R_{1,N}(\bm \lambda,T_i) \geq \epsilon, \forall \bm \lambda \in \mathcal{L},T_i \in \mathcal{B}_0) \\
 \leq &\sum_{\bm \lambda \in \mathcal{L},T_i \in \mathcal{B}_0} \mathbb{P}(R_{1}(\bm \lambda,T_i) - R_{1,N}(\bm \lambda,T_i) \geq \epsilon) \\
 \leq &|\mathcal{L}|\cdot |\mathcal{B}_0| \exp(-2N\epsilon^2).
\end{aligned}    
\end{equation}

Hence, for all $\delta > 0$ and $\bm \lambda \in \mathcal{L},T_i \in \mathcal{B}_0$ with probability at least $1 - \delta$, we obtain
\begin{align}
    R_{1}(\bm \lambda,T_i) \leq R_{1,N}(\bm \lambda,T_i) + \sqrt{\frac{\ln|\mathcal{L}|+\ln|\mathcal{B}_0| - \ln\delta}{2N}}.
\end{align}
With probability at least $1 - (M + 1)\delta$, we have
\begin{align}
    \sum_{i=0}^M \omega_i R_{1}(\bm \lambda,T_i) \leq \sum_{i=0}^M \omega_i R_{1,N}(\bm \lambda,T_i) + \sqrt{\frac{\ln|\mathcal{L}|+\ln|\mathcal{B}_0| - \ln\delta}{2N}}.
\end{align}
Similarly, with probability at least $1 - (M + 1)\delta$, we can write
\begin{align}
    \sum_{i=0}^M \omega_i R_{2}(\bm \lambda,T_i) \leq \sum_{i=0}^{M} \omega_i R_{2,N}(\bm \lambda,T_i) + \sum_{i=0}^{M} \frac{ \omega_i p_i}{1 - p_i} \sqrt{\frac{ \ln|\mathcal{L}|+\ln|\mathcal{B}_0| - \ln \delta}{2N}}
\end{align}
Thus, for small $\delta > 0$, with probability at least $1 - 2(M + 1)\delta$, we obtain
\begin{align} \label{ieq:generalization}
   R(\bm \lambda,\bm T) \leq R_{N}(\bm \lambda,\bm T) + \sum_{i=0}^{M} \frac{ \omega_i}{1 - p_i} \sqrt{\frac{\ln|\mathcal{L}|+\ln|\mathcal{B}_0| - \ln \delta}{2N}} .
\end{align}
Since the inequality (\ref{ieq:generalization}) is hold for any $\bm \lambda \in \mathcal{L}$ and $\bm T \in \mathcal{B}$, with probability at least $1 - 2(M + 1)\delta$, we have
\begin{align}
\sup_{\bm \lambda \in \mathcal{L},\bm T \in \mathcal{B}} \left\{ R(\bm \lambda,\bm T) - R_{N}(\bm \lambda,\bm T) \right\} \le \sum_{i=0}^{M} \frac{ \omega_i}{1 - p_i} \sqrt{\frac{\ln|\mathcal{L}|+\ln|{\mathcal{B}}_0| - \ln \delta}{2N}} ,
\end{align}
\end{proof}
\section{Intuitive Interpretation of AUNBC}\label{app:interpretation_aunbc}

AUNBC provides an aggregated measure of the decision utility of a prediction model over a range of threshold probabilities. Unlike conventional discrimination metrics (e.g., AUROC), which evaluate the ability to distinguish between positive and negative outcomes, AUNBC evaluates whether using a model leads to better decisions compared with alternative strategies, such as treating all individuals or treating none.

In decision curve analysis, the threshold probability $p_t$ represents the minimum risk at which a decision maker would consider taking an intervention. Different decision makers may choose different thresholds depending on their preferences. For example, a lower threshold indicates a preference to avoid missed cases, while a higher threshold indicates a preference to avoid unnecessary interventions.

For each threshold probability, the corresponding net benefit quantifies the utility of applying the prediction model. Therefore, the net benefit curve (decision curve) describes the usefulness of a model across different decision thresholds. AUNBC is obtained by calculating the area under the net benefit curve. Formally, for a risk model $c:\mathcal{X}\to[0,1]$, 
\begin{equation}
    \begin{aligned}
        \mathrm{AUNBC}(c) &= \int_0^1 \textrm{NB}(p;c)\,\mathrm{d}p\\&=\int_0^1 \left(\pi\cdot\mathrm{TPR}(p;c)-(1-\pi)\cdot\mathrm{FPR}(p;c)\cdot\frac{p}{1-p}\right)\, \mathrm{d}p,
    \end{aligned}
\end{equation}
where $\pi$ denotes the prevalence of positive cases, $\mathrm{NB}(p;c)$, $\mathrm{TPR}(p;c)$, and $\mathrm{FPR}(p;c)$ denote the net benefit, true positive rate, and false positive rate of model $c$ at threshold $p$, respectively. If $c(x)<\delta<1$ holds, the integral is Riemann integrable and therefore measures the overall decision utility over the interval $[0,1]$. A larger AUNBC(c) indicates that the model provides higher average net benefit over the interval $[0,1]$. In practice, we can use a discrete form as defined in \Cref{def:aunbc} to approximate the AUNBC, i.e.,
\begin{equation}
    \mathrm{AUNBC}(\bm p; c) = \sum_{i=0}^M (p_{i+1}-p_i)\mathrm{NB}(p_i;c).
\end{equation}
\begin{figure}
    \centering
    \includegraphics[width=0.8\linewidth]{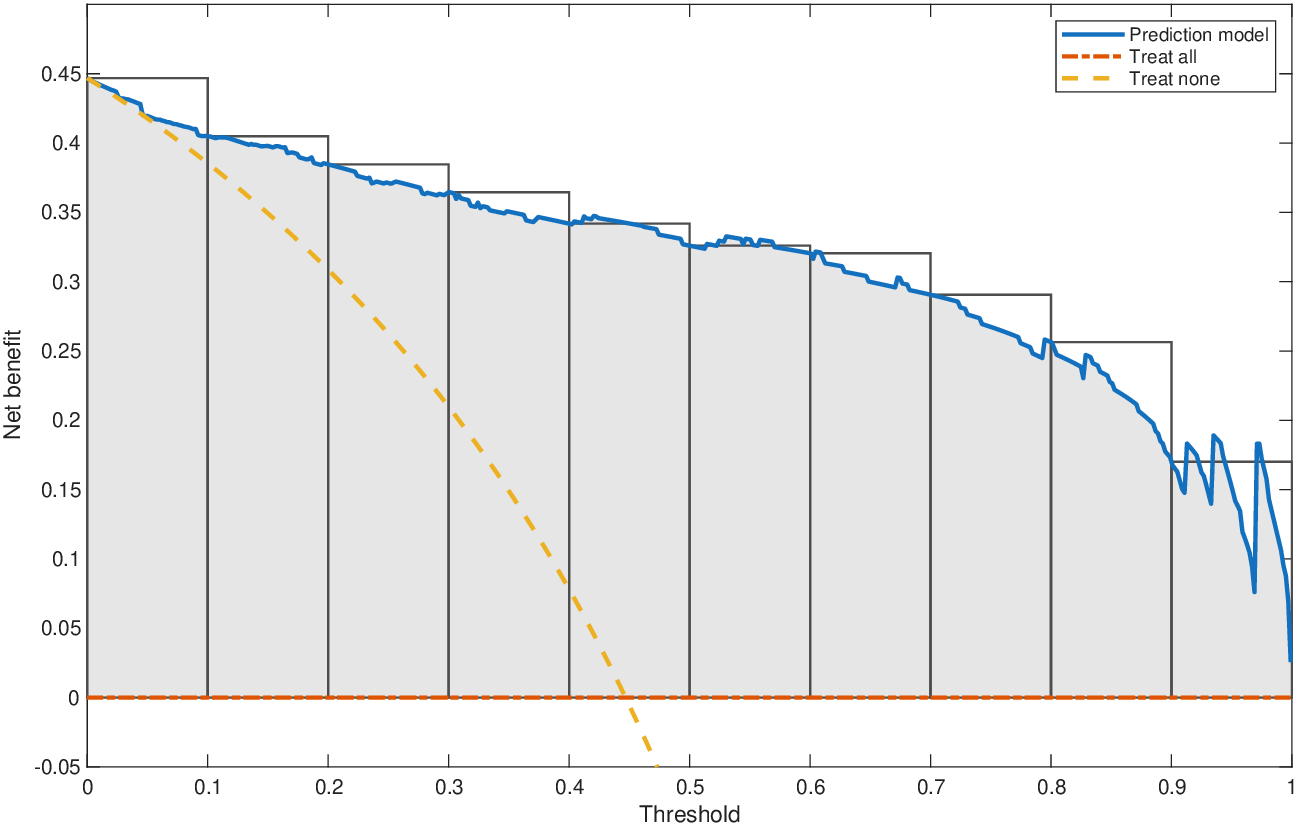}
    \caption{Illustration of continuous and discretized AUNBC. The shaded area corresponds to the continuous AUNBC of the prediction model. The total area of the rectangles represents the discretized approximation AUNBC.}
    \label{fig:AUNBC}
\end{figure}
\section{Ablation Studies and Post-Processing Comparisons} \label{app:ablation}
This appendix provides ablation studies to investigate the effects of the sparsity penalty $C_0$ and the post-processing procedure Algorithm \ref{alg:moderate_calibration} in RSS-DNB and RSS-DNB-SA. Specifically, we consider two variants of each model. First, we remove the sparsity penalty by setting $C_0$ to evaluate its effect on the sparsity and performance of the models. Second, we remove the post-processing procedure Algorithm \ref{alg:moderate_calibration} to assess its contribution to improving AUNBC and calibration. The results are summarized in \Cref{tab:ablation}.

As shown in \Cref{tab:ablation}, removing the sparsity penalty ($C_0=0$) increases the number of selected variables, confirming its role in controlling model sparsity. Meanwhile, the predictive performance remains relatively stable, indicating that the regularization term mainly affects model complexity rather than the optimization objective.

We further examine the effect of \Cref{alg:moderate_calibration} on other probabilistic baselines by comparing their performance before and after applying the procedure. Since the decision tree is already moderately calibrated on the training set and remains unchanged after applying \Cref{alg:moderate_calibration}, we omit its post-processing results. The results for the remaining models are summarized in \Cref{tab:post-processing}.

As shown in \Cref{tab:post-processing}, on the training set, each baseline model achieves zero ECE after applying \Cref{alg:moderate_calibration}, while AUROC remains unchanged and AUNBC is preserved or improved. These results are consistent with the conclusion in \Cref{subsec:calibration&utility}. On the test set, ECE decreases in most cases, providing empirical evidence of improved calibration on unseen data.
\begin{landscape}
\footnotesize
\setlength{\tabcolsep}{4pt}
\renewcommand{\arraystretch}{1.1}
\begin{longtable}{crcccccc}
\multicolumn{8}{l}{\textbf{Table \thetable\ :}  Ablation studies of RSS-DNB and RSS-DNB-SA.}\label{tab:ablation} \\
\toprule
Dataset & Metric & RSS-DNB & RSS-DNB ($C_0=0$) & RSS-DNB w/o Alg. 2 & RSS-DNB-SA & RSS-DNB-SA ($C_0=0$) & RSS-DNB-SA w/o Alg. 2  \\
\midrule
\endfirsthead

\multicolumn{8}{l}{\textbf{Table \thetable\ :}  continued} \\
\toprule
Dataset & Metric & RSS-DNB & RSS-DNB ($C_0=0$) & RSS-DNB w/o Alg. 2 & RSS-DNB-SA & RSS-DNB-SA ($C_0=0$) & RSS-DNB-SA w/o Alg. 2  \\
\midrule
\endhead

\bottomrule
\multicolumn{8}{l}{\footnotesize Notes: All values are reported as mean $\pm$ standard deviation over 10-fold cross-validation. Size is reported as mean (minimum--maximum) across the 10 folds.} \\
\multicolumn{8}{l}{\footnotesize The best and second-best results are highlighted in bold and underline, respectively.} \\
\endlastfoot
adult&Train AUROC&\underline{0.881 $\pm$ 0.002} & 0.881 $\pm$ 0.002 & \textbf{0.881 $\pm$ 0.002} & 0.872 $\pm$ 0.003 & 0.880 $\pm$ 0.001 & 0.872 $\pm$ 0.003\\
&Test AUROC&\underline{0.880 $\pm$ 0.007} & 0.879 $\pm$ 0.007 & \textbf{0.880 $\pm$ 0.007} & 0.872 $\pm$ 0.008 & 0.878 $\pm$ 0.006 & 0.872 $\pm$ 0.008\\
&Train ECE&\textbf{0.000 $\pm$ 0.000} & \textbf{0.000 $\pm$ 0.000} & 0.029 $\pm$ 0.005 & \textbf{0.000 $\pm$ 0.000} & \textbf{0.000 $\pm$ 0.000} & 0.038 $\pm$ 0.004\\
&Test ECE&\textbf{0.013 $\pm$ 0.003} & \textbf{0.013 $\pm$ 0.003} & 0.033 $\pm$ 0.005 & 0.014 $\pm$ 0.004 & 0.015 $\pm$ 0.004 & 0.038 $\pm$ 0.007\\
&Train AUNBC&\underline{0.102 $\pm$ 0.000} & \underline{0.102 $\pm$ 0.000} & 0.102 $\pm$ 0.000 & 0.101 $\pm$ 0.001 & \textbf{0.104 $\pm$ 0.000} & 0.101 $\pm$ 0.001\\
&Test AUNBC&\underline{0.102 $\pm$ 0.003} & \underline{0.102 $\pm$ 0.003} & 0.102 $\pm$ 0.003 & 0.100 $\pm$ 0.003 & \textbf{0.103 $\pm$ 0.003} & 0.100 $\pm$ 0.003\\
&Size&\underline{22.2 (19-24)} & 31.4 (26-33) & \underline{22.2 (19-24)} & \textbf{18.6 (15-22)} & 25.4 (22-28) & \textbf{18.6 (15-22)}\\
\midrule
bankruptcy&Train AUROC&\textbf{1.000 $\pm$ 0.000} & \textbf{1.000 $\pm$ 0.000} & \textbf{1.000 $\pm$ 0.000} & \underline{1.000 $\pm$ 0.000} & \textbf{1.000 $\pm$ 0.000} & \underline{1.000 $\pm$ 0.000}\\
&Test AUROC&\textbf{0.997 $\pm$ 0.011} & 0.982 $\pm$ 0.033 & \textbf{0.997 $\pm$ 0.011} & \textbf{0.997 $\pm$ 0.011} & \underline{0.987 $\pm$ 0.032} & \textbf{0.997 $\pm$ 0.011}\\
&Train ECE&\textbf{0.000 $\pm$ 0.000} & \textbf{0.000 $\pm$ 0.000} & 0.057 $\pm$ 0.002 & \textbf{0.000 $\pm$ 0.000} & \textbf{0.000 $\pm$ 0.000} & \underline{0.057 $\pm$ 0.002}\\
&Test ECE&\textbf{0.004 $\pm$ 0.013} & 0.008 $\pm$ 0.017 & 0.061 $\pm$ 0.019 & \underline{0.006 $\pm$ 0.013} & \textbf{0.004 $\pm$ 0.013} & 0.062 $\pm$ 0.019\\
&Train AUNBC&\textbf{0.572 $\pm$ 0.016} & \textbf{0.572 $\pm$ 0.016} & \textbf{0.572 $\pm$ 0.016} & \underline{0.569 $\pm$ 0.016} & \textbf{0.572 $\pm$ 0.016} & \underline{0.569 $\pm$ 0.016}\\
&Test AUNBC&\textbf{0.568 $\pm$ 0.147} & 0.557 $\pm$ 0.139 & \textbf{0.568 $\pm$ 0.147} & \underline{0.568 $\pm$ 0.147} & 0.561 $\pm$ 0.135 & \underline{0.568 $\pm$ 0.147}\\
&Size&\underline{2.9 (2-3)} & 5.7 (5-6) & \underline{2.9 (2-3)} & \textbf{2.1 (2-3)} & 3.2 (3-5) & \textbf{2.1 (2-3)}\\
\midrule
breastcancer&Train AUROC&0.992 $\pm$ 0.002 & \underline{0.994 $\pm$ 0.001} & 0.992 $\pm$ 0.002 & 0.990 $\pm$ 0.002 & \textbf{0.994 $\pm$ 0.001} & 0.990 $\pm$ 0.002\\
&Test AUROC&\textbf{0.985 $\pm$ 0.015} & 0.980 $\pm$ 0.020 & \textbf{0.985 $\pm$ 0.015} & 0.972 $\pm$ 0.030 & \underline{0.980 $\pm$ 0.022} & 0.972 $\pm$ 0.030\\
&Train ECE&\textbf{0.000 $\pm$ 0.000} & \textbf{0.000 $\pm$ 0.000} & 0.033 $\pm$ 0.002 & \textbf{0.000 $\pm$ 0.000} & \textbf{0.000 $\pm$ 0.000} & 0.030 $\pm$ 0.003\\
&Test ECE&0.022 $\pm$ 0.010 & \underline{0.021 $\pm$ 0.014} & 0.041 $\pm$ 0.019 & 0.022 $\pm$ 0.013 & \textbf{0.019 $\pm$ 0.017} & 0.043 $\pm$ 0.010\\
&Train AUNBC&0.325 $\pm$ 0.006 & \textbf{0.325 $\pm$ 0.006} & 0.325 $\pm$ 0.006 & 0.320 $\pm$ 0.007 & \underline{0.325 $\pm$ 0.006} & 0.320 $\pm$ 0.007\\
&Test AUNBC&\underline{0.305 $\pm$ 0.057} & \textbf{0.309 $\pm$ 0.058} & \underline{0.305 $\pm$ 0.057} & 0.298 $\pm$ 0.066 & 0.304 $\pm$ 0.061 & 0.298 $\pm$ 0.066\\
&Size&\underline{6.5 (5-7)} & 8.6 (8-9) & \underline{6.5 (5-7)} & \textbf{3.5 (3-5)} & 8.8 (8-9) & \textbf{3.5 (3-5)}\\
\midrule
haberman&Train AUROC&0.733 $\pm$ 0.015 & \underline{0.739 $\pm$ 0.008} & 0.733 $\pm$ 0.015 & 0.728 $\pm$ 0.011 & \textbf{0.740 $\pm$ 0.010} & 0.728 $\pm$ 0.011\\
&Test AUROC&0.662 $\pm$ 0.119 & \underline{0.676 $\pm$ 0.157} & 0.662 $\pm$ 0.119 & 0.667 $\pm$ 0.138 & \textbf{0.686 $\pm$ 0.131} & 0.667 $\pm$ 0.139\\
&Train ECE&\textbf{0.000 $\pm$ 0.000} & \textbf{0.000 $\pm$ 0.000} & 0.056 $\pm$ 0.014 & \textbf{0.000 $\pm$ 0.000} & \textbf{0.000 $\pm$ 0.000} & 0.049 $\pm$ 0.010\\
&Test ECE&\textbf{0.118 $\pm$ 0.061} & 0.124 $\pm$ 0.053 & 0.123 $\pm$ 0.056 & 0.123 $\pm$ 0.049 & \underline{0.121 $\pm$ 0.045} & 0.127 $\pm$ 0.058\\
&Train AUNBC&0.476 $\pm$ 0.012 & \underline{0.480 $\pm$ 0.012} & 0.476 $\pm$ 0.012 & 0.474 $\pm$ 0.011 & \textbf{0.481 $\pm$ 0.012} & 0.473 $\pm$ 0.012\\
&Test AUNBC&0.440 $\pm$ 0.107 & 0.437 $\pm$ 0.127 & 0.440 $\pm$ 0.107 & \underline{0.442 $\pm$ 0.111} & \textbf{0.448 $\pm$ 0.113} & 0.434 $\pm$ 0.118\\
&Size&\underline{2.4 (2-3)} & 3.0 (3-3) & \underline{2.4 (2-3)} & \textbf{1.9 (1-2)} & 3.0 (3-3) & \textbf{1.9 (1-2)}\\
\midrule
heartdisease&Train AUROC&\underline{0.926 $\pm$ 0.008} & \textbf{0.934 $\pm$ 0.010} & \underline{0.926 $\pm$ 0.008} & 0.916 $\pm$ 0.010 & 0.925 $\pm$ 0.014 & 0.916 $\pm$ 0.010\\
&Test AUROC&0.819 $\pm$ 0.086 & 0.837 $\pm$ 0.074 & 0.819 $\pm$ 0.086 & \underline{0.841 $\pm$ 0.077} & \textbf{0.850 $\pm$ 0.066} & \underline{0.841 $\pm$ 0.077}\\
&Train ECE&\textbf{0.000 $\pm$ 0.000} & \textbf{0.000 $\pm$ 0.000} & 0.058 $\pm$ 0.007 & \textbf{0.000 $\pm$ 0.000} & \textbf{0.000 $\pm$ 0.000} & 0.059 $\pm$ 0.012\\
&Test ECE&0.079 $\pm$ 0.033 & \underline{0.071 $\pm$ 0.051} & 0.138 $\pm$ 0.066 & 0.079 $\pm$ 0.052 & \textbf{0.066 $\pm$ 0.030} & 0.118 $\pm$ 0.066\\
&Train AUNBC&\underline{0.359 $\pm$ 0.011} & \textbf{0.366 $\pm$ 0.012} & \underline{0.359 $\pm$ 0.011} & 0.352 $\pm$ 0.008 & 0.357 $\pm$ 0.008 & 0.352 $\pm$ 0.008\\
&Test AUNBC&0.206 $\pm$ 0.145 & 0.228 $\pm$ 0.135 & 0.206 $\pm$ 0.145 & \underline{0.248 $\pm$ 0.128} & \textbf{0.263 $\pm$ 0.120} & \underline{0.248 $\pm$ 0.128}\\
&Size&\textbf{16.3 (12-22)} & 30.9 (29-32) & \textbf{16.3 (12-22)} & \underline{23.5 (20-26)} & 30.2 (28-31) & \underline{23.5 (20-26)}\\
\midrule
mammo&Train AUROC&0.859 $\pm$ 0.003 & \textbf{0.862 $\pm$ 0.004} & 0.859 $\pm$ 0.003 & 0.852 $\pm$ 0.006 & \underline{0.860 $\pm$ 0.004} & 0.852 $\pm$ 0.006\\
&Test AUROC&0.845 $\pm$ 0.031 & 0.845 $\pm$ 0.034 & 0.845 $\pm$ 0.031 & \underline{0.847 $\pm$ 0.033} & 0.845 $\pm$ 0.030 & \textbf{0.848 $\pm$ 0.032}\\
&Train ECE&\textbf{0.000 $\pm$ 0.000} & \textbf{0.000 $\pm$ 0.000} & 0.053 $\pm$ 0.012 & \textbf{0.000 $\pm$ 0.000} & \textbf{0.000 $\pm$ 0.000} & 0.058 $\pm$ 0.010\\
&Test ECE&0.078 $\pm$ 0.019 & \textbf{0.070 $\pm$ 0.018} & 0.093 $\pm$ 0.028 & 0.078 $\pm$ 0.020 & \underline{0.075 $\pm$ 0.023} & 0.100 $\pm$ 0.033\\
&Train AUNBC&0.262 $\pm$ 0.006 & \textbf{0.264 $\pm$ 0.006} & 0.262 $\pm$ 0.006 & 0.257 $\pm$ 0.006 & \underline{0.263 $\pm$ 0.006} & 0.256 $\pm$ 0.006\\
&Test AUNBC&\textbf{0.252 $\pm$ 0.052} & \underline{0.252 $\pm$ 0.054} & \textbf{0.252 $\pm$ 0.052} & 0.251 $\pm$ 0.052 & 0.250 $\pm$ 0.050 & 0.248 $\pm$ 0.054\\
&Size&\underline{6.5 (6-7)} & 13.7 (13-14) & \underline{6.5 (6-7)} & \textbf{4.6 (4-6)} & 13.3 (12-14) & \textbf{4.6 (4-6)}\\
\midrule
mushroom&Train AUROC&\textbf{1.000 $\pm$ 0.000} & \textbf{1.000 $\pm$ 0.000} & \textbf{1.000 $\pm$ 0.000} & \textbf{1.000 $\pm$ 0.000} & \textbf{1.000 $\pm$ 0.000} & \textbf{1.000 $\pm$ 0.000}\\
&Test AUROC&\textbf{1.000 $\pm$ 0.000} & \textbf{1.000 $\pm$ 0.000} & \textbf{1.000 $\pm$ 0.000} & \textbf{1.000 $\pm$ 0.000} & \textbf{1.000 $\pm$ 0.000} & \textbf{1.000 $\pm$ 0.000}\\
&Train ECE&\textbf{0.000 $\pm$ 0.000} & \textbf{0.000 $\pm$ 0.000} & \underline{0.048 $\pm$ 0.000} & \textbf{0.000 $\pm$ 0.000} & \textbf{0.000 $\pm$ 0.000} & \underline{0.048 $\pm$ 0.000}\\
&Test ECE&\textbf{0.000 $\pm$ 0.000} & \textbf{0.000 $\pm$ 0.000} & \underline{0.048 $\pm$ 0.002} & \textbf{0.000 $\pm$ 0.000} & \textbf{0.000 $\pm$ 0.000} & \underline{0.048 $\pm$ 0.002}\\
&Train AUNBC&\textbf{0.482 $\pm$ 0.002} & \textbf{0.482 $\pm$ 0.002} & \textbf{0.482 $\pm$ 0.002} & \textbf{0.482 $\pm$ 0.002} & \textbf{0.482 $\pm$ 0.002} & \textbf{0.482 $\pm$ 0.002}\\
&Test AUNBC&\textbf{0.482 $\pm$ 0.017} & \textbf{0.482 $\pm$ 0.017} & \textbf{0.482 $\pm$ 0.017} & \textbf{0.482 $\pm$ 0.017} & \textbf{0.482 $\pm$ 0.017} & \textbf{0.482 $\pm$ 0.017}\\
&Size&\underline{22.3 (19-43)} & 106.1 (53-113) & \underline{22.3 (19-43)} & 35.3 (30-43) & \textbf{19.8 (18-21)} & 35.3 (30-43)\\
\midrule
spambase&Train AUROC&0.954 $\pm$ 0.010 & \underline{0.961 $\pm$ 0.005} & 0.958 $\pm$ 0.009 & 0.917 $\pm$ 0.037 & \textbf{0.965 $\pm$ 0.002} & 0.917 $\pm$ 0.037\\
&Test AUROC&0.951 $\pm$ 0.017 & 0.953 $\pm$ 0.018 & \underline{0.956 $\pm$ 0.012} & 0.919 $\pm$ 0.038 & \textbf{0.959 $\pm$ 0.011} & 0.919 $\pm$ 0.038\\
&Train ECE&\textbf{0.000 $\pm$ 0.000} & \textbf{0.000 $\pm$ 0.000} & 0.083 $\pm$ 0.015 & \textbf{0.000 $\pm$ 0.000} & \textbf{0.000 $\pm$ 0.000} & 0.038 $\pm$ 0.008\\
&Test ECE&\textbf{0.011 $\pm$ 0.008} & \underline{0.015 $\pm$ 0.005} & 0.079 $\pm$ 0.017 & 0.028 $\pm$ 0.009 & 0.020 $\pm$ 0.006 & 0.050 $\pm$ 0.015\\
&Train AUNBC&0.298 $\pm$ 0.008 & \underline{0.312 $\pm$ 0.005} & 0.284 $\pm$ 0.010 & 0.249 $\pm$ 0.037 & \textbf{0.324 $\pm$ 0.003} & 0.249 $\pm$ 0.037\\
&Test AUNBC&0.295 $\pm$ 0.020 & \underline{0.300 $\pm$ 0.023} & 0.282 $\pm$ 0.017 & 0.249 $\pm$ 0.044 & \textbf{0.312 $\pm$ 0.019} & 0.249 $\pm$ 0.044\\
&Size&\underline{33.2 (28-37)} & 38.4 (33-44) & \underline{33.2 (28-37)} & \textbf{32.1 (23-39)} & 41.5 (37-48) & \textbf{32.1 (23-39)}\\

\end{longtable}
\end{landscape}

\begin{landscape}
\footnotesize
\setlength{\tabcolsep}{4pt}
\renewcommand{\arraystretch}{1.1}
\begin{longtable}{crcccccccc}
\multicolumn{10}{l}{\textbf{Table \thetable\ :}  Performance of baseline models before and after applying \Cref{alg:moderate_calibration}.} \label{tab:post-processing} \\
\toprule
Dataset & Metric & Logistic & \makecell{Logistic\\+ Alg. 2} & LASSO & \makecell{LASSO\\+ Alg. 2} & RISKSLIM  & \makecell{RISKSLIM\\+ Alg. 2} & XGBoost & \makecell{XGBoost\\+ Alg. 2}\\
\midrule
\endfirsthead

\multicolumn{10}{l}{\textbf{Table \thetable\ :}  continued} \\
\toprule
Dataset & Metric & Logistic & \makecell{Logistic\\+ Alg. 2} & LASSO & \makecell{LASSO\\+ Alg. 2} & RISKSLIM  & \makecell{RISKSLIM\\+ Alg. 2} & XGBoost & \makecell{XGBoost\\+ Alg. 2}\\
\midrule
\endhead

\bottomrule
\multicolumn{10}{l}{\footnotesize Notes: All values are reported as mean $\pm$ standard deviation over 10-fold cross-validation. Size is reported as mean (minimum--maximum) across the 10 folds. The best } \\
\multicolumn{10}{l}{\footnotesize and second-best results are highlighted in bold and underline, respectively.}\\
\endlastfoot
adult&Train AUROC&\underline{0.891 $\pm$ 0.001} & \underline{0.891 $\pm$ 0.001} & 0.890 $\pm$ 0.001 & 0.890 $\pm$ 0.001 & 0.882 $\pm$ 0.003 & 0.882 $\pm$ 0.003 & \textbf{0.900 $\pm$ 0.001} & \textbf{0.900 $\pm$ 0.001}\\
&Test AUROC&\underline{0.891 $\pm$ 0.006} & \underline{0.891 $\pm$ 0.006} & 0.889 $\pm$ 0.006 & 0.889 $\pm$ 0.006 & 0.881 $\pm$ 0.007 & 0.881 $\pm$ 0.007 & \textbf{0.894 $\pm$ 0.005} & \textbf{0.894 $\pm$ 0.005}\\
&Train ECE&0.010 $\pm$ 0.001 & \textbf{0.000 $\pm$ 0.000} & 0.012 $\pm$ 0.002 & \textbf{0.000 $\pm$ 0.000} & 0.030 $\pm$ 0.007 & \textbf{0.000 $\pm$ 0.000} & 0.033 $\pm$ 0.001 & \textbf{0.000 $\pm$ 0.000}\\
&Test ECE&0.016 $\pm$ 0.004 & 0.012 $\pm$ 0.004 & 0.017 $\pm$ 0.004 & \textbf{0.012 $\pm$ 0.003} & 0.029 $\pm$ 0.007 & \underline{0.012 $\pm$ 0.004} & 0.036 $\pm$ 0.005 & \textbf{0.012 $\pm$ 0.003}\\
&Train AUNBC&0.104 $\pm$ 0.000 & 0.104 $\pm$ 0.000 & 0.103 $\pm$ 0.000 & 0.103 $\pm$ 0.000 & 0.097 $\pm$ 0.001 & 0.100 $\pm$ 0.001 & \underline{0.107 $\pm$ 0.000} & \textbf{0.107 $\pm$ 0.000}\\
&Test AUNBC&0.103 $\pm$ 0.003 & 0.103 $\pm$ 0.003 & 0.103 $\pm$ 0.003 & 0.103 $\pm$ 0.003 & 0.096 $\pm$ 0.004 & 0.099 $\pm$ 0.003 & \underline{0.103 $\pm$ 0.003} & \textbf{0.104 $\pm$ 0.004}\\
\midrule
bankruptcy&Train AUROC&\textbf{1.000 $\pm$ 0.000} & \textbf{1.000 $\pm$ 0.000} & \textbf{1.000 $\pm$ 0.000} & \textbf{1.000 $\pm$ 0.000} & \textbf{1.000 $\pm$ 0.000} & \textbf{1.000 $\pm$ 0.000} & \textbf{1.000 $\pm$ 0.000} & \textbf{1.000 $\pm$ 0.000}\\
&Test AUROC&0.990 $\pm$ 0.032 & 0.990 $\pm$ 0.032 & \underline{0.998 $\pm$ 0.006} & \underline{0.998 $\pm$ 0.006} & 0.996 $\pm$ 0.011 & 0.996 $\pm$ 0.011 & \textbf{1.000 $\pm$ 0.000} & \textbf{1.000 $\pm$ 0.000}\\
&Train ECE&\textbf{0.000 $\pm$ 0.000} & \textbf{0.000 $\pm$ 0.000} & 0.015 $\pm$ 0.007 & \textbf{0.000 $\pm$ 0.000} & 0.001 $\pm$ 0.001 & \textbf{0.000 $\pm$ 0.000} & 0.024 $\pm$ 0.001 & \textbf{0.000 $\pm$ 0.000}\\
&Test ECE&0.004 $\pm$ 0.013 & 0.004 $\pm$ 0.013 & 0.016 $\pm$ 0.012 & \textbf{0.000 $\pm$ 0.000} & 0.007 $\pm$ 0.015 & 0.004 $\pm$ 0.013 & 0.027 $\pm$ 0.015 & \textbf{0.000 $\pm$ 0.000}\\
&Train AUNBC&\textbf{0.572 $\pm$ 0.016} & \textbf{0.572 $\pm$ 0.016} & 0.567 $\pm$ 0.014 & \textbf{0.572 $\pm$ 0.016} & \underline{0.572 $\pm$ 0.016} & \textbf{0.572 $\pm$ 0.016} & 0.570 $\pm$ 0.016 & \textbf{0.572 $\pm$ 0.016}\\
&Test AUNBC&0.561 $\pm$ 0.135 & 0.561 $\pm$ 0.135 & 0.559 $\pm$ 0.137 & 0.560 $\pm$ 0.135 & 0.558 $\pm$ 0.139 & 0.557 $\pm$ 0.139 & \textbf{0.568 $\pm$ 0.150} & \underline{0.568 $\pm$ 0.152}\\
\midrule
breastcancer&Train AUROC&0.996 $\pm$ 0.001 & 0.996 $\pm$ 0.001 & 0.996 $\pm$ 0.001 & 0.996 $\pm$ 0.001 & 0.993 $\pm$ 0.002 & 0.993 $\pm$ 0.002 & \textbf{1.000 $\pm$ 0.000} & \textbf{1.000 $\pm$ 0.000}\\
&Test AUROC&\textbf{0.995 $\pm$ 0.005} & \textbf{0.995 $\pm$ 0.005} & \underline{0.995 $\pm$ 0.006} & \underline{0.995 $\pm$ 0.006} & 0.990 $\pm$ 0.009 & 0.990 $\pm$ 0.009 & 0.992 $\pm$ 0.010 & 0.992 $\pm$ 0.010\\
&Train ECE&0.016 $\pm$ 0.002 & \textbf{0.000 $\pm$ 0.000} & 0.033 $\pm$ 0.007 & \textbf{0.000 $\pm$ 0.000} & 0.013 $\pm$ 0.004 & \textbf{0.000 $\pm$ 0.000} & 0.034 $\pm$ 0.004 & \textbf{0.000 $\pm$ 0.000}\\
&Test ECE&0.032 $\pm$ 0.010 & \textbf{0.017 $\pm$ 0.007} & 0.048 $\pm$ 0.015 & 0.022 $\pm$ 0.016 & 0.031 $\pm$ 0.013 & 0.029 $\pm$ 0.012 & 0.041 $\pm$ 0.011 & \underline{0.017 $\pm$ 0.014}\\
&Train AUNBC&0.318 $\pm$ 0.006 & 0.320 $\pm$ 0.006 & 0.314 $\pm$ 0.006 & 0.317 $\pm$ 0.006 & 0.306 $\pm$ 0.007 & 0.309 $\pm$ 0.007 & \underline{0.342 $\pm$ 0.006} & \textbf{0.348 $\pm$ 0.006}\\
&Test AUNBC&\textbf{0.310 $\pm$ 0.054} & \underline{0.307 $\pm$ 0.055} & 0.307 $\pm$ 0.053 & 0.306 $\pm$ 0.057 & 0.297 $\pm$ 0.061 & 0.294 $\pm$ 0.063 & 0.306 $\pm$ 0.057 & 0.304 $\pm$ 0.058\\
\midrule
haberman&Train AUROC&\underline{0.701 $\pm$ 0.009} & \underline{0.701 $\pm$ 0.009} & 0.584 $\pm$ 0.108 & 0.584 $\pm$ 0.108 & 0.500 $\pm$ 0.000 & 0.500 $\pm$ 0.000 & \textbf{0.930 $\pm$ 0.035} & \textbf{0.930 $\pm$ 0.035}\\
&Test AUROC&\textbf{0.684 $\pm$ 0.113} & \textbf{0.684 $\pm$ 0.113} & 0.571 $\pm$ 0.094 & 0.571 $\pm$ 0.094 & 0.500 $\pm$ 0.000 & 0.500 $\pm$ 0.000 & \underline{0.666 $\pm$ 0.123} & \underline{0.666 $\pm$ 0.123}\\
&Train ECE&0.054 $\pm$ 0.011 & \textbf{0.000 $\pm$ 0.000} & 0.026 $\pm$ 0.034 & \textbf{0.000 $\pm$ 0.000} & 0.008 $\pm$ 0.004 & \textbf{0.000 $\pm$ 0.000} & 0.177 $\pm$ 0.020 & \textbf{0.000 $\pm$ 0.000}\\
&Test ECE&0.127 $\pm$ 0.043 & 0.106 $\pm$ 0.044 & 0.079 $\pm$ 0.038 & 0.078 $\pm$ 0.046 & \textbf{0.061 $\pm$ 0.037} & \underline{0.069 $\pm$ 0.039} & 0.154 $\pm$ 0.029 & 0.122 $\pm$ 0.075\\
&Train AUNBC&0.452 $\pm$ 0.014 & 0.457 $\pm$ 0.013 & 0.427 $\pm$ 0.013 & 0.431 $\pm$ 0.015 & 0.422 $\pm$ 0.012 & 0.422 $\pm$ 0.012 & \underline{0.519 $\pm$ 0.023} & \textbf{0.577 $\pm$ 0.032}\\
&Test AUNBC&\textbf{0.442 $\pm$ 0.100} & \underline{0.440 $\pm$ 0.099} & 0.425 $\pm$ 0.105 & 0.423 $\pm$ 0.102 & 0.421 $\pm$ 0.106 & 0.421 $\pm$ 0.106 & 0.439 $\pm$ 0.104 & 0.326 $\pm$ 0.169\\
\midrule
heartdisease&Train AUROC&\underline{0.938 $\pm$ 0.005} & \underline{0.938 $\pm$ 0.005} & 0.924 $\pm$ 0.006 & 0.924 $\pm$ 0.005 & 0.930 $\pm$ 0.007 & 0.930 $\pm$ 0.007 & \textbf{1.000 $\pm$ 0.000} & \textbf{1.000 $\pm$ 0.000}\\
&Test AUROC&0.870 $\pm$ 0.067 & 0.870 $\pm$ 0.067 & \textbf{0.897 $\pm$ 0.069} & 0.889 $\pm$ 0.063 & 0.853 $\pm$ 0.066 & 0.854 $\pm$ 0.065 & \underline{0.892 $\pm$ 0.055} & \underline{0.892 $\pm$ 0.055}\\
&Train ECE&0.031 $\pm$ 0.007 & \textbf{0.000 $\pm$ 0.000} & 0.088 $\pm$ 0.014 & \textbf{0.000 $\pm$ 0.000} & 0.037 $\pm$ 0.009 & \textbf{0.000 $\pm$ 0.000} & 0.172 $\pm$ 0.006 & \textbf{0.000 $\pm$ 0.000}\\
&Test ECE&0.137 $\pm$ 0.055 & 0.116 $\pm$ 0.050 & 0.158 $\pm$ 0.030 & \underline{0.109 $\pm$ 0.048} & 0.132 $\pm$ 0.048 & 0.114 $\pm$ 0.048 & 0.130 $\pm$ 0.038 & \textbf{0.098 $\pm$ 0.036}\\
&Train AUNBC&0.332 $\pm$ 0.010 & 0.335 $\pm$ 0.010 & 0.305 $\pm$ 0.009 & 0.319 $\pm$ 0.012 & 0.326 $\pm$ 0.012 & 0.328 $\pm$ 0.013 & \underline{0.378 $\pm$ 0.009} & \textbf{0.450 $\pm$ 0.012}\\
&Test AUNBC&0.269 $\pm$ 0.089 & 0.263 $\pm$ 0.090 & \underline{0.280 $\pm$ 0.079} & \textbf{0.284 $\pm$ 0.099} & 0.257 $\pm$ 0.106 & 0.252 $\pm$ 0.107 & 0.276 $\pm$ 0.089 & 0.200 $\pm$ 0.170\\
\midrule
mammo&Train AUROC&\underline{0.860 $\pm$ 0.004} & \underline{0.860 $\pm$ 0.004} & 0.852 $\pm$ 0.005 & 0.852 $\pm$ 0.005 & 0.852 $\pm$ 0.004 & 0.852 $\pm$ 0.004 & \textbf{0.877 $\pm$ 0.004} & \textbf{0.877 $\pm$ 0.004}\\
&Test AUROC&\textbf{0.851 $\pm$ 0.035} & \underline{0.850 $\pm$ 0.034} & 0.848 $\pm$ 0.039 & 0.847 $\pm$ 0.039 & 0.848 $\pm$ 0.034 & 0.848 $\pm$ 0.034 & 0.845 $\pm$ 0.037 & 0.845 $\pm$ 0.037\\
&Train ECE&0.024 $\pm$ 0.004 & \textbf{0.000 $\pm$ 0.000} & 0.056 $\pm$ 0.014 & \textbf{0.000 $\pm$ 0.000} & 0.026 $\pm$ 0.005 & \textbf{0.000 $\pm$ 0.000} & 0.040 $\pm$ 0.003 & \textbf{0.000 $\pm$ 0.000}\\
&Test ECE&0.095 $\pm$ 0.025 & 0.081 $\pm$ 0.026 & 0.086 $\pm$ 0.023 & \textbf{0.069 $\pm$ 0.018} & 0.078 $\pm$ 0.021 & \underline{0.077 $\pm$ 0.019} & 0.097 $\pm$ 0.034 & 0.081 $\pm$ 0.034\\
&Train AUNBC&0.256 $\pm$ 0.006 & 0.259 $\pm$ 0.006 & 0.248 $\pm$ 0.007 & 0.255 $\pm$ 0.006 & 0.254 $\pm$ 0.006 & 0.254 $\pm$ 0.006 & \underline{0.266 $\pm$ 0.005} & \textbf{0.267 $\pm$ 0.005}\\
&Test AUNBC&0.248 $\pm$ 0.053 & \underline{0.252 $\pm$ 0.053} & 0.246 $\pm$ 0.048 & 0.250 $\pm$ 0.050 & \textbf{0.253 $\pm$ 0.055} & 0.252 $\pm$ 0.055 & 0.249 $\pm$ 0.051 & 0.250 $\pm$ 0.052\\
\midrule
mushroom&Train AUROC&\textbf{1.000 $\pm$ 0.000} & \textbf{1.000 $\pm$ 0.000} & \textbf{1.000 $\pm$ 0.000} & \textbf{1.000 $\pm$ 0.000} & \textbf{1.000 $\pm$ 0.000} & \textbf{1.000 $\pm$ 0.000} & \textbf{1.000 $\pm$ 0.000} & \textbf{1.000 $\pm$ 0.000}\\
&Test AUROC&\textbf{1.000 $\pm$ 0.000} & \textbf{1.000 $\pm$ 0.000} & \textbf{1.000 $\pm$ 0.000} & \textbf{1.000 $\pm$ 0.000} & \textbf{1.000 $\pm$ 0.000} & \textbf{1.000 $\pm$ 0.000} & \textbf{1.000 $\pm$ 0.000} & \textbf{1.000 $\pm$ 0.000}\\
&Train ECE&\textbf{0.000 $\pm$ 0.000} & \textbf{0.000 $\pm$ 0.000} & 0.001 $\pm$ 0.000 & \textbf{0.000 $\pm$ 0.000} & \textbf{0.000 $\pm$ 0.000} & \textbf{0.000 $\pm$ 0.000} & 0.005 $\pm$ 0.001 & \textbf{0.000 $\pm$ 0.000}\\
&Test ECE&\textbf{0.000 $\pm$ 0.000} & \textbf{0.000 $\pm$ 0.000} & 0.001 $\pm$ 0.001 & \textbf{0.000 $\pm$ 0.000} & \textbf{0.000 $\pm$ 0.000} & \textbf{0.000 $\pm$ 0.000} & 0.005 $\pm$ 0.001 & \textbf{0.000 $\pm$ 0.000}\\
&Train AUNBC&\textbf{0.482 $\pm$ 0.002} & \textbf{0.482 $\pm$ 0.002} & \underline{0.482 $\pm$ 0.002} & \textbf{0.482 $\pm$ 0.002} & \textbf{0.482 $\pm$ 0.002} & \textbf{0.482 $\pm$ 0.002} & 0.481 $\pm$ 0.002 & \textbf{0.482 $\pm$ 0.002}\\
&Test AUNBC&\textbf{0.482 $\pm$ 0.017} & \textbf{0.482 $\pm$ 0.017} & \underline{0.482 $\pm$ 0.017} & \textbf{0.482 $\pm$ 0.017} & \textbf{0.482 $\pm$ 0.017} & \textbf{0.482 $\pm$ 0.017} & 0.481 $\pm$ 0.017 & \textbf{0.482 $\pm$ 0.017}\\
\midrule
spambase&Train AUROC&0.965 $\pm$ 0.039 & 0.965 $\pm$ 0.039 & \underline{0.974 $\pm$ 0.001} & \underline{0.974 $\pm$ 0.001} & 0.973 $\pm$ 0.002 & 0.973 $\pm$ 0.002 & \textbf{0.994 $\pm$ 0.000} & \textbf{0.994 $\pm$ 0.000}\\
&Test AUROC&0.959 $\pm$ 0.045 & 0.959 $\pm$ 0.045 & \underline{0.970 $\pm$ 0.009} & \underline{0.970 $\pm$ 0.009} & 0.969 $\pm$ 0.009 & 0.969 $\pm$ 0.009 & \textbf{0.987 $\pm$ 0.003} & \textbf{0.987 $\pm$ 0.003}\\
&Train ECE&0.030 $\pm$ 0.025 & \textbf{0.000 $\pm$ 0.000} & 0.046 $\pm$ 0.004 & \textbf{0.000 $\pm$ 0.000} & 0.027 $\pm$ 0.008 & \textbf{0.000 $\pm$ 0.000} & 0.019 $\pm$ 0.002 & \textbf{0.000 $\pm$ 0.000}\\
&Test ECE&0.041 $\pm$ 0.029 & 0.017 $\pm$ 0.007 & 0.050 $\pm$ 0.012 & \textbf{0.013 $\pm$ 0.004} & 0.039 $\pm$ 0.012 & \underline{0.014 $\pm$ 0.006} & 0.027 $\pm$ 0.006 & 0.019 $\pm$ 0.008\\
&Train AUNBC&0.313 $\pm$ 0.016 & 0.315 $\pm$ 0.016 & 0.307 $\pm$ 0.004 & 0.311 $\pm$ 0.004 & 0.306 $\pm$ 0.004 & 0.310 $\pm$ 0.004 & \underline{0.361 $\pm$ 0.002} & \textbf{0.361 $\pm$ 0.002}\\
&Test AUNBC&0.303 $\pm$ 0.026 & 0.306 $\pm$ 0.026 & 0.303 $\pm$ 0.019 & 0.305 $\pm$ 0.020 & 0.299 $\pm$ 0.020 & 0.304 $\pm$ 0.022 & \textbf{0.343 $\pm$ 0.016} & \underline{0.341 $\pm$ 0.016}\\

\end{longtable}
\end{landscape}
\section{Sensitivity Analysis of Threshold Gridding and Weighting Scheme}\label{app: sensitivity_analysis}
In this appendix, we investigate the sensitivity of RSS-DNB and RSS-DNB-SA to the threshold gridding and weighting scheme used in the optimization objective. All evaluation metrics are computed using the original evaluation protocol with uniformly spaced thresholds $p_i=\frac{i}{10}$, $i=0,1,\ldots,9$, as described in \Cref{sec:experiments}. Therefore, any performance differences depend on optimization setting, rather than evaluation criteria. 

We consider five experiment settings, as shown in \Cref{tab:sensitivity_settings}. The first four settings investigate the effect of threshold by using different threshold grids with equal weights. The last two settings investigate the effect of the weighting scheme using the same randomly generated threshold grid. 

\begin{table}[htbp]
    \centering
    \renewcommand{\arraystretch}{1.2}
    \caption{Experiment settings for the sensitivity analysis}
    \label{tab:sensitivity_settings}
    \begin{tabular}{ccc}
    \toprule
    Setting & Thresholds & Weights\\
    \midrule
    Original & $p_i = \frac{i}{10}$ & Equal \\
    Fine     & $p_i = \frac{i}{20}$ & Equal \\
    Coarse   & $p_i = \frac{i}{5}$  & Equal \\
    Random-EW   & Fixed random thresholds & Equal \\
    Random-LW   & Same random thresholds & $\omega_i=p_{i+1}-p_i$\\
    \bottomrule
    \end{tabular}
\end{table}
For the random threshold setting, a single threshold grid was generated once and shared across all datasets. The thresholds are $\{0,  \,  0.0857,  \,  0.1102, \,   0.1325,  \,  0.2292,  \,  0.4491,  \,  0.5409,  \,  0.5727,   \, 0.5778,  \,  0.6975\}$.

\Cref{tab:sensitivity_rss-dnb,tab:sensitivity_rss-dnb-sa} summarizes the performance of RSS-DNB and RSS-DNB-SA under different experiment settings, respectively. As shown in tables, test AUROC and AUNBC remain similar across several settings, but some differences are non-negligible. For RSS-DNB, both metric decrease on adult dataset with the fine grid and on spambase dataset with the coarse grid. For RSS-DNB-SA, both metric decrease on spambase with the fine grid. On haberman dataset, the random grid yield lower test AUROC for both models. Overall, the fine grid more often produce better performance, and the random grid also gives results comparable to the uniform grid on the several other datasets.

Overall, the results suggests that both RSS-DNB and RSS-DNB-SA are reasonably robust to choice of threshold gridding and weighting scheme, and the conclusion presented in the \Cref{sec:experiments} do not highly depend on the particular optimization settings.

\begin{landscape}
\footnotesize
\setlength{\tabcolsep}{4pt}
\renewcommand{\arraystretch}{1.1}
\newcolumntype{C}[1]{>{\centering\arraybackslash}p{#1}}
\newcolumntype{R}[1]{>{\raggedleft\arraybackslash}p{#1}}
\begin{longtable}{C{2cm}R{2.5cm}C{2.5cm}C{2.5cm}C{2.5cm}C{2.5cm}C{2.5cm}}
\multicolumn{7}{l}{\textbf{Table \thetable\ :}  Performance of RSS-DNB under different experiment settings.} \label{tab:sensitivity_rss-dnb} \\
\toprule
Dataset & Metric & Original & Fine & Coarse & Random-EW & Random-LW  \\
\midrule
\endfirsthead

\multicolumn{7}{l}{\textbf{Table \thetable\ :}  continued} \\
\toprule
Dataset & Metric & Original & Fine & Coarse & Random-EW & Random-LW \\
\midrule
\endhead

\bottomrule
\multicolumn{7}{l}{\footnotesize Notes: All values are reported as mean $\pm$ standard deviation over 10-fold cross-validation. Size is reported as mean (minimum--maximum) across the 10 folds.} \\
\multicolumn{7}{l}{\footnotesize The best and second-best results are highlighted in bold and underline, respectively.} \\
\endlastfoot
adult&Train AUROC&\textbf{0.881 $\pm$ 0.002} & 0.840 $\pm$ 0.057 & 0.861 $\pm$ 0.003 & 0.876 $\pm$ 0.002 & \underline{0.876 $\pm$ 0.002}\\
&Test AUROC&\textbf{0.880 $\pm$ 0.007} & 0.838 $\pm$ 0.059 & 0.859 $\pm$ 0.008 & 0.875 $\pm$ 0.007 & \underline{0.875 $\pm$ 0.008}\\
&Train ECE&\textbf{0.000 $\pm$ 0.000} & \textbf{0.000 $\pm$ 0.000} & \textbf{0.000 $\pm$ 0.000} & \textbf{0.000 $\pm$ 0.000} & \textbf{0.000 $\pm$ 0.000}\\
&Test ECE&0.013 $\pm$ 0.003 & \textbf{0.009 $\pm$ 0.005} & \underline{0.010 $\pm$ 0.003} & 0.010 $\pm$ 0.004 & 0.011 $\pm$ 0.004\\
&Train AUNBC&\textbf{0.102 $\pm$ 0.000} & 0.091 $\pm$ 0.016 & \underline{0.101 $\pm$ 0.001} & 0.099 $\pm$ 0.001 & 0.099 $\pm$ 0.001\\
&Test AUNBC&\textbf{0.102 $\pm$ 0.003} & 0.090 $\pm$ 0.015 & \underline{0.100 $\pm$ 0.003} & 0.099 $\pm$ 0.004 & 0.099 $\pm$ 0.004\\
&Size&22.2 (19-24) & \textbf{14.2 (2-24)} & 23.0 (20-25) & \underline{22.0 (19-24)} & \underline{22.0 (19-24)}\\
\midrule
bankruptcy&Train AUROC&\textbf{1.000 $\pm$ 0.000} & \textbf{1.000 $\pm$ 0.000} & \textbf{1.000 $\pm$ 0.000} & \underline{0.999 $\pm$ 0.003} & 0.996 $\pm$ 0.007\\
&Test AUROC&\textbf{0.997 $\pm$ 0.011} & 0.987 $\pm$ 0.032 & 0.987 $\pm$ 0.032 & \underline{0.990 $\pm$ 0.032} & 0.981 $\pm$ 0.034\\
&Train ECE&\textbf{0.000 $\pm$ 0.000} & \textbf{0.000 $\pm$ 0.000} & \textbf{0.000 $\pm$ 0.000} & \textbf{0.000 $\pm$ 0.000} & \textbf{0.000 $\pm$ 0.000}\\
&Test ECE&0.004 $\pm$ 0.013 & 0.004 $\pm$ 0.013 & 0.004 $\pm$ 0.013 & \underline{0.002 $\pm$ 0.004} & \textbf{0.000 $\pm$ 0.000}\\
&Train AUNBC&\textbf{0.572 $\pm$ 0.016} & \textbf{0.572 $\pm$ 0.016} & \textbf{0.572 $\pm$ 0.016} & \underline{0.567 $\pm$ 0.017} & 0.563 $\pm$ 0.018\\
&Test AUNBC&\textbf{0.568 $\pm$ 0.147} & \underline{0.561 $\pm$ 0.135} & \underline{0.561 $\pm$ 0.135} & 0.556 $\pm$ 0.152 & 0.541 $\pm$ 0.157\\
&Size&2.9 (2-3) & 3.4 (3-5) & 3.4 (3-5) & \underline{1.9 (1-2)} & \textbf{1.7 (1-2)}\\
\midrule
breastcancer&Train AUROC&\underline{0.992 $\pm$ 0.002} & \textbf{0.996 $\pm$ 0.002} & 0.990 $\pm$ 0.003 & 0.988 $\pm$ 0.003 & 0.988 $\pm$ 0.003\\
&Test AUROC&\underline{0.985 $\pm$ 0.015} & \textbf{0.988 $\pm$ 0.014} & 0.970 $\pm$ 0.028 & 0.971 $\pm$ 0.027 & 0.965 $\pm$ 0.029\\
&Train ECE&\textbf{0.000 $\pm$ 0.000} & \textbf{0.000 $\pm$ 0.000} & \textbf{0.000 $\pm$ 0.000} & \textbf{0.000 $\pm$ 0.000} & \textbf{0.000 $\pm$ 0.000}\\
&Test ECE&0.022 $\pm$ 0.010 & 0.023 $\pm$ 0.010 & \textbf{0.018 $\pm$ 0.017} & 0.018 $\pm$ 0.015 & \textbf{0.018 $\pm$ 0.012}\\
&Train AUNBC&\underline{0.325 $\pm$ 0.006} & 0.324 $\pm$ 0.006 & \textbf{0.326 $\pm$ 0.006} & 0.318 $\pm$ 0.007 & 0.316 $\pm$ 0.006\\
&Test AUNBC&\underline{0.305 $\pm$ 0.057} & \textbf{0.312 $\pm$ 0.054} & 0.300 $\pm$ 0.056 & 0.294 $\pm$ 0.066 & 0.282 $\pm$ 0.072\\
&Size&6.5 (5-7) & 8.5 (8-9) & 7.7 (6-9) & \underline{3.2 (3-4)} & \textbf{3.1 (2-4)}\\
\midrule
haberman&Train AUROC&0.733 $\pm$ 0.015 & \underline{0.734 $\pm$ 0.010} & \textbf{0.736 $\pm$ 0.008} & 0.683 $\pm$ 0.017 & 0.683 $\pm$ 0.017\\
&Test AUROC&0.662 $\pm$ 0.119 & \underline{0.666 $\pm$ 0.125} & \textbf{0.682 $\pm$ 0.108} & 0.636 $\pm$ 0.089 & 0.636 $\pm$ 0.089\\
&Train ECE&\textbf{0.000 $\pm$ 0.000} & \textbf{0.000 $\pm$ 0.000} & \textbf{0.000 $\pm$ 0.000} & \textbf{0.000 $\pm$ 0.000} & \textbf{0.000 $\pm$ 0.000}\\
&Test ECE&0.118 $\pm$ 0.061 & 0.142 $\pm$ 0.042 & \underline{0.117 $\pm$ 0.040} & \textbf{0.094 $\pm$ 0.041} & \textbf{0.094 $\pm$ 0.041}\\
&Train AUNBC&\underline{0.476 $\pm$ 0.012} & 0.475 $\pm$ 0.012 & \textbf{0.479 $\pm$ 0.012} & 0.461 $\pm$ 0.011 & 0.461 $\pm$ 0.011\\
&Test AUNBC&0.440 $\pm$ 0.107 & \underline{0.444 $\pm$ 0.096} & \textbf{0.452 $\pm$ 0.113} & 0.443 $\pm$ 0.093 & 0.443 $\pm$ 0.093\\
&Size&2.4 (2-3) & \underline{2.2 (2-3)} & 2.9 (2-3) & \textbf{1.0 (1-1)} & \textbf{1.0 (1-1)}\\
\midrule
heartdisease&Train AUROC&0.926 $\pm$ 0.008 & \textbf{0.940 $\pm$ 0.005} & 0.928 $\pm$ 0.014 & \underline{0.930 $\pm$ 0.010} & \underline{0.930 $\pm$ 0.007}\\
&Test AUROC&0.819 $\pm$ 0.086 & 0.863 $\pm$ 0.076 & 0.847 $\pm$ 0.073 & \textbf{0.868 $\pm$ 0.075} & \underline{0.865 $\pm$ 0.083}\\
&Train ECE&\textbf{0.000 $\pm$ 0.000} & \textbf{0.000 $\pm$ 0.000} & \textbf{0.000 $\pm$ 0.000} & \textbf{0.000 $\pm$ 0.000} & \textbf{0.000 $\pm$ 0.000}\\
&Test ECE&0.079 $\pm$ 0.033 & 0.074 $\pm$ 0.039 & \underline{0.070 $\pm$ 0.039} & \textbf{0.064 $\pm$ 0.032} & 0.081 $\pm$ 0.040\\
&Train AUNBC&\underline{0.359 $\pm$ 0.011} & 0.357 $\pm$ 0.014 & \textbf{0.366 $\pm$ 0.015} & 0.348 $\pm$ 0.015 & 0.345 $\pm$ 0.012\\
&Test AUNBC&0.206 $\pm$ 0.145 & \textbf{0.258 $\pm$ 0.131} & \underline{0.240 $\pm$ 0.155} & 0.236 $\pm$ 0.177 & 0.239 $\pm$ 0.160\\
&Size&16.3 (12-22) & 20.6 (15-26) & 19.4 (16-25) & \underline{12.5 (10-17)} & \textbf{11.8 (10-14)}\\
\midrule
mammo&Train AUROC&\underline{0.859 $\pm$ 0.003} & \textbf{0.863 $\pm$ 0.004} & 0.857 $\pm$ 0.004 & 0.851 $\pm$ 0.003 & 0.847 $\pm$ 0.003\\
&Test AUROC&\underline{0.845 $\pm$ 0.031} & \textbf{0.845 $\pm$ 0.029} & 0.835 $\pm$ 0.034 & 0.841 $\pm$ 0.031 & 0.842 $\pm$ 0.039\\
&Train ECE&\textbf{0.000 $\pm$ 0.000} & \textbf{0.000 $\pm$ 0.000} & \textbf{0.000 $\pm$ 0.000} & \textbf{0.000 $\pm$ 0.000} & \textbf{0.000 $\pm$ 0.000}\\
&Test ECE&0.078 $\pm$ 0.019 & 0.073 $\pm$ 0.021 & \textbf{0.062 $\pm$ 0.020} & \underline{0.068 $\pm$ 0.020} & 0.074 $\pm$ 0.021\\
&Train AUNBC&0.262 $\pm$ 0.006 & \textbf{0.263 $\pm$ 0.006} & \underline{0.263 $\pm$ 0.006} & 0.257 $\pm$ 0.006 & 0.254 $\pm$ 0.006\\
&Test AUNBC&\textbf{0.252 $\pm$ 0.052} & \underline{0.250 $\pm$ 0.050} & 0.247 $\pm$ 0.055 & 0.248 $\pm$ 0.055 & 0.249 $\pm$ 0.054\\
&Size&6.5 (6-7) & 11.5 (11-12) & 11.0 (10-12) & \underline{4.3 (4-6)} & \textbf{3.2 (3-4)}\\
\midrule
mushroom&Train AUROC&\textbf{1.000 $\pm$ 0.000} & \textbf{1.000 $\pm$ 0.000} & \textbf{1.000 $\pm$ 0.000} & \textbf{1.000 $\pm$ 0.000} & \textbf{1.000 $\pm$ 0.000}\\
&Test AUROC&\textbf{1.000 $\pm$ 0.000} & \textbf{1.000 $\pm$ 0.000} & \textbf{1.000 $\pm$ 0.000} & \textbf{1.000 $\pm$ 0.000} & \textbf{1.000 $\pm$ 0.000}\\
&Train ECE&\textbf{0.000 $\pm$ 0.000} & \textbf{0.000 $\pm$ 0.000} & \textbf{0.000 $\pm$ 0.000} & \textbf{0.000 $\pm$ 0.000} & \textbf{0.000 $\pm$ 0.000}\\
&Test ECE&\textbf{0.000 $\pm$ 0.000} & \textbf{0.000 $\pm$ 0.000} & \textbf{0.000 $\pm$ 0.000} & \textbf{0.000 $\pm$ 0.000} & \textbf{0.000 $\pm$ 0.000}\\
&Train AUNBC&\textbf{0.482 $\pm$ 0.002} & \textbf{0.482 $\pm$ 0.002} & \textbf{0.482 $\pm$ 0.002} & \textbf{0.482 $\pm$ 0.002} & \textbf{0.482 $\pm$ 0.002}\\
&Test AUNBC&\textbf{0.482 $\pm$ 0.017} & \textbf{0.482 $\pm$ 0.017} & \textbf{0.482 $\pm$ 0.017} & \textbf{0.482 $\pm$ 0.017} & \textbf{0.482 $\pm$ 0.017}\\
&Size&22.3 (19-43) & 12.6 (7-15) & \underline{8.3 (7-13)} & \textbf{7.0 (7-7)} & \textbf{7.0 (7-7)}\\
\midrule
spambase&Train AUROC&\textbf{0.954 $\pm$ 0.010} & \underline{0.953 $\pm$ 0.020} & 0.866 $\pm$ 0.139 & 0.945 $\pm$ 0.011 & 0.936 $\pm$ 0.017\\
&Test AUROC&\textbf{0.951 $\pm$ 0.017} & \underline{0.949 $\pm$ 0.016} & 0.867 $\pm$ 0.139 & 0.938 $\pm$ 0.018 & 0.934 $\pm$ 0.016\\
&Train ECE&\textbf{0.000 $\pm$ 0.000} & \textbf{0.000 $\pm$ 0.000} & \textbf{0.000 $\pm$ 0.000} & \textbf{0.000 $\pm$ 0.000} & \textbf{0.000 $\pm$ 0.000}\\
&Test ECE&0.011 $\pm$ 0.008 & \textbf{0.010 $\pm$ 0.007} & 0.012 $\pm$ 0.008 & \underline{0.011 $\pm$ 0.006} & 0.018 $\pm$ 0.005\\
&Train AUNBC&\underline{0.298 $\pm$ 0.008} & \underline{0.298 $\pm$ 0.013} & 0.262 $\pm$ 0.068 & \textbf{0.301 $\pm$ 0.008} & 0.284 $\pm$ 0.014\\
&Test AUNBC&\textbf{0.295 $\pm$ 0.020} & \underline{0.293 $\pm$ 0.016} & 0.264 $\pm$ 0.071 & 0.289 $\pm$ 0.025 & 0.277 $\pm$ 0.020\\
&Size&\textbf{33.2 (28-37)} & \underline{33.6 (29-38)} & 35.5 (31-44) & 34.3 (26-41) & 37.1 (32-50)\\

\end{longtable}
\end{landscape}

\begin{landscape}
\footnotesize
\setlength{\tabcolsep}{4pt}
\renewcommand{\arraystretch}{1.1}
\newcolumntype{C}[1]{>{\centering\arraybackslash}p{#1}}
\newcolumntype{R}[1]{>{\raggedleft\arraybackslash}p{#1}}
\begin{longtable}{C{2cm}R{2.5cm}C{2.5cm}C{2.5cm}C{2.5cm}C{2.5cm}C{2.5cm}}
\multicolumn{7}{l}{\textbf{Table \thetable\ :}  Performance of RSS-DNB-SA under different experiment settings.} \label{tab:sensitivity_rss-dnb-sa} \\
\toprule
Dataset & Metric & Original & Fine & Coarse & Random-EW & Random-LW  \\
\midrule
\endfirsthead

\multicolumn{7}{l}{\textbf{Table \thetable\ :}  continued} \\
\toprule
Dataset & Metric & Original & Fine & Coarse & Random-EW & Random-LW \\
\midrule
\endhead

\bottomrule
\multicolumn{7}{l}{\footnotesize Notes: All values are reported as mean $\pm$ standard deviation over 10-fold cross-validation. Size is reported as mean (minimum--maximum) across the 10 folds.} \\
\multicolumn{7}{l}{\footnotesize The best and second-best results are highlighted in bold and underline, respectively.} \\
\endlastfoot
adult&Train AUROC&0.872 $\pm$ 0.003 & \textbf{0.880 $\pm$ 0.002} & 0.855 $\pm$ 0.005 & \underline{0.873 $\pm$ 0.003} & 0.872 $\pm$ 0.002\\
&Test AUROC&0.872 $\pm$ 0.008 & \textbf{0.879 $\pm$ 0.006} & 0.854 $\pm$ 0.007 & \underline{0.872 $\pm$ 0.008} & 0.869 $\pm$ 0.008\\
&Train ECE&\textbf{0.000 $\pm$ 0.000} & \textbf{0.000 $\pm$ 0.000} & \textbf{0.000 $\pm$ 0.000} & \textbf{0.000 $\pm$ 0.000} & \textbf{0.000 $\pm$ 0.000}\\
&Test ECE&0.014 $\pm$ 0.004 & 0.013 $\pm$ 0.003 & \textbf{0.009 $\pm$ 0.003} & 0.013 $\pm$ 0.004 & \underline{0.012 $\pm$ 0.003}\\
&Train AUNBC&\underline{0.101 $\pm$ 0.001} & \textbf{0.101 $\pm$ 0.001} & 0.099 $\pm$ 0.001 & 0.099 $\pm$ 0.001 & 0.098 $\pm$ 0.001\\
&Test AUNBC&\textbf{0.100 $\pm$ 0.003} & \textbf{0.100 $\pm$ 0.004} & 0.098 $\pm$ 0.003 & 0.098 $\pm$ 0.003 & 0.097 $\pm$ 0.003\\
&Size&18.6 (15-22) & \textbf{17.7 (14-24)} & 18.7 (13-23) & 19.5 (12-25) & \underline{18.4 (15-21)}\\
\midrule
bankruptcy&Train AUROC&\textbf{1.000 $\pm$ 0.000} & \textbf{1.000 $\pm$ 0.000} & 0.999 $\pm$ 0.001 & 0.999 $\pm$ 0.003 & 0.996 $\pm$ 0.007\\
&Test AUROC&\underline{0.997 $\pm$ 0.011} & \textbf{0.999 $\pm$ 0.002} & \underline{0.997 $\pm$ 0.011} & 0.990 $\pm$ 0.032 & 0.981 $\pm$ 0.034\\
&Train ECE&\textbf{0.000 $\pm$ 0.000} & \textbf{0.000 $\pm$ 0.000} & \textbf{0.000 $\pm$ 0.000} & \textbf{0.000 $\pm$ 0.000} & \textbf{0.000 $\pm$ 0.000}\\
&Test ECE&0.006 $\pm$ 0.013 & 0.007 $\pm$ 0.011 & 0.010 $\pm$ 0.014 & \underline{0.002 $\pm$ 0.004} & \textbf{0.000 $\pm$ 0.000}\\
&Train AUNBC&\textbf{0.569 $\pm$ 0.016} & 0.567 $\pm$ 0.016 & \underline{0.568 $\pm$ 0.015} & 0.567 $\pm$ 0.017 & 0.563 $\pm$ 0.018\\
&Test AUNBC&\textbf{0.568 $\pm$ 0.147} & \underline{0.565 $\pm$ 0.148} & 0.564 $\pm$ 0.150 & 0.556 $\pm$ 0.152 & 0.541 $\pm$ 0.157\\
&Size&2.1 (2-3) & \underline{2.0 (2-2)} & 2.2 (2-3) & \underline{2.0 (1-3)} & \textbf{1.9 (1-3)}\\
\midrule
breastcancer&Train AUROC&\underline{0.990 $\pm$ 0.002} & \textbf{0.994 $\pm$ 0.002} & 0.987 $\pm$ 0.003 & 0.987 $\pm$ 0.003 & 0.986 $\pm$ 0.003\\
&Test AUROC&0.972 $\pm$ 0.030 & \textbf{0.986 $\pm$ 0.014} & 0.972 $\pm$ 0.031 & \underline{0.973 $\pm$ 0.025} & 0.971 $\pm$ 0.029\\
&Train ECE&\textbf{0.000 $\pm$ 0.000} & \textbf{0.000 $\pm$ 0.000} & \textbf{0.000 $\pm$ 0.000} & \textbf{0.000 $\pm$ 0.000} & \textbf{0.000 $\pm$ 0.000}\\
&Test ECE&0.022 $\pm$ 0.013 & 0.022 $\pm$ 0.012 & 0.019 $\pm$ 0.014 & \underline{0.017 $\pm$ 0.013} & \textbf{0.014 $\pm$ 0.013}\\
&Train AUNBC&\underline{0.320 $\pm$ 0.007} & \textbf{0.320 $\pm$ 0.007} & 0.319 $\pm$ 0.006 & 0.317 $\pm$ 0.006 & 0.316 $\pm$ 0.006\\
&Test AUNBC&\underline{0.298 $\pm$ 0.066} & \textbf{0.305 $\pm$ 0.056} & 0.297 $\pm$ 0.059 & 0.296 $\pm$ 0.065 & 0.290 $\pm$ 0.072\\
&Size&\underline{3.5 (3-5)} & 3.9 (3-5) & \textbf{3.4 (3-4)} & 3.6 (3-4) & \textbf{3.4 (3-4)}\\
\midrule
haberman&Train AUROC&\underline{0.728 $\pm$ 0.011} & \textbf{0.735 $\pm$ 0.008} & 0.696 $\pm$ 0.011 & 0.683 $\pm$ 0.017 & 0.683 $\pm$ 0.017\\
&Test AUROC&\textbf{0.667 $\pm$ 0.138} & 0.654 $\pm$ 0.123 & \underline{0.665 $\pm$ 0.100} & 0.639 $\pm$ 0.092 & 0.639 $\pm$ 0.092\\
&Train ECE&\textbf{0.000 $\pm$ 0.000} & \textbf{0.000 $\pm$ 0.000} & \textbf{0.000 $\pm$ 0.000} & \textbf{0.000 $\pm$ 0.000} & \textbf{0.000 $\pm$ 0.000}\\
&Test ECE&0.123 $\pm$ 0.049 & 0.133 $\pm$ 0.045 & \underline{0.111 $\pm$ 0.039} & \textbf{0.096 $\pm$ 0.039} & \textbf{0.096 $\pm$ 0.039}\\
&Train AUNBC&\underline{0.474 $\pm$ 0.011} & \textbf{0.475 $\pm$ 0.011} & 0.462 $\pm$ 0.011 & 0.461 $\pm$ 0.011 & 0.461 $\pm$ 0.011\\
&Test AUNBC&0.442 $\pm$ 0.111 & 0.438 $\pm$ 0.107 & \textbf{0.448 $\pm$ 0.095} & \underline{0.442 $\pm$ 0.093} & \underline{0.442 $\pm$ 0.093}\\
&Size&\underline{1.9 (1-2)} & 2.0 (2-2) & \textbf{1.0 (1-1)} & \textbf{1.0 (1-1)} & \textbf{1.0 (1-1)}\\
\midrule
heartdisease&Train AUROC&0.916 $\pm$ 0.010 & \textbf{0.930 $\pm$ 0.005} & 0.921 $\pm$ 0.010 & \underline{0.926 $\pm$ 0.009} & 0.924 $\pm$ 0.006\\
&Test AUROC&0.841 $\pm$ 0.077 & \textbf{0.875 $\pm$ 0.072} & \underline{0.866 $\pm$ 0.063} & 0.842 $\pm$ 0.052 & 0.854 $\pm$ 0.071\\
&Train ECE&\textbf{0.000 $\pm$ 0.000} & \textbf{0.000 $\pm$ 0.000} & \textbf{0.000 $\pm$ 0.000} & \textbf{0.000 $\pm$ 0.000} & \textbf{0.000 $\pm$ 0.000}\\
&Test ECE&0.079 $\pm$ 0.052 & 0.089 $\pm$ 0.026 & \textbf{0.063 $\pm$ 0.023} & \underline{0.070 $\pm$ 0.028} & 0.083 $\pm$ 0.030\\
&Train AUNBC&0.352 $\pm$ 0.008 & \textbf{0.355 $\pm$ 0.010} & \underline{0.354 $\pm$ 0.011} & 0.347 $\pm$ 0.014 & 0.338 $\pm$ 0.012\\
&Test AUNBC&0.248 $\pm$ 0.128 & \textbf{0.300 $\pm$ 0.114} & \underline{0.272 $\pm$ 0.101} & 0.237 $\pm$ 0.102 & 0.243 $\pm$ 0.105\\
&Size&\textbf{23.5 (20-26)} & 24.0 (22-28) & \textbf{23.5 (18-28)} & \textbf{23.5 (20-27)} & \underline{23.9 (21-28)}\\
\midrule
mammo&Train AUROC&\underline{0.852 $\pm$ 0.006} & \textbf{0.854 $\pm$ 0.004} & 0.845 $\pm$ 0.004 & 0.851 $\pm$ 0.004 & 0.849 $\pm$ 0.005\\
&Test AUROC&\textbf{0.847 $\pm$ 0.033} & \underline{0.846 $\pm$ 0.029} & 0.840 $\pm$ 0.031 & 0.835 $\pm$ 0.034 & 0.844 $\pm$ 0.036\\
&Train ECE&\textbf{0.000 $\pm$ 0.000} & \textbf{0.000 $\pm$ 0.000} & \textbf{0.000 $\pm$ 0.000} & \textbf{0.000 $\pm$ 0.000} & \textbf{0.000 $\pm$ 0.000}\\
&Test ECE&0.078 $\pm$ 0.020 & 0.078 $\pm$ 0.015 & \textbf{0.059 $\pm$ 0.024} & \underline{0.064 $\pm$ 0.023} & 0.071 $\pm$ 0.019\\
&Train AUNBC&\underline{0.257 $\pm$ 0.006} & \textbf{0.258 $\pm$ 0.006} & 0.255 $\pm$ 0.006 & 0.257 $\pm$ 0.006 & 0.255 $\pm$ 0.006\\
&Test AUNBC&\textbf{0.251 $\pm$ 0.052} & 0.248 $\pm$ 0.047 & \underline{0.251 $\pm$ 0.048} & 0.244 $\pm$ 0.053 & \underline{0.251 $\pm$ 0.054}\\
&Size&4.6 (4-6) & 5.0 (5-5) & \textbf{3.4 (3-5)} & 4.6 (4-6) & \underline{3.6 (3-5)}\\
\midrule
mushroom&Train AUROC&\textbf{1.000 $\pm$ 0.000} & \textbf{1.000 $\pm$ 0.000} & \textbf{1.000 $\pm$ 0.000} & \textbf{1.000 $\pm$ 0.000} & \textbf{1.000 $\pm$ 0.000}\\
&Test AUROC&\textbf{1.000 $\pm$ 0.000} & \textbf{1.000 $\pm$ 0.000} & \textbf{1.000 $\pm$ 0.000} & \textbf{1.000 $\pm$ 0.000} & \textbf{1.000 $\pm$ 0.000}\\
&Train ECE&\textbf{0.000 $\pm$ 0.000} & \textbf{0.000 $\pm$ 0.000} & \textbf{0.000 $\pm$ 0.000} & \textbf{0.000 $\pm$ 0.000} & \textbf{0.000 $\pm$ 0.000}\\
&Test ECE&\textbf{0.000 $\pm$ 0.000} & \textbf{0.000 $\pm$ 0.000} & \textbf{0.000 $\pm$ 0.000} & \textbf{0.000 $\pm$ 0.000} & \textbf{0.000 $\pm$ 0.000}\\
&Train AUNBC&\textbf{0.482 $\pm$ 0.002} & \textbf{0.482 $\pm$ 0.002} & \textbf{0.482 $\pm$ 0.002} & \textbf{0.482 $\pm$ 0.002} & \textbf{0.482 $\pm$ 0.002}\\
&Test AUNBC&\textbf{0.482 $\pm$ 0.017} & \textbf{0.482 $\pm$ 0.017} & \textbf{0.482 $\pm$ 0.017} & \textbf{0.482 $\pm$ 0.017} & \textbf{0.482 $\pm$ 0.017}\\
&Size&\underline{35.3 (30-43)} & \underline{35.3 (30-43)} & \underline{35.3 (30-43)} & \textbf{35.2 (30-43)} & \textbf{35.2 (30-43)}\\
\midrule
spambase&Train AUROC&\textbf{0.917 $\pm$ 0.037} & 0.863 $\pm$ 0.053 & \underline{0.906 $\pm$ 0.060} & 0.895 $\pm$ 0.062 & 0.893 $\pm$ 0.055\\
&Test AUROC&\textbf{0.919 $\pm$ 0.038} & 0.851 $\pm$ 0.058 & \underline{0.898 $\pm$ 0.066} & 0.885 $\pm$ 0.068 & 0.886 $\pm$ 0.046\\
&Train ECE&\textbf{0.000 $\pm$ 0.000} & \textbf{0.000 $\pm$ 0.000} & \textbf{0.000 $\pm$ 0.000} & \textbf{0.000 $\pm$ 0.000} & \textbf{0.000 $\pm$ 0.000}\\
&Test ECE&0.028 $\pm$ 0.009 & 0.025 $\pm$ 0.008 & \underline{0.017 $\pm$ 0.007} & 0.018 $\pm$ 0.010 & \textbf{0.015 $\pm$ 0.006}\\
&Train AUNBC&\underline{0.249 $\pm$ 0.037} & 0.219 $\pm$ 0.036 & \textbf{0.262 $\pm$ 0.054} & 0.245 $\pm$ 0.045 & 0.245 $\pm$ 0.045\\
&Test AUNBC&\underline{0.249 $\pm$ 0.044} & 0.210 $\pm$ 0.041 & \textbf{0.255 $\pm$ 0.064} & 0.238 $\pm$ 0.054 & 0.236 $\pm$ 0.041\\
&Size&32.1 (23-39) & \textbf{20.6 (9-48)} & 30.9 (13-40) & \underline{27.8 (9-39)} & 29.6 (9-41)\\

\end{longtable}
\end{landscape}
\section{Statistical Analysis}\label{app:statistical_analysis}
To assess the statistical significance of the differences among the compared methods, we conducted the Friedman test followed by the Nemenyi test across all datasets in Section 3. The analysis was performed separately for AUROC, AUNBC, ECE, and model size. The corresponding Friedman test statistics, average ranks, and pairwise comparison results from the Nemenyi test are reported in the following sections.
\subsection{Friedman test results}
The Friedman test was first applied to determine whether statistically significant differences existed among the compared methods for each evaluation metric. The null hypothesis assumes that all methods have equivalent performance. The chi-square statistics, degree of freedom  and corresponding p-values are summarized in \Cref{tab:friedman}.
\begin{table}[htbp]
    \centering
    \caption{Friedman test results}
    \label{tab:friedman}
    \begin{tabular}{crrr}
    \toprule
    Metric & $\chi^2$ & df & p-value  \\
    \midrule
    AUROC  & 29.00 & 7 & $<0.001$  \\
    AUNBC  & 18.36 & 7 & $0.010$ \\
    ECE    & 25.93 & 7 & $<0.001$ \\
    Size   & 17.46 & 5 & $0.004$ \\
    \bottomrule
    \end{tabular}
\end{table}
\subsection{Nemenyi Test Results}
The Nemenyi test was conducted to identify statistically significant differences between individual pairs of methods. The average ranks are summarized in \Cref{tab:ranks}. The pairwise p-values for AUROC, AUNBC, ECE, and model size are reported in \Cref{tab:nemenyi_auroc,tab:nemenyi_aunbc,tab:nemenyi_ece,tab:nemenyi_size}.
\begin{table}[htbp]
    \centering
    \caption{Average ranks for AUROC, AUNBC, ECE, and, model size. Lower average ranks indicate better performance.}
    \label{tab:ranks}
    \begin{tabular}{crrrr}
    \toprule
    Method & AUROC & AUNBC &ECE &Model Size\\
    \midrule
    RSS-DNB & 4.8125 &3.9375 &2.5625 & 3.0625\\
    RSS-DNB-SA     & 4.9375 &4.3125&3.3750  &2.1250\\
    Logistic     &2.8125 &2.8750 &5.2500&5.8125\\
    LASSO &2.9375  &4.0000 &6.6250 &3.3750\\
    Decision tree &6.8125 &6.3125 &3.0000 &-\\
    SLIM &6.6875 &6.6250 & 3.3125&3.1250\\
    RISKSLIM &4.3125  & 4.8125&5.2500  &3.5000\\
    XGBoost &2.6875 &3.1250  &6.6250 &-\\
    \bottomrule
    \end{tabular}
\end{table}
\begin{table}[htbp]
    \centering
    \caption{Nemenyi test for AUROC}
    \label{tab:nemenyi_auroc}
    \begin{tabular}{crrrrrrr}
    \toprule
    & RSS-DNB & RSS-DNB-SA & Logistic & LASSO& Decision tree& SLIM  & RISKSLIM\\
    \midrule
    RSS-DNB-SA &  1.000   &-&-&-&-&-&-\\
    Logistic   &0.730  &0.664 &-&-&-&-&-\\
    LASSO & 0.791 &0.730 &1.000 &-&-&-&-\\
    Decision tree &0.730  & 0.791    &  0.024  &  0.033 &-&-&-\\
    SLIM &0.791 &  0.844  &    0.033 &   0.046& 1.000&-&-\\
    RISKSLIM & 1.000   &1.000    &  0.925  &  0.952& 0.454  &       0.524&-\\
    XGBoost & 0.664 &  0.595    &  1.000   & 1.000& 0.017    &     0.024 &0.889\\
    \bottomrule
    \end{tabular}
\end{table}

\begin{table}[htbp]
    \centering
    \caption{Nemenyi test for AUNBC}
    \label{tab:nemenyi_aunbc}
    \begin{tabular}{crrrrrrr}
    \toprule
    & RSS-DNB & RSS-DNB-SA & Logistic & LASSO& Decision tree& SLIM  & RISKSLIM\\
    \midrule
    RSS-DNB-SA     & 1.000 &-&-&-&-&-&-\\
    Logistic     & 0.989  & 0.939 &-&-&-&-&-\\
    LASSO & 1.000 &  1.000  &    0.984 &-&-&-&-\\
    Decision tree & 0.524 &  0.730   &   0.093 &   0.559 &-&-&-\\
    SLIM & 0.355 &  0.559   &   0.046 &   0.387 &1.000  &-&-\\
    RISKSLIM &0.997&   1.000  &    0.761  &  0.998& 0.925   &      0.818&-\\
    XGBoost & 0.998 &  0.979    &  1.000  &  0.997& 0.155    &     0.081& 0.868\\
    \bottomrule
    \end{tabular}
\end{table}

\begin{table}[htbp]
    \centering
    \caption{Nemenyi test for ECE}
    \label{tab:nemenyi_ece}
    \begin{tabular}{crrrrrrr}
    \toprule
    & RSS-DNB & RSS-DNB-SA & Logistic & LASSO& Decision tree& SLIM  & RISKSLIM\\
    \midrule
    RSS-DNB-SA     & 0.998 &-&-&-&-&-&-\\
    Logistic     & 0.355 &  0.791 &-&-&-&-&-\\
    LASSO & 0.020 &  0.137 &0.952  &-&-&-&-\\
    Decision tree & 1.000 &  1.000  &    0.595   & 0.061 &-&-&-\\
    SLIM & 0.999 &  1.000   &   0.761  &  0.121& 1.000 &-&-\\
    RISKSLIM & 0.355 &  0.791  &    1.000 &   0.952& 0.595&         0.761&-\\
    XGBoost &0.020&   0.137   &   0.952   & 1.000& 0.061    &     0.121& 0.952\\
    \bottomrule
    \end{tabular}
\end{table}

\begin{table}[htbp]
    \centering
    \caption{Nemenyi test for model size}
    \label{tab:nemenyi_size}
    \begin{tabular}{crrrrr}
    \toprule
    & RSS-DNB & RSS-DNB-SA & Logistic & LASSO& SLIM \\
    \midrule
    RSS-DNB-SA     & 0.917  &-&-&-&-\\
    Logistic     & 0.039 & 0.001 &-&-&-\\
    LASSO & 1.000 & 0.765    & 0.096 &-&-\\
    SLIM & 1.000 & 0.894  &   0.047  & 1.000&- \\
    RISKSLIM &0.997&  0.684  &   0.132 &  1.000& 0.999 \\
    \bottomrule
    \end{tabular}
\end{table}
\section{Additional Materials}\label{app:additional}
This appendix provides additional materials, including a table of computational statistics of MILP-based models in \Cref{sec:experiments} and the algorithms for synthetic prediction in \Cref{sec:Method} and for solving problem (\ref{eq:risk prob}). 
\subsection{Additional Table}
\Cref{tab:computational_statistics} summarizes the computational statistics of RSS-DNB, SLIM, and RISKSLIM. The reported runtime and optimality gap are averaged over the 10 cross-validation folds. Since the solver termination status was consistent across all folds for each dataset and method, it is reported once.
\begin{table}[!htbp]
    \centering
    \renewcommand{\arraystretch}{1.2}
    \caption{Computational statistics of RSS-DNB, SLIM, and RISKSLIM.}
    \label{tab:computational_statistics}
    \begin{tabular}{cccccc}
    \toprule
    Dataset & Method & Runtime (s) & Gap (\%) & Solver Status & Solved Folds  \\
    \midrule
    \multirow{3}{*}{adult}
    & RSS-DNB & 600.1 & 85.0 & Time limit& 0/10\\
    & SLIM &600.1 & 57.9 &Time limit &0/10\\
    & RISKSLIM & 648.1 & 100.0 & Time limit&0/10\\
    \midrule
    \multirow{3}{*}{bankruptcy}
    & RSS-DNB & 0.2 & 0.0 & Optimal &10/10\\
    & SLIM &$<$0.1 &0.0 &Optimal &10/10\\
    & RISKSLIM &16.2 &0.0 & Optimal &10/10\\
    \midrule
    \multirow{3}{*}{breastcancer}
    & RSS-DNB & 600.1 & 8.3 &Time limit &0/10\\
    & SLIM & 9.0&0.0 &Optimal &10/10\\
    & RISKSLIM &600.3 & 58.1& Time limit&0/10\\
    \midrule
    \multirow{3}{*}{haberman}
    & RSS-DNB & 600.1& 22.7&Time limit &0/10\\
    & SLIM & 88.3&0.0 &Optimal &10/10\\
    & RISKSLIM & 12.1& 0.0& Optimal &10/10\\
    \midrule
    \multirow{3}{*}{heartdisease}
    & RSS-DNB & 600.1& 30.7&Time limit &0/10\\
    & SLIM & 600.1&62.8 &Time limit &0/10\\
    & RISKSLIM &600.2 &68.6 & Time limit&0/10\\
    \midrule
    \multirow{3}{*}{mammo}
    & RSS-DNB &600.1 &1.1 &Time limit &0/10\\
    & SLIM & 0.6&0.0 &Optimal &10/10\\
    & RISKSLIM &600.1 &5.6 & Time limit&0/10\\
    \midrule
    \multirow{3}{*}{mushroom}
    & RSS-DNB &83.2 &0.0 &Optimal &10/10\\
    & SLIM &0.9 &0.0 &Optimal &10/10\\
    & RISKSLIM & 764.1&47.2 & Time limit&0/10\\
    \midrule
    \multirow{3}{*}{spambase}
    & RSS-DNB &4.0 & 0.0& Optimal&10/10\\
    & SLIM & 600.0&89.6 & Time limit&0/10\\
    & RISKSLIM & 615.3& 100.0 & Time limit&0/10\\
    \bottomrule
    \end{tabular}
\end{table}
\subsection{Additional Algorithms}
\Cref{alg:synthetic_1} present the algorithm for generating synthetic predictions with a specified Pearson correlation with the output. \Cref{alg:synthetic_2} presents the algorithm for generating synthetic predictions that achieve the maximum AUNBC$(\bm p; c)$ for a given AUROC$(\bm p; c)$. \Cref{alg:sa} presents the algorithm for solving problem \ref{eq:risk prob} using simulated annealing. \Cref{alg:find_opt_T} provides a subroutine involked within \Cref{alg:sa} to find the optimal intercept vector $\bm T^{opt}$ for a given coefficient vector $\bm \lambda$ and threshold vector $\bm p$.
\begin{algorithm}\SetKwInOut{Input}{Input}\SetKwInOut{Output}{Output }
    \caption{Generation of Synthetic Predictions (Type I: Controlled Correlation)}\label{alg:synthetic_1}
    \Input{Binary outcome $\bm Y$, target correlation level $r$}
    \Output{Synthetic prediction scores $\bm s$ such that $\mathrm{corr}(\bm s, \bm Y) = r$}
    \BlankLine
    $N \gets \mathrm{length}(\bm Y)$ \;
    $\bm y \gets \left(\bm Y - \mathrm{mean}(\bm Y)\right)/\mathrm{std}(\bm Y)$    \tcp*{Standardize the outcome}
    Sample $\bm z \sim \mathcal{N}(0,I_N)$  \tcp*{Generate a random vector $\bm z$}
    $\bm z \gets \bm z - \frac{\langle \bm z, \bm y \rangle}{\langle \bm y, \bm y \rangle} \bm y $ \tcp*{Remove projection onto $\bm y$}
    $\bm z \gets \bm z / \mathrm{std}(\bm z)$ \;
    $\bm s \gets r \cdot \bm y + \sqrt{1 - r^2} \cdot \bm z$ \tcp*{Construct the synthetic prediction}
    $\bm s \gets \bm s - \min(\bm s)$, $\bm s \gets \bm s / \max(\bm s)$
     \tcp*{Rescale to $[0,1]$}
\end{algorithm}

\begin{algorithm}\SetKwInOut{Input}{Input}\SetKwInOut{Output}{Output }
    \caption{Generation of Synthetic Predictions (Type II: Boundary-Attaining)} \label{alg:synthetic_2}
    \Input{Binary outcome $\bm Y$, target d-AUROC value $G$, threshold $\bm p$ with $0 = p_0<p_1<\ldots<p_M<p_{M+1}=1$}
    \Output{Synthetic prediction scores $\bm s$ achieving maximum d-AUNBC when $\text{d-AUROC}=G$}
    $N \gets \mathrm{length}(\bm Y)$ \;
    $a_0 \gets \mathrm{mean}(\bm Y)$  \tcp*{Proportion of positive labels}
    $b_0 \gets 1 - a_0$  \tcp*{Proportion of negative labels}
    \For{$k \gets 1$ to $M$}{
    $P_k \gets \sum_{i=1}^k\frac{(p_{i+1}-p_i)p_i}{1-p}$
    }
    \For{$k \gets 1$ to $M$ }{\tcp*{Compute maximum d-AUNBC for each k}
    \eIf{$G \le 1 - \frac{b_0 P_k}{(1-p_k) a_0} $}{
    $\bm b_{1:k} \gets b_0$
    }{$\bm b_{1:k} \gets \sqrt{\frac{{(1-p_k){a_0 b_0 (1-G)}}}{P_k}}$}
    $b_{k+1:M} \gets 0$, $a_{1:k-1} \gets a_0$, $a_{k:M} \gets a_0 - \frac{(1-G) a_0 b_0}{b_1}$ \;
    $F_k \gets \sum_{i=0}^M (p_{i+1}-p_i)\left(a_i - b_i \cdot \frac{p_i}{1-p_i}\right)$
    }
    $K = \underset{1\le i\le M} {\arg\max}\, F_i$ 
    \eIf{$G \le 1 - \frac{b_0 P_K}{(1-p_K) a_0} $}{
    $\bm b_{1} \gets b_0$
    }{$\bm b_{1} \gets \sqrt{\frac{{(1-p_K){a_0 b_0 (1-G)}}}{P_K}}$}
    $a_{K} \gets a_0 - \frac{(1-G) a_0 b_0}{b_1}$ \;
    \tcp{Generate prediction score $\bm s$ such that d-$\mathrm{AUNBC}=F_K$}
    Initialize $\bm s$ as an array of size $N$\;
    \ForEach{$i$ such that $Y_i = 0$}{
        $s_i \gets p_{K}$\;
    }
    Find the smallest index $I$ such that $\sum_{i=1}^I Y_i = \mathrm{round}(N\cdot a_K)$\;
    \ForEach{$i$ such that $Y_i = 1$}{
    \eIf{$i \le I$}{$s_{i} \gets 1$}
    {$s_{i} \gets p_{K-1}$}}
\end{algorithm}
\begin{algorithm}\SetKwInOut{Input}{Input}\SetKwInOut{Output}{Output }
    \caption{Simulated Annealing for RSS-DNB}\label{alg:sa}
    \KwData{$\{(\bm{x}_j,y_j)\}_{j=1}^N \subseteq \mathbb{R}^P \times \{0,1\}$}
    \Input{Initial Coefficient $\bm \lambda^0$, initial temperature $t^0$, cooling rate $\alpha$, minimum temperature $t^{min}$, threshold $\bm p$ with $0=p_0<p_1<\ldots<p_M<p_{M+1}=1$, maximum coefficient $\Lambda$, penalty factor $C_0$, iterations per temperature $L$}
    \Output{Optimal Coefficient $\bm \lambda^{opt}$, and optimal intercept $\bm T^{opt}$}
    $t \leftarrow t^0$, $\bm \lambda \leftarrow \bm \lambda^0$\;
    $[\bm T,Loss] \leftarrow \mathrm{FindOptimalT}\left(\{(\bm{x}_j,y_j)\}_{j=1}^N;\bm
    \lambda, \bm p, C_0\right)$ \;
    $Loss^{opt} \leftarrow Loss$ \;
    \While{$t > t^{min} $}{
    \For{$iter \leftarrow 1$ to $L$}{
        $\bm \lambda^{new} \leftarrow \bm \lambda$\;
        $i \leftarrow \mathrm{RandomInteger}([1,P])$ \;
        $\lambda_i^{new} \leftarrow \mathrm{RandomInteger}([-\Lambda,\Lambda]\backslash\{\lambda_i\} )$ \;
        $\left[\bm T^{new}, {Loss}^{new}\right] \leftarrow \mathrm{FindOptimalT}\left(\{(\bm{x}_j,y_j)\}_{j=1}^N;\bm \lambda^{new}, \bm p, C_0\right)$ \;
        \If{${Loss}^{new} < {Loss}$}{
            ${Loss} \leftarrow {Loss}^{new} $, $\bm \lambda \leftarrow \bm \lambda^{new}$, $T \leftarrow T^{new}$ \;
            \If{$Loss < Loss^{opt}$}{
            ${Loss}^{opt} \leftarrow {Loss} $, $\bm \lambda^{opt} \leftarrow \bm \lambda$, $T^{opt} \leftarrow T$ \;
            }          
        }
        \ElseIf{$\mathrm{Random}() < \exp\{\left(Loss - Loss^{new}\right)/t\}$}{
            ${Loss} \leftarrow {Loss}^{new} $, $\bm \lambda \leftarrow \bm \lambda^{new}$, $T \leftarrow T^{new}$ \;
        }
    }
    $t \leftarrow t - \alpha$
    }
\end{algorithm}
\begin{algorithm}\SetKwInOut{Input}{Input}\SetKwInOut{Output}{Output }
    \caption{FindOptimalT}\label{alg:find_opt_T}
    \KwData{$\{(\bm{x}_j,y_j)\}_{j=1}^N \subseteq \mathbb{R}^P \times \{0,1\}$}
    \Input{Coefficient $\bm \lambda$, threshold $\bm p$ with $0=p_0<p_1<\ldots<p_M<p_{M+1}=1$, and penalty factor $C_0$}
    \Output{Optimal intercept $\bm T^{opt}$ and $Loss$}
    $\hat{y}_j \leftarrow \langle \bm x_j, \bm \lambda \rangle$, for $1\le j \le N$\;
    $\mathcal{B} \leftarrow \mathrm{Sort}\left(\left\{\lfloor \hat y_j \rfloor: 1 \le j \le N \right\} \cup \left\{\max_{1\le j \le N} \lfloor \hat y_j \rfloor+1\right\}\right)$\;
    $T_0^{opt} \leftarrow \min _{1\le j \le N}\lfloor \hat y_j \rfloor$\;
    $NB^0 \leftarrow \sum_{j=1}^N (p_1-p_0)\cdot I(y_j=1)/N$ \;
    \For{$i \leftarrow 1$ to M}{
        $k \leftarrow 1$ \;
        \ForEach{$T \in \mathcal{B}\cap [T_{i-1}^{opt}, \infty)$ }{
            $\mathrm{TP}_k^i \leftarrow \sum_{j=1}^N I(\hat{y}_j \ge T, y_j = 1)$\;
            $\mathrm{FP}_k^i \leftarrow \sum_{j=1}^N I(\hat{y}_j \ge T, y_j = 0)$\;
            $\mathrm{NB}_k^i \leftarrow \mathrm{TP}_k/N - \mathrm{FP}_k/N \cdot {p_i}/{(1-p_i)}$\;
            $T_k^i \leftarrow T$\;
            $k \leftarrow k+1$ \;
        }
        $Index \leftarrow {\arg\max}_{k \ge 1}\left\{\mathrm{NB}_k^i\right\}$ \;        
        $\mathrm{NB}^i \leftarrow \mathrm{NB}_{Index}^i$\;
        $T_i^{opt} \leftarrow T_{Index}^i$ \;
    }
    $Loss \leftarrow -\sum_{i=0}^M (p_{i+1}-p_i)NB^i + C_0 \cdot \|\lambda\|_0$\;
    $\bm T^{opt} \leftarrow \left(T_0^{opt},T_1^{opt},\ldots,T_M^{opt}\right)$\;
\end{algorithm}

\clearpage


\bibliographystyle{plainnat}  
\bibliography{references}

\end{document}